\documentclass{article}

\usepackage[final, nonatbib]{neurips_2023}

\usepackage[utf8]{inputenc} 
\usepackage[T1]{fontenc}    
\usepackage{hyperref}       
\usepackage{url}            
\usepackage{booktabs}       
\usepackage{amsfonts}       
\usepackage{nicefrac}       
\usepackage{microtype}      
\usepackage{xcolor}         

\usepackage{amsmath}
\usepackage{amssymb}
\usepackage{mathtools}
\usepackage{amsthm}
\usepackage{xspace}
\usepackage{booktabs}
\usepackage{authblk}
\usepackage{amsmath}
\usepackage{amsthm}
\usepackage{amsfonts}
\usepackage{amssymb}
\usepackage{graphicx}
\usepackage{mathrsfs}
\usepackage{float}
\usepackage{tikz}
\usepackage{tikz-network}
\usepackage{pgfplots}
\usepackage{colortbl}
\usepackage{enumitem}
\usepackage{caption}
\setlist[description]{leftmargin=\parindent,labelindent=\parindent}

\newcommand{\qcr}[1]{{\fontfamily{qcr}\selectfont #1}}
\newcommand{\tsf}[1]{{\text{\textsf{#1}}}}

\usetikzlibrary{decorations.pathreplacing,calc,shapes,positioning,arrows,arrows}
\usepackage{pgfplots}
\usepackage{subcaption}
\pgfplotsset{%
	,compat=1.12
	,every axis x label/.style={at={(current axis.right of origin)},anchor=north west}
	,every axis y label/.style={at={(current axis.above origin)},anchor=north east}
}
\usetikzlibrary{datavisualization,shapes.geometric,arrows.meta,decorations.markings, fit}
\usetikzlibrary{datavisualization.formats.functions}
\definecolor{scarlet}{rgb}{1.0, 0.13, 0.0}
\definecolor{brightmaroon}{rgb}{0.76, 0.13, 0.28}
\definecolor{mediumturquoise}{rgb}{0.28, 0.82, 0.8}
\definecolor{fandango}{rgb}{0.71, 0.2, 0.54}
\definecolor{antiquewhite}{rgb}{0.98, 0.92, 0.84}
\definecolor{babyblue}{rgb}{0.54, 0.81, 0.94}
\definecolor{brilliantlavender}{rgb}{0.96, 0.73, 1.0}
\definecolor{bronze}{rgb}{0.8, 0.5, 0.2}
\definecolor{cornsilk}{rgb}{1.0, 0.97, 0.86}
\definecolor{lavenderpink}{rgb}{0.98, 0.68, 0.82}
\definecolor{sandybrown}{rgb}{0.96, 0.64, 0.38}
\definecolor{celadon}{rgb}{0.67, 0.88, 0.69}

\newcommand{\ph}{CoxPH\xspace}

\newcommand{\ind}{\perp\!\!\!\!\perp} 
\newcommand{\esssup}{\text{esssup}}
\newcommand{\holderspace}[3]{\mathcal{W}^{#1}_{#2}(#3)}

\newcommand{\result}[2]{${#1}_{\pm#2}$}
\newcommand{\resultf}[2]{$\mathbf{#1}_{\pm#2}$}
\newcommand{\results}[2]{$\underline{#1}_{\pm#2}$}

\newcommand{\cn}[3]{N\left(#1, #2, #3\right)}

\newcommand{\bn}[3]{N_{[]}\left(#1, #2, #3\right)}

\newcommand{\var}{\text{Var}}
\newcommand{\vcdim}[1]{\text{VC}\left(#1\right)}
\newcommand{\fdivergence}[3]{\text{D}_{\text{#1}}\left(#2 \parallel #3\right)}

\theoremstyle{plain}
\newtheorem{theorem}{Theorem}[section]

\newtheorem{lemma}[theorem]{Lemma}

\theoremstyle{definition}

\theoremstyle{remark}
\newtheorem{remark}[theorem]{Remark}
\theoremstyle{plain}
\newtheorem{condition}[theorem]{Condition}

\title{Neural Frailty Machine: Beyond proportional hazard assumption in neural survival regressions}

\author[$\dagger$]{Ruofan Wu$^*$}
\author[$\ddagger$]{Jiawei Qiao\thanks{Equal contribution}\ \ }
\author[$\S$]{Mingzhe Wu}
\author[$\ddagger$]{\authorcr Wen Yu}
\author[$\ddagger$]{Ming Zheng}
\author[$\dagger$]{Tengfei Liu}
\author[$\dagger$]{Tianyi Zhang}
\author[$\dagger$]{Weiqiang Wang}
\affil[$\dagger$]{Ant Group}
\affil[$\ddagger$]{Fudan University}
\affil[$\S$]{Coupang}
\affil[ ]{\footnotesize{\texttt{\{ruofan.wrf, aaron.ltf, zty113091, weiqiang.wwq\}@antgroup.com}}}
\affil[ ]{\footnotesize{\texttt{jeremyqjw@163.com, wumingzhe.darcy@gmail.com, \{wenyu, mingzheng\}@fudan.edu.cn}}}

\begin{document}

\maketitle

    \begin{abstract}
        We present neural frailty machine (NFM), a powerful and flexible neural modeling framework for survival regressions. The NFM framework utilizes the classical idea of multiplicative frailty in survival analysis as a principled way of extending the proportional hazard assumption, at the same time being able to leverage the strong approximation power of neural architectures for handling nonlinear covariate dependence. Two concrete models are derived under the framework that extends neural proportional hazard models and nonparametric hazard regression models. Both models allow efficient training under the likelihood objective. Theoretically, for both proposed models, we establish statistical guarantees of neural function approximation with respect to nonparametric components via characterizing their rate of convergence. Empirically, we provide synthetic experiments that verify our theoretical statements. We also conduct experimental evaluations over $6$ benchmark datasets of different scales, showing that the proposed NFM models achieve predictive performance comparable to or sometimes surpassing state-of-the-art survival models. Our code is publicly availabel at \href{https://github.com/Rorschach1989/nfm}{\texttt{https://github.com/Rorschach1989/nfm}}
    \end{abstract}

    \section{Introduction}
    Regression analysis of time-to-event data \cite{kalbfleisch2002statistical} has been among the most important modeling tools for clinical studies and has witnessed a growing interest in areas like corporate finance \cite{duffie2009frailty}, recommendation systems \cite{jing2017neural}, and computational advertising \cite{wu2015predicting}. The key feature that differentiates time-to-event data from other types of data is that they are often \emph{incompletely observed}, with the most prevailing form of incompleteness being the \emph{right censoring} mechanism \cite{kalbfleisch2002statistical}. In the right censoring mechanism, the duration time of a sampled subject is (sometimes) only known to be larger than the observation time instead of being recorded precisely. It is well known in the community of survival analysis that even in the case of linear regression, naively discarding the censored observations produces estimation results that are statistically biased \cite{buckley1979linear}, at the same time losses sample efficiency if the censoring proportion is high.\par 
    Cox's proportional hazard (\ph) model \cite{cox1972regression} using the convex objective of negative partial likelihood \cite{cox1975partial} is the \emph{de facto} choice in modeling right censored time-to-event data (hereafter abbreviated as censored data without misunderstandings). The model is \emph{semiparametric} \cite{bickel1993efficient} in the sense that the baseline hazard function needs no parametric assumptions. 
    The original formulation of \ph model assumes a linear form and therefore has limited flexibility since the truth is not necessarily linear. Subsequent studies extended \ph model to nonlinear variants using ideas from nonparametric regression \cite{huang1999efficient, cai2007partially, cai2008partially}, ensemble learning \cite{ishwaran2008random}, and neural networks \cite{faraggi1995neural, katzman2018deepsurv}. While such extensions allowed a more flexible nonlinear dependence structure with the covariates, the learning objectives were still derived under the proportional hazards (PH) assumption, which was shown to be inadequate in many real-world scenarios \cite{gray2000estimation}. The most notable case was the failure of modeling the phenomenon of crossing hazards \cite{stablein1985two}. It is thus of significant interest to explore extensions of \ph that both allow nonlinear dependence over covariates and relaxations of the PH assumption. \par 
    Frailty models \cite{wienke2010frailty, duchateau2007frailty} are among the most important research topics in modern survival analysis, in that they provide a principled way of extending \ph model via incorporating a multiplicative random effect to capture unobserved heterogeneity. The resulting parameterization contains many useful variants of \ph like the proportional odds model \cite{bennett1983analysis}, under specific choices of frailty families. While the theory of frailty models has been well-established \cite{murphy1994consistency, murphy1995asymptotic, parner1998asymptotic, kosorok2004robust}, most of them focused on the linear case. Recent developments on applying neural approaches to survival analysis \cite{katzman2018deepsurv, kvamme2019time, tang2022soden, rindt2022a} have shown promising results in terms of empirical predictive performance, with most of them lacking theoretical discussions. Therefore, it is of significant interest to build more powerful frailty models via adopting techniques in modern deep learning \cite{goodfellow2016deep} with provable statistical guarantees. \par
    In this paper, we present a general framework for neural extensions of frailty models called the \textbf{neural frailty machine (NFM)}. Two concrete neural architectures are derived under the framework: The first one adopts the proportional frailty assumption, allowing an intuitive interpretation of the neural \ph model with a multiplicative random effect.
    The second one further relaxes the proportional frailty assumption and could be viewed as an extension of nonparametric hazard regression (NHR) \cite{cox1990asymptotic, kooperberg1995hazard}, sometimes referred to as "fully neural" models under the context of neural survival analysis \cite{omi2019fully}. We summarize our contributions as follows.\par
    \begin{itemize}[leftmargin=*]
        \item We propose the neural frailty machine (NFM) framework as a principled way of incorporating unobserved heterogeneity into neural survival regression models. The framework includes many commonly used survival regression models as special cases.
        \item We derive two model architectures based on the NFM framework that extend neural \ph models and neural NHR models. Both models allow stochastic training and scale to large datasets.
        \item Theoretically, we show \emph{statistical correctness} of the two proposed models via characterizing the rates of convergence of the proposed nonparametric function estimators. The proof technique is different from previous theoretical studies on neural survival analysis and is applicable to many other types of neural survival models.
        \item Empirically, we verify the \emph{empirical efficacy} of the proposed framework via conducting extensive studies on various benchmark datasets at different scales. Under standard performance metrics, both models are empirically shown to perform competitively, matching or sometimes outperforming state-of-the-art neural survival models.
    \end{itemize}
    
    \section{Related works}
    \subsection{Nonlinear extensions of \ph}
    Most nonlinear extensions of \ph model stem from the equivalence of partial likelihood and semiparametric profile likelihood \cite{murphy2000profile} of \ph model, resulting in nonlinear variants that essentially replaces the linear term in partial likelihood with nonlinear variants: \cite{huang1999efficient} used smoothing splines, \cite{cai2007partially, cai2008partially} used local polynomial regression \cite{fan1996local}. The empirical success of tree-based models inspired subsequent developments like \cite{ishwaran2008random} that equip tree-based models such as gradient boosting trees and random forests with losses in the form of negative log partial likelihood. Early developments of neural survival analysis \cite{faraggi1995neural} adopted similar extension strategies and obtained neural versions of partial likelihood. Later attempts \cite{katzman2018deepsurv} suggested using the successful practice of stochastic training which is believed to be at the heart of the empirical success of modern neural methods \cite{hardt2016train}. However, stochastic training under the partial likelihood objective is highly non-trivial, as mini-batch versions of log partial likelihood \cite{katzman2018deepsurv} are no longer valid stochastic gradients of the full-sample log partial likelihood \cite{tang2022soden}. 
    
    \subsection{Beyond \ph in survival analysis}
    In linear survival modeling, there are standard alternatives to \ph such as the accelerated failure time (AFT) model \cite{buckley1979linear, ying1993large}, the extended hazard regression model \cite{etezadi1987extended}, and the family of linear transformation models \cite{zeng2006efficient}. While these models allow certain types of nonlinear extensions, the resulting form of (conditional) hazard function is still restricted to be of a specific form. The idea of nonparametric hazard regression (NHR) \cite{cox1990asymptotic, kooperberg1995hazard, strawderman1996asymptotic} further improves the flexibility of nonparametric survival analysis via directly modeling the conditional hazard function by nonparametric regression techniques such as spline approximation. Neural versions of NHR have been developed lately such as the CoxTime model \cite{kvamme2019time}. \cite{rindt2022a} used a neural network to approximate the conditional survival function and could be thus viewed as another trivial extension of NHR.\par
    Aside from developments in NHR, \cite{lee2018deephit} proposed a discrete-time model with its objective being a mix of the discrete likelihood and a rank-based score; \cite{zhong2021deep} proposed a neural version of the extended hazard model, unifying both neural \ph and neural AFT model; \cite{tang2022soden} used an ODE approach to model the hazard and cumulative hazard functions. 
    
    \subsection{Theoretical justification of neural survival models}
    Despite the abundance of neural survival models, assessment of their theoretical properties remains nascent. In \cite{zhong2021partially}, the authors developed minimax theories of partially linear cox model using neural networks as the functional approximator. \cite{zhong2021deep} provided convergence guarantees of neural estimates under the extended hazard model. The theoretical developments therein rely on specific forms of objective function (partial likelihood and kernel pseudo-likelihood) and are not directly applicable to the standard likelihood-based objective which is frequently used in survival analysis.
    
    \section{Methodology}
    \subsection{The neural frailty machine framework}
    Let $\tilde{T} \ge 0$ be the interested event time with survival function denoted by $S(t)=\mathbb{P}(\tilde{T}>t)$ associated with a feature(covariate) vector $Z \in \mathbb{R}^d $. Suppose that $\tilde{T}$ is a continuous random variable and let $f(t)$ be its density function. Then $\lambda(t)=f(t)/S(t)$ is the hazard function and $\Lambda(t)=\int_0^t\lambda(s)ds$ is the cumulative hazard function. Aside from the covariate $Z$, we use a positive scalar random variable $\omega \in \mathbb{R}^+$ to express the unobserved heterogeneity corresponding to individuals, or \emph{frailty}.
    \footnote{For example in medical biology, it was observed that genetically identical animals kept in as similar an environment as possible will typically not behave the same upon exposure to environmental carcinogens \cite{brennan2002gene}}. 
    In this paper we will assume the following generating scheme of $\tilde{T}$ via specifying its conditional hazard function:
    \begin{align}\label{eqn: frailty_general}
        \lambda(t| Z, \omega) = \omega \widetilde{\nu}(t, Z).
    \end{align}
    Here $\widetilde{\nu}$ is an unspecified non-negative function, and we let the distribution of $\omega$ be parameterized by a one-dimensional parameter $\theta \in \mathbb{R}$. 
    \footnote{The choice of one-dimensional frailty family is mostly for simplicity and clearness of theoretical derivations. Note that there exist multi-dimensional frailty families like the PVF family \cite{wienke2010frailty}. Generalizing our theoretical results to such kinds of families would require additional sets of regularity conditions, and will be left to future explorations.}
    The formulation \eqref{eqn: frailty_general} is quite general and contains several important models in both traditional and neural survival analysis:
    \begin{enumerate}[leftmargin=*]
        \item When $\omega$ follows parametric distributional assumptions, and $\widetilde{\nu}(t, Z) = \lambda(t) e^{\beta^\top Z}$, \eqref{eqn: frailty_general} reduces to the standard proportional frailty model \cite{kosorok2004robust}. A special case is when $\omega$ is degenerate, i.e., it has no randomness, then the model corresponds to the classic \ph model.
        \item When $\omega$ is degenerate and $\widetilde{\nu}$ is arbitrary, the model becomes equivalent to nonparametric hazard regression (NHR) \cite{cox1990asymptotic, kooperberg1995hazard}. In NHR, the function parameter of interest is usually the logarithm of the (conditional) hazard function.
    \end{enumerate}
    In this paper we construct neural approximations to the logarithm of $\widetilde{\nu}$, i.e., $\nu(t, Z) = \log \widetilde{\nu}(t, Z)$. The resulting models are called \textbf{Neural Frailty Machines (NFM)}. Depending on the prior knowledge of the function $\nu$, we propose two function approximation schemes:\par
    \textbf{The proportional frailty (PF) scheme} assumes the dependence of $\nu$ on event time and covariates to be completely \emph{decoupled}, i.e., 
    \begin{align}\label{eqn: proportional_frailty}
        \nu(t, Z) = h(t) + m(Z).
    \end{align}
    Proportional-style assumption over hazard functions has been shown to be a useful inductive bias in survival analysis. We will treat both $h$ and $m$ in \eqref{eqn: proportional_frailty} as function parameters, and device two multi-layer perceptrons (MLP) to approximate them separately. \par
    \textbf{The fully neural (FN) scheme} imposes no a priori assumptions over $\nu$ and is the most general version of NFM. It is straightforward to see that the most commonly used survival models, such as \ph, AFT\cite{ying1993large}, EH\cite{zhong2021deep}, or PF models are included in the proposed model space as special cases. We treat $\nu = \nu(t, Z)$ as the function parameter with input dimension $d + 1$ and use a multi-layer perceptron (MLP) as the function approximator to $\nu$. Similar approximation schemes with respect to the hazard function have been proposed in some recent works \cite{omi2019fully, rindt2022a}, referred to as "fully neural approaches" without theoretical characterizations. \par
    \textbf{The choice of frailty family} There are many commonly used families of frailty distributions \cite{kosorok2004robust, duchateau2007frailty, wienke2010frailty}, among which the most popular one is the \emph{gamma frailty}, where $\omega$ follows a gamma distribution with mean $1$ and variance $\theta$. We briefly introduce some other types of frailty families in appendix \ref{sec: frailty_spec}.
    
    \subsection{Parameter learning under censored observations}
    In time-to-event modeling scenarios, the event times are typically observed under right censoring. Let $C$ be the right censoring time which is assumed to be conditionally independent of the event time $\tilde{T}$ given $Z$, i.e., $\tilde{T} \ind C | Z$. In data collection, one can observe the minimum of the survival time and the censoring time, that is, observe $T=\tilde{T}\wedge C$ as well as the censoring indicator $\delta=I(\tilde{T}\leqslant C)$, where $a\wedge b=\min(a,b)$ for constants $a$ and $b$ and $I(\cdot)$ stands for the indicator function. We assume $n$ independent and identically distributed (i.i.d.) copies of $(T, \delta, Z)$ are used as the training sample $(T_i, \delta_i, Z_i), i \in [n]$, where we use $[n]$ to denote the set $\{1, 2, \ldots, n\}$. Additionally, we assume the unobserved frailties are independent and identically distributed, i.e., $\omega_i \overset{\text{i.i.d.}}{\sim} f_\theta(\omega), i \in [n]$. 
    Next, we derive the learning procedure based on the \textbf{observed log-likelihood (OLL)} objective under both PF and FN scheme.
    To obtain the observed likelihood, we first integrate the conditional survival function given the frailty:
    \begin{align}
        \begin{aligned}
            S(t|Z) &= \mathbb{E}_{\omega \sim f_\theta}\left[e^{-\omega \int_0^t e^{\nu(s, Z)}ds}\right] =: e^{- G_\theta\left(\int_0^t e^{\nu(s, Z)}ds\right)}.
        \end{aligned}
    \end{align}
    Here the \emph{frailty transform} $G_\theta (x) = - \log \left(\mathbb{E}_{\omega \sim f_\theta}\left[e^{-\omega x}\right]\right)$ is defined as the negative of the logarithm of the Laplace transform of the frailty distribution. The conditional cumulative hazard function is thus $\Lambda(t|Z) = G_\theta(\int_0^t e^{\nu(s, Z)}ds)$. For the PF scheme of NFM, we use two MLPs $\widehat{h} = \widehat{h}(t; \mathbf{W}^h, \mathbf{b}^h)$ and $\widehat{m} = \widehat{m}(Z; \mathbf{W}^m, \mathbf{b}^m)$ as function approximators to $\nu$ and $m$, parameterized by $(\mathbf{W}^h, \mathbf{b}^h)$ and $(\mathbf{W}^m, \mathbf{b}^m)$, respectively. 
    \footnote{Here we adopt the conventional notation that $\mathbf{W}$ is the collection of the weight matrices of the MLP in all layers, and $\mathbf{b}$ corresponds to the collection of the bias vectors in all layers.}
    According to standard results on censored data likelihood \cite{kalbfleisch2002statistical}, we write the learning objective under the PF scheme as:
    \begin{align}\label{eqn: pf_obj}
    \resizebox{0.93\columnwidth}{!}{$
        \begin{aligned}
            &\mathcal{L}(\mathbf{W}^h, \mathbf{b}^h, \mathbf{W}^m, \mathbf{b}^m, \theta) \\
            =& \frac{1}{n}\left[\sum_{i \in [n]} \delta_i \log g_{\theta}\left(e^{\widehat{m}(Z_i)}\int_{0}^{T_i}e^{\widehat{h}(s)}ds\right) +\delta_i \widehat{h}(T_i)+\delta_i \widehat{m}(Z_i) -G_{\theta}\left(e^{\widehat{m}(Z_i)}\int_{0}^{T_i}e^{\widehat{h}(s)}ds\right)\right].
        \end{aligned}
    $}
    \end{align}
    Here we define $g_\theta(x) = \frac{\partial}{\partial x}G_\theta (x)$. Let $(\widehat{\mathbf{W}}^h_n, \widehat{\mathbf{b}}^h_n, \widehat{\mathbf{W}}^m_n, \widehat{\mathbf{b}}^m_n, \widehat{\theta}_n)$ be the maximizer of \eqref{eqn: pf_obj} and further denote $\widehat{h}_n(t) = \widehat{h}(t; \widehat{\mathbf{W}}^h_n, \widehat{\mathbf{b}}^h_n)$ and $\widehat{m}_n(Z) = \widehat{m}(Z; \widehat{\mathbf{W}}^m_n, \widehat{\mathbf{b}}^m_n)$. 
    The resulting estimators for conditional cumulative hazard and survival functions are:
    \begin{align}
        \begin{aligned}
            \widehat{\Lambda}_{\tsf{PF}}(t|Z) = G_{\widehat{\theta}_n}\left(\int_0^t e^{\widehat{h}_n(s) + \widehat{m}_n(Z)} ds\right),\qquad \widehat{S}_{\tsf{PF}}(t|Z) = e^{-\widehat{\Lambda}_{\tsf{PF}}(t|Z)},
        \end{aligned}
    \end{align}
    For the FN scheme, we use $\widehat{\nu} = \widehat{\nu}(t, Z; \mathbf{W}^\nu, \mathbf{b}^\nu)$ to approximate $\nu(t, Z)$ parameterized by $(\mathbf{W}^\nu, \mathbf{b}^\nu)$. The OLL objective is written as:
    \begin{align}\label{eqn: fn_obj}
    \resizebox{0.93\columnwidth}{!}{$
        \begin{aligned}
            &\mathcal{L}(\mathbf{W}^\nu, \mathbf{b}^\nu, \theta) \\
            =& \frac{1}{n}\left[\sum_{i \in [n]} \delta_i \log g_\theta\left(\int_0^{T_i} e^{\widehat{\nu}(s, Z_i; \mathbf{W}^\nu, \mathbf{b}^\nu)}ds\right) + \delta_i \widehat{\nu}(T_i, Z_i; \mathbf{W}^\nu, \mathbf{b}^\nu) - G_\theta\left(\int_0^{T_i} e^{\widehat{\nu}(s, Z_i; \mathbf{W}^\nu, \mathbf{b}^\nu)}ds\right)\right].
        \end{aligned}
    $}
    \end{align}
    Let $(\widehat{\mathbf{W}}^\nu_n, \widehat{\mathbf{b}}^\nu_n, \widehat{\theta}_n)$ be the maximizer of \eqref{eqn: fn_obj}, and further denote $\widehat{\nu}_n(t, Z) = \widehat{\nu}(t, Z; \widehat{\mathbf{W}}^\nu_n, \widehat{\mathbf{b}}^\nu_n)$. The conditional cumulative hazard and survival functions are therefore estimated as:
    \begin{align}
        \begin{aligned}
            \widehat{\Lambda}_{\tsf{FN}}(t|Z) = G_{\widehat{\theta}_n}\left(\int_0^t e^{\widehat{\nu}_n(s, Z)} ds\right),\qquad \widehat{S}_{\tsf{FN}}(t|Z) = e^{-\widehat{\Lambda}_{\tsf{FN}}(t|Z)}.
        \end{aligned}
    \end{align}
    The evaluation of objectives like \eqref{eqn: fn_obj} and its gradient requires computing a definite integral of an exponentially transformed MLP function. Instead of using exact computations that are available for only a restricted type of activation functions and network structures, we use numerical integration for such kinds of evaluations, using the method of Clenshaw-Curtis quadrature \cite{boyd2001chebyshev}, which has shown competitive performance and efficiency in recent applications to monotonic neural networks \cite{wehenkel2019unconstrained}.
    \begin{remark}
        The interpretation of frailty terms differs in the two schemes. In the PF scheme, introducing the frailty effect strictly increases the modeling capability (i.e., the capability of modeling crossing hazard) in comparison to \ph or neural variants of \ph \cite{kosorok2004robust}. In the FN scheme, it is arguable that in the i.i.d. case, the marginal hazard function is a reparameterization of the hazard function in the context of NHR. Therefore, we view the incorporation of frailty effect as injecting a domain-specific inductive bias that has proven to be useful in survival analysis and time-to-event regression modeling and verify this claim empirically in section \ref{sec: benchmark_data}. Moreover, frailty becomes especially helpful when handling correlated or clustered data where the frailty term is assumed to be shared among certain groups of individuals \cite{parner1998asymptotic}. Extending NFM to such scenarios is valuable and we left it to future explorations.
    \end{remark}

    \section{Statistical guarantees}\label{sec: theory}
    In this section, we present statistical guarantees regarding both NFM estimates in the sense of nonparametric regression \cite{tsybakov2008introduction}, where we obtain rates of convergence to the ground truth function parameters (which is frequently referred to as the \emph{true parameter} in statistics literature). The results in this section is interpreted as showing the \emph{statistical correctness} of our approach. \\
    \textbf{Proof strategy} Our proof technique is based on the method of sieves \cite{shen1994convergence, shen1997methods, chen2007large} that views neural networks as a special kind of nonlinear sieve \cite{chen2007large} that satisfies desirable approximation properties \cite{yarotsky2017error}. Our strategy is different from previous theoretical works on neural survival models \cite{zhong2021deep, zhong2021partially} where the developments implicitly requires the loss function to be well-controlled by the $L_2$ loss and is therefore not directly applicable to our model due to the flexibility in choosing the frailty transform. \\
    Since both models produce estimates of function parameters, we need to specify a suitable function space to work with. Here we choose the following H\"older ball as was also used in previous works on nonparametric estimation using neural networks \cite{schmidt2020nonparametric, farrell2021deep, zhong2021partially}
    \begin{align}\label{eqn: holder_ball}
        \holderspace{\beta}{M}{\mathcal{X}} = \left\lbrace f: \max_{\mathbf{\alpha}: \left|\mathbf{\alpha}\right|\le \beta }\underset{x \in \mathcal{X}}{\esssup} \left|D^{\mathbf{\alpha}}(f(x))\right| \le M\right\rbrace,
    \end{align}
    where the domain $\mathcal{X}$ is assumed to be a subset of $d$-dimensional euclidean space. $\mathbf{\alpha} = (\alpha_1, \ldots, \alpha_d)$ is a $d$-dimensional tuple of nonnegative integers satisfying $\left|\mathbf{\alpha}\right| = \alpha_1 + \cdots + \alpha_d$ and $D^{\mathbf{\alpha}}f = \frac{\partial^{\left|\mathbf{\alpha}\right|}f}{\partial x_1^{\alpha_1} \cdots x_d^{\alpha_d}}$ is the weak derivative of $f$. Now assume that $M$ is a reasonably large constant, and let $\Theta$ be a closed interval over the real line. We make the following assumptions for the \emph{true parameters} under both schemes:
    \begin{condition}[True parameter, PF scheme]\label{cond: param_PF}
        The euclidean parameter $\theta_0 \in \Theta \subset \mathbb{R}$, and the two function parameters $m_0 \in \holderspace{\beta}{M}{[-1, 1]^d}, h_0 \in \holderspace{\beta}{M}{[0, \tau]}$, and $\tau > 0$ is the ending time of the study duration, which is usually adopted in the theoretical studies in survival analysis \cite{van2000asymptotic}.
    \end{condition}
    \begin{condition}[True parameter, FN scheme]\label{cond: param_FN}
        The euclidean parameter $\theta_0 \in \Theta \subset \mathbb{R}$, and the function parameter $\nu_0 \in \holderspace{\beta}{M}{[0, \tau] \times [-1, 1]^d}$, 
    \end{condition}
    
    Next, we construct sieve spaces for function parameter approximation via restricting the complexity of the MLPs to "scale" with the sample size $n$. 
    \begin{condition}[Sieve space, PF scheme]\label{cond: sieve_PF}
        The sieve space $\mathcal{H}_n$ is constructed as a set of MLPs satisfying $\widehat{h} \in \holderspace{\beta}{M_{h}}{[0, \tau]}$, with depth of order $O(\log n)$ and total number of parameters of order $O(n^{\frac{1}{\beta + d}}\log n)$. The sieve space $\mathcal{M}_n$ is constructed as a set of MLPs satisfying $\widehat{m} \in \holderspace{\beta}{M_{m}}{[-1, 1]^d}$, with depth of order $O(\log n)$ and total number of parameters of order $O(n^{\frac{d}{\beta + d}}\log n)$. Here $M_h$ and $M_m$ are sufficiently large constants such that every function in $\holderspace{\beta}{M}{[-1, 1]^d}$ and $\holderspace{\beta}{M}{[0, \tau]}$ could be accurately approximated by functions inside $\mathcal{H}_n$ and $\mathcal{M}_n$, according to \cite[Theorem 1]{yarotsky2017error}.
    \end{condition}
    \begin{condition}[Sieve space, FN scheme]\label{cond: sieve_FN}
        The sieve space $\mathcal{V}_n$ is constructed as a set of MLPs satisfying $\widehat{\nu} \in \holderspace{\beta}{M_{\nu}}{[0, \tau]}$, with depth of order $O(\log n)$ and total number of parameters of order $O(n^{\frac{d+1}{\beta + d + 1}}\log n)$. Here $M_\nu$ is a sufficiently large constant such that $\mathcal{V}_n$ satisfies approximation properties, analogous to condition \ref{cond: sieve_PF}.
    \end{condition}
    For technical reasons, we will assume the nonparametric function estimators are constrained to fall inside the corresponding sieve spaces, i.e., $\widehat{h}_n \in \mathcal{H}_n$, $\widehat{m}_n \in \mathcal{M}_n$ and $\widehat{\nu} \in \mathcal{V}_n$. This will not affect the implementation of optimization routines as was discussed in \cite{farrell2021deep}. Furthermore, we restrict the estimate $\widehat{\theta}_n \in \Theta$ in both PF and FN schemes.
    
    Additionally, we need the following regularity condition on the function $G_\theta(x)$:
    \begin{condition}\label{cond: G}
        $G_{\theta}(x)$ is viewed as a bivariate function $G: \Theta \times \mathcal{B} \mapsto \mathbb{R}$, where $\mathcal{B}$ is a compact set on $\mathbb{R}$. 
        The functions $G_{\theta}(x)$,$\frac{\partial}{\partial \theta}G_{\theta}(x)$,$\frac{\partial}{\partial x}G_{\theta}(x)$,$\log g_{\theta}(x)$,$\frac{\partial}{\partial \theta}\log g_{\theta}(x)$, $\frac{\partial}{\partial x}\log g_{\theta}(x)$ are bounded on $\Theta \times \mathcal{B}$.
    \end{condition}
    We define two metrics that measures convergence of parameter estimates: For the PF scheme, let $\phi_0 = (h_0, m_0, \theta_0)$ be the true parameters and $\widehat{\phi}_n = (\widehat{h}_n, \widehat{m}_n, \widehat{\theta}_n)$ be the estimates. We abbreviate $\mathbb{P}_{\phi_0, Z=z}$ as the conditional probability distribution of $(T, \delta)$ given $Z=z$ under the true parameter, and $ \mathbb{P}_{\widehat{\phi}_n, Z=z} $ as the conditional probability distribution of $(T, \delta)$ given $Z=z$ under the estimates. Define the following metric
    \begin{align}
        d_{\tsf{PF}}\left(\widehat{\phi}_n, \phi_0\right) = \sqrt{\mathbb{E}_{z\sim \mathbb{P}_Z}\left[H^{2}(\mathbb{P}_{\widehat{\phi}_n, Z=z}\parallel\mathbb{P}_{\phi_0, Z=z})\right]},
    \end{align}
    where $H^2(\mathbb{P}\parallel \mathbb{Q}) = \int \left(\sqrt{d\mathbb{P}} - \sqrt{d\mathbb{Q}}\right)^2$ is the squared Hellinger distance between probability distributions $\mathbb{P}$ and $\mathbb{Q}$. The case for the FN scheme is similar: Let $\psi_0 = (\nu_0, \theta_0)$ be the parameters and $\widehat{\nu}_n = (\widehat{\nu}_n, \widehat{\theta}_n)$ be the estimates. Analogous to the definitions above, we define $\mathbb{P}_{\psi_0, Z=z}$ as the true conditional distribution given $Z=z$, and $ \mathbb{P}_{\widehat{\psi}_n, Z=z} $ be the estimated conditional distribution, we will use the following metric in the FN scheme:
    \begin{align}
        d_{\tsf{FN}}\left(\widehat{\psi}_n, \psi_0\right) = \sqrt{\mathbb{E}_{z\sim \mathbb{P}_Z}\left[H^{2}(\mathbb{P}_{\widehat{\psi}_n, Z=z}\parallel\mathbb{P}_{\psi_0, Z=z})\right]}.
    \end{align}
    Now we state our main theorems. We denote $\mathbb{P}$ as the data generating distribution and use $\widetilde{O}$ to hide poly-logarithmic factors in the big-O notation.
    \begin{theorem}[Rate of convergence, PF scheme]\label{thm: rate_pf}
        In the PF scheme, under condition \ref{cond: param_PF}, \ref{cond: sieve_PF}, \ref{cond: G}, we have that $d_{\tsf{PF}}\left(\widehat{\phi}_n, \phi_0\right) = \widetilde{O}_\mathbb{P}\left(n^{-\frac{\beta}{2\beta+2d}}\right)$. 
    \end{theorem}
    \begin{theorem}[Rate of convergence, FN scheme]\label{thm: rate_fn}
        In the FN scheme, under condition \ref{cond: param_FN}, \ref{cond: sieve_FN}, \ref{cond: G}, we have that $d_{\tsf{FN}}\left(\widehat{\psi}_n, \psi_0\right) = \widetilde{O}_\mathbb{P}\left(n^{-\frac{\beta}{2\beta+2d+2}}\right)$.
    \end{theorem}
    \begin{remark}
        The idea of using Hellinger distance to measure the convergence rate of sieve MLEs was proposed in \cite{wong1995probability}. Obtaining rates under a stronger topology such as $L_2$ is possible if the likelihood function satisfies certain conditions such as the curvature condition \cite{farrell2021deep}. However, such kind of condition are in general too stringent for likelihood-based objectives, instead, we use Hellinger convergence that has minimal requirements. Consequently, our proof strategy is applicable to many other survival models that rely on neural function approximation such as \cite{rindt2022a}, with some modification to the regularity conditions. For proper choices of metrics in sieve theory, see also the discussion in \cite[Chapter 2]{chen2007large}.
    \end{remark}
    \section{Experiments}\label{sec: experiments}
    In this section, we report the empirical performance of NFM, we will focus on the following two research questions:
    \begin{description}
        \item[RQ$1$(Verfication of statistical correctness):]\label{rq: 1} The results listed in section characterized the convergence results \emph{in theory}, providing a crude guide on the number of samples required for an accurate estimate. Nonetheless, theoretical rates are often pessimistic, thus we want to investigate \textbf{whether a moderate number of sample size suffices for good approximation}. 
        \item[RQ$2$(Assessment of empirical efficacy):]\label{rq: 2} While NFM is theoretically sound in terms of \emph{estimation accuracy}, the theory we have developed does not necessarily guarantee its \emph{empirical efficacy} as a method of doing prognosis. It is therefore valuable to inspect \textbf{how useful NFM is regarding real-world predictive tasks in survival analysis}.
    \end{description}
    \subsection{Synthetic experiments}
    To answer RQ$1$, we conduct synthetic experiments to check the empirical convergence. Specifically, we investigate the empirical recovery of underlying ground truth parameters under various level of sample size. \\
    \textbf{Ground truth} We set the true underlying model to be a nonlinear gamma-frailty model with a $5$-dimensional feature. We generate three training datasets of different scales, with $n \in \{1000, 5000, 10000\}$. The assessment will be made on a fixed test sample of $100$ hold-out points that are independently drawn from the generating scheme of the event time. A censoring mechanism is applied such that the censoring ratio is around $40\%$ for each dataset. The precise form of the frailty model as well as the generating distribution of the feature vectors are detailed in appendix \ref{sec: synthetic_details}. \\
    \textbf{Empirical recovery results} We report the empirical recovery of the nonlinear component $\nu(t, Z)$ regarding the hold-out test set in in figure \ref{fig: synthesis_log_hazard}. We observe from graphical illustrations that under a moderate sample size $n = 1000$, NFM already exhibits satisfactory recovery for a (relatively low-dimensional) feature space, which is the prevailing case in most public benchmark datasets. We also present additional assessments about: (i) The recovery of $m(Z)$ when using PF scheme in appendix \ref{sec: m_z}, (ii) The recovery of survival functions under both PF and FN scheme in appendix \ref{sec: surv}, and (iii) The numerical recovery results of survival function in appendix \ref{sec: synthetic_numericals}.
    
    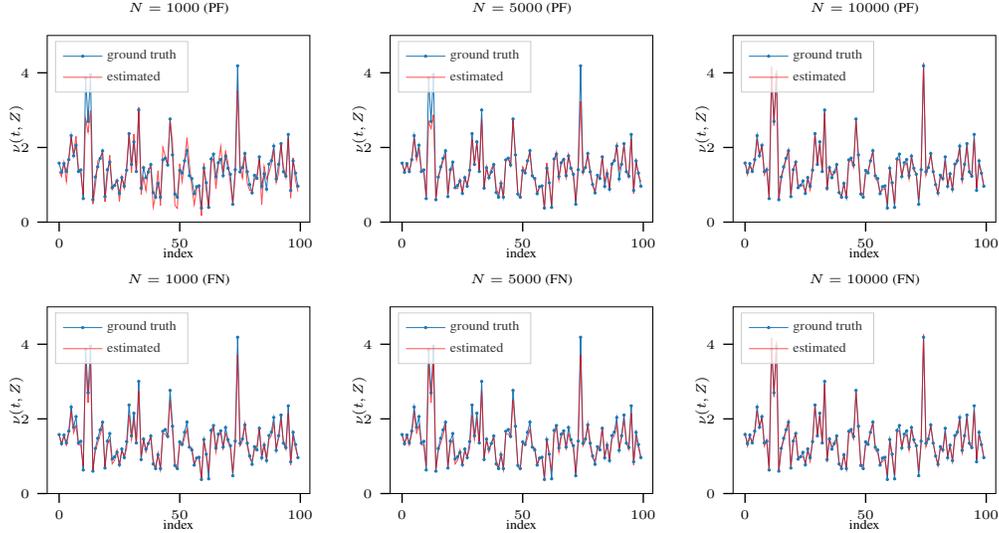
\begin{figure}
        \centering
        \begin{subfigure}[b]{0.32\textwidth}
        \resizebox{\linewidth}{0.8\linewidth}{
\begin{tikzpicture}

\definecolor{darkgray176}{RGB}{176,176,176}
\definecolor{lightgray204}{RGB}{204,204,204}
\definecolor{steelblue31119180}{RGB}{31,119,180}

\begin{axis}[
legend cell align={left},
legend style={
  font=\tiny,
  fill opacity=0.8,
  draw opacity=1,
  text opacity=1,
  at={(0.03,0.97)},
  anchor=north west,
  draw=lightgray204
},
scale only axis,
width=4cm,
height=3cm,
tick align=outside,
tick pos=left,
title={\fontsize{6}{6}\selectfont $N=1000\ \text{(PF)}$},
tick label style={font=\tiny},
every tick/.style={
black,
semithick,
},
x label style={at={(axis description cs:0.5,-0.1)},anchor=north,font=\tiny},
y label style={at={(axis description cs:-0.06,.5)},rotate=90,anchor=south,font=\tiny},
x grid style={darkgray176},
xlabel={index},
xmin=-4.95, xmax=103.95,
xtick style={color=black},
y grid style={darkgray176},
ylabel={$\nu(t, Z)$},
ymin=0, ymax=5,
ytick style={color=black}
]
\addplot [width=1pt, steelblue31119180, mark=*, mark size=0.5, mark options={solid}]
table {%
0 1.5803998708725
1 1.33550202846527
2 1.56831169128418
3 1.35660552978516
4 1.67359554767609
5 2.31658291816711
6 1.77639508247375
7 2.05915355682373
8 1.3549679517746
9 1.39965355396271
10 0.630053579807281
11 3.85166573524475
12 2.69694805145264
13 3.94658541679382
14 0.599072337150574
15 1.20898294448853
16 1.48052942752838
17 1.70718002319336
18 1.91135597229004
19 0.683144271373749
20 1.40443646907806
21 1.60692179203033
22 0.922011435031891
23 0.982581973075867
24 1.10471200942993
25 0.770239174365997
26 1.1936936378479
27 0.961292088031769
28 1.3913631439209
29 2.3709568977356
30 1.54486167430878
31 2.15229058265686
32 1.35555577278137
33 3.00449728965759
34 0.908643305301666
35 1.45915484428406
36 1.18871068954468
37 1.34153246879578
38 1.54231023788452
39 0.792442858219147
40 0.668309092521667
41 1.03730869293213
42 0.668369293212891
43 1.66741919517517
44 1.71565139293671
45 1.5300098657608
46 2.7632257938385
47 1.8025906085968
48 0.752088010311127
49 0.668880820274353
50 1.38437247276306
51 1.31735861301422
52 1.64243221282959
53 1.91407084465027
54 1.24209487438202
55 1.16776537895203
56 0.761925280094147
57 0.949998795986176
58 0.97191709280014
59 0.376778155565262
60 1.44549453258514
61 1.05457043647766
62 0.394052714109421
63 1.69112682342529
64 1.82295894622803
65 1.22094392776489
66 1.59055876731873
67 1.67814803123474
68 1.24474513530731
69 1.76935398578644
70 1.44282734394073
71 1.2813868522644
72 0.476826071739197
73 1.40479564666748
74 4.18656396865845
75 1.3455274105072
76 1.46364843845367
77 1.83773052692413
78 1.35060667991638
79 1.00683796405792
80 0.783329665660858
81 1.25577104091644
82 1.16983044147491
83 1.74855446815491
84 0.957872331142426
85 1.29508185386658
86 0.882132947444916
87 1.55416762828827
88 1.64668881893158
89 2.03826498985291
90 1.15617573261261
91 1.54844224452972
92 2.10270237922668
93 1.34900939464569
94 1.23071610927582
95 2.34977126121521
96 0.848920166492462
97 1.64094316959381
98 1.31309974193573
99 0.961935579776764
};
\addlegendentry{ground truth}
\addplot [width=1pt, red, opacity=0.7]
table {%
0 1.35992729663849
1 1.21821808815002
2 1.60440564155579
3 1.08843588829041
4 1.70284140110016
5 2.24704647064209
6 1.82116889953613
7 2.29662990570068
8 1.41985857486725
9 1.30038142204285
10 0.740548729896545
11 2.76282906532288
12 2.4078369140625
13 2.99106884002686
14 0.484234869480133
15 0.950772285461426
16 1.63555586338043
17 1.72286903858185
18 1.84918904304504
19 0.545770108699799
20 1.22970151901245
21 1.79207098484039
22 0.869961261749268
23 1.00067114830017
24 1.14098036289215
25 0.55354779958725
26 1.2200585603714
27 0.908354878425598
28 1.52795779705048
29 2.33165979385376
30 1.32208514213562
31 2.35832738876343
32 1.34974277019501
33 3.0761559009552
34 0.743091940879822
35 1.20128202438354
36 0.840166211128235
37 1.43583059310913
38 1.4594863653183
39 0.396053075790405
40 0.725504040718079
41 1.37858200073242
42 0.448578327894211
43 2.00592756271362
44 1.76618945598602
45 1.58055675029755
46 2.72026085853577
47 2.12745380401611
48 0.446991801261902
49 0.369144171476364
50 1.5640789270401
51 1.08488035202026
52 1.61692225933075
53 2.26110672950745
54 1.54559993743896
55 0.964318215847015
56 0.464856445789337
57 0.881598055362701
58 0.985526382923126
59 0.172752022743225
60 1.57314908504486
61 0.849768042564392
62 0.369219750165939
63 1.58153867721558
64 1.69545841217041
65 0.903922200202942
66 1.55092787742615
67 2.02639579772949
68 1.36438381671906
69 1.89355063438416
70 1.59132361412048
71 1.02900183200836
72 0.489230692386627
73 1.26681733131409
74 3.51573896408081
75 1.50808596611023
76 1.17921841144562
77 1.82505548000336
78 1.04785525798798
79 0.846673011779785
80 0.841594398021698
81 1.14877259731293
82 1.23843884468079
83 1.67801535129547
84 0.470785886049271
85 1.48722553253174
86 1.16082632541656
87 1.31765699386597
88 1.55303680896759
89 1.93100392818451
90 1.05674159526825
91 1.30409908294678
92 2.06405878067017
93 1.46753752231598
94 1.19331693649292
95 2.2450647354126
96 0.643195629119873
97 1.720827460289
98 1.11767327785492
99 0.816306352615356
};
\addlegendentry{estimated}
\end{axis}

\end{tikzpicture}
         }
        \end{subfigure}
        \begin{subfigure}[b]{0.32\textwidth}
        \resizebox{\linewidth}{0.8\linewidth}{
\begin{tikzpicture}

\definecolor{darkgray176}{RGB}{176,176,176}
\definecolor{lightgray204}{RGB}{204,204,204}
\definecolor{steelblue31119180}{RGB}{31,119,180}

\begin{axis}[
legend cell align={left},
legend style={
  font=\tiny,
  fill opacity=0.8,
  draw opacity=1,
  text opacity=1,
  at={(0.03,0.97)},
  anchor=north west,
  draw=lightgray204
},
scale only axis,
width=4cm,
height=3cm,
tick align=outside,
tick pos=left,
title={\fontsize{6}{6}\selectfont $N=5000\ \text{(PF)}$},
tick label style={font=\tiny},
every tick/.style={
black,
semithick,
},
x label style={at={(axis description cs:0.5,-0.1)},anchor=north,font=\tiny},
y label style={at={(axis description cs:-0.06,.5)},rotate=90,anchor=south,font=\tiny},
x grid style={darkgray176},
xlabel={index},
xmin=-4.95, xmax=103.95,
xtick style={color=black},
y grid style={darkgray176},
ylabel={$\nu(t, Z)$},
ymin=0, ymax=5,
ytick style={color=black}
]
\addplot [width=1pt, steelblue31119180, mark=*, mark size=0.5, mark options={solid}]
table {%
0 1.5803998708725
1 1.33550202846527
2 1.56831169128418
3 1.35660552978516
4 1.67359554767609
5 2.31658291816711
6 1.77639508247375
7 2.05915355682373
8 1.3549679517746
9 1.39965355396271
10 0.630053579807281
11 3.85166573524475
12 2.69694805145264
13 3.94658541679382
14 0.599072337150574
15 1.20898294448853
16 1.48052942752838
17 1.70718002319336
18 1.91135597229004
19 0.683144271373749
20 1.40443646907806
21 1.60692179203033
22 0.922011435031891
23 0.982581973075867
24 1.10471200942993
25 0.770239174365997
26 1.1936936378479
27 0.961292088031769
28 1.3913631439209
29 2.3709568977356
30 1.54486167430878
31 2.15229058265686
32 1.35555577278137
33 3.00449728965759
34 0.908643305301666
35 1.45915484428406
36 1.18871068954468
37 1.34153246879578
38 1.54231023788452
39 0.792442858219147
40 0.668309092521667
41 1.03730869293213
42 0.668369293212891
43 1.66741919517517
44 1.71565139293671
45 1.5300098657608
46 2.7632257938385
47 1.8025906085968
48 0.752088010311127
49 0.668880820274353
50 1.38437247276306
51 1.31735861301422
52 1.64243221282959
53 1.91407084465027
54 1.24209487438202
55 1.16776537895203
56 0.761925280094147
57 0.949998795986176
58 0.97191709280014
59 0.376778155565262
60 1.44549453258514
61 1.05457043647766
62 0.394052714109421
63 1.69112682342529
64 1.82295894622803
65 1.22094392776489
66 1.59055876731873
67 1.67814803123474
68 1.24474513530731
69 1.76935398578644
70 1.44282734394073
71 1.2813868522644
72 0.476826071739197
73 1.40479564666748
74 4.18656396865845
75 1.3455274105072
76 1.46364843845367
77 1.83773052692413
78 1.35060667991638
79 1.00683796405792
80 0.783329665660858
81 1.25577104091644
82 1.16983044147491
83 1.74855446815491
84 0.957872331142426
85 1.29508185386658
86 0.882132947444916
87 1.55416762828827
88 1.64668881893158
89 2.03826498985291
90 1.15617573261261
91 1.54844224452972
92 2.10270237922668
93 1.34900939464569
94 1.23071610927582
95 2.34977126121521
96 0.848920166492462
97 1.64094316959381
98 1.31309974193573
99 0.961935579776764
};
\addlegendentry{ground truth}
\addplot [width=1pt, red, opacity=0.7]
table {%
0 1.61169183254242
1 1.35346913337708
2 1.54914796352386
3 1.34038150310516
4 1.66806411743164
5 2.2279200553894
6 1.65838241577148
7 1.94123935699463
8 1.38878130912781
9 1.35380816459656
10 0.71142053604126
11 2.69603419303894
12 2.50882077217102
13 2.87483906745911
14 0.598593413829803
15 1.18356347084045
16 1.41311740875244
17 1.63503861427307
18 1.7953040599823
19 0.7613285779953
20 1.3535373210907
21 1.49374222755432
22 0.913654625415802
23 0.895472168922424
24 1.14922082424164
25 0.803572952747345
26 1.20976781845093
27 1.00814843177795
28 1.27285587787628
29 2.21610617637634
30 1.54459381103516
31 2.12943553924561
32 1.34946715831757
33 2.74569964408875
34 0.955883264541626
35 1.48588335514069
36 1.16549408435822
37 1.43464291095734
38 1.545201420784
39 0.932527363300323
40 0.728605031967163
41 1.08973908424377
42 0.65239542722702
43 1.6808670759201
44 1.71336078643799
45 1.50909209251404
46 2.7446711063385
47 1.73943829536438
48 0.741521000862122
49 0.648966312408447
50 1.48351633548737
51 1.3465850353241
52 1.60723805427551
53 1.84433579444885
54 1.21493029594421
55 1.16193449497223
56 0.770434498786926
57 0.933859527111053
58 0.962571322917938
59 0.342689037322998
60 1.58445060253143
61 1.11809909343719
62 0.46181508898735
63 1.64964354038239
64 1.87478852272034
65 1.13443422317505
66 1.62193441390991
67 1.70829379558563
68 1.20149397850037
69 1.81070482730865
70 1.44679987430573
71 1.29021203517914
72 0.57251238822937
73 1.42143797874451
74 3.23232507705688
75 1.29971385002136
76 1.4787825345993
77 1.78422367572784
78 1.34773564338684
79 0.989336371421814
80 0.810009896755219
81 1.22975671291351
82 1.18059277534485
83 1.76464641094208
84 1.09783184528351
85 1.39039480686188
86 0.940960645675659
87 1.59180998802185
88 1.61952221393585
89 2.05971932411194
90 1.40246105194092
91 1.59184682369232
92 2.04787969589233
93 1.39708077907562
94 1.205242395401
95 2.2276656627655
96 0.762388229370117
97 1.62482035160065
98 1.26078069210052
99 1.08035314083099
};
\addlegendentry{estimated}
\end{axis}

\end{tikzpicture}
        }
        \end{subfigure}
        \begin{subfigure}[b]{0.32\textwidth}
        \resizebox{\linewidth}{0.8\linewidth}{
\begin{tikzpicture}

\definecolor{darkgray176}{RGB}{176,176,176}
\definecolor{lightgray204}{RGB}{204,204,204}
\definecolor{steelblue31119180}{RGB}{31,119,180}

\begin{axis}[
legend cell align={left},
legend style={
  font=\tiny,
  fill opacity=0.8,
  draw opacity=1,
  text opacity=1,
  at={(0.03,0.97)},
  anchor=north west,
  draw=lightgray204
},
scale only axis,
width=4cm,
height=3cm,
tick align=outside,
tick pos=left,
title={\fontsize{6}{6}\selectfont $N=10000\ \text{(PF)}$},
tick label style={font=\tiny},
every tick/.style={
black,
semithick,
},
x label style={at={(axis description cs:0.5,-0.1)},anchor=north,font=\tiny},
y label style={at={(axis description cs:-0.06,.5)},rotate=90,anchor=south,font=\tiny},
x grid style={darkgray176},
xlabel={index},
xmin=-4.95, xmax=103.95,
xtick style={color=black},
y grid style={darkgray176},
ylabel={$\nu(t, Z)$},
ymin=0, ymax=5,
ytick style={color=black}
]
\addplot [width=1pt, steelblue31119180, mark=*, mark size=0.5, mark options={solid}]
table {%
0 1.5803998708725
1 1.33550202846527
2 1.56831169128418
3 1.35660552978516
4 1.67359554767609
5 2.31658291816711
6 1.77639508247375
7 2.05915355682373
8 1.3549679517746
9 1.39965355396271
10 0.630053579807281
11 3.85166573524475
12 2.69694805145264
13 3.94658541679382
14 0.599072337150574
15 1.20898294448853
16 1.48052942752838
17 1.70718002319336
18 1.91135597229004
19 0.683144271373749
20 1.40443646907806
21 1.60692179203033
22 0.922011435031891
23 0.982581973075867
24 1.10471200942993
25 0.770239174365997
26 1.1936936378479
27 0.961292088031769
28 1.3913631439209
29 2.3709568977356
30 1.54486167430878
31 2.15229058265686
32 1.35555577278137
33 3.00449728965759
34 0.908643305301666
35 1.45915484428406
36 1.18871068954468
37 1.34153246879578
38 1.54231023788452
39 0.792442858219147
40 0.668309092521667
41 1.03730869293213
42 0.668369293212891
43 1.66741919517517
44 1.71565139293671
45 1.5300098657608
46 2.7632257938385
47 1.8025906085968
48 0.752088010311127
49 0.668880820274353
50 1.38437247276306
51 1.31735861301422
52 1.64243221282959
53 1.91407084465027
54 1.24209487438202
55 1.16776537895203
56 0.761925280094147
57 0.949998795986176
58 0.97191709280014
59 0.376778155565262
60 1.44549453258514
61 1.05457043647766
62 0.394052714109421
63 1.69112682342529
64 1.82295894622803
65 1.22094392776489
66 1.59055876731873
67 1.67814803123474
68 1.24474513530731
69 1.76935398578644
70 1.44282734394073
71 1.2813868522644
72 0.476826071739197
73 1.40479564666748
74 4.18656396865845
75 1.3455274105072
76 1.46364843845367
77 1.83773052692413
78 1.35060667991638
79 1.00683796405792
80 0.783329665660858
81 1.25577104091644
82 1.16983044147491
83 1.74855446815491
84 0.957872331142426
85 1.29508185386658
86 0.882132947444916
87 1.55416762828827
88 1.64668881893158
89 2.03826498985291
90 1.15617573261261
91 1.54844224452972
92 2.10270237922668
93 1.34900939464569
94 1.23071610927582
95 2.34977126121521
96 0.848920166492462
97 1.64094316959381
98 1.31309974193573
99 0.961935579776764
};
\addlegendentry{ground truth}
\addplot [width=1pt, red, opacity=0.7]
table {%
0 1.63320815563202
1 1.26339435577393
2 1.49360692501068
3 1.37349796295166
4 1.5994473695755
5 2.19188642501831
6 1.77786326408386
7 2.03979229927063
8 1.26683855056763
9 1.44541573524475
10 0.618413627147675
11 4.16975736618042
12 2.58802628517151
13 4.06672239303589
14 0.616865992546082
15 1.19383478164673
16 1.41739630699158
17 1.67845940589905
18 1.93933284282684
19 0.748689949512482
20 1.41527378559113
21 1.61980819702148
22 0.957438468933105
23 1.01047730445862
24 0.990949928760529
25 0.825788140296936
26 1.13232183456421
27 0.856281042098999
28 1.43209671974182
29 2.36079573631287
30 1.57284951210022
31 2.08605074882507
32 1.32223844528198
33 2.92045259475708
34 0.844692766666412
35 1.47533512115479
36 1.26069629192352
37 1.30722749233246
38 1.44950520992279
39 0.795714974403381
40 0.713843584060669
41 1.03243792057037
42 0.671241581439972
43 1.60867714881897
44 1.66476702690125
45 1.46602940559387
46 2.70830225944519
47 1.69029664993286
48 0.802866458892822
49 0.74275940656662
50 1.28350901603699
51 1.37070202827454
52 1.64175605773926
53 1.87594199180603
54 1.17252945899963
55 1.08627676963806
56 0.791209399700165
57 0.863278567790985
58 0.91286563873291
59 0.532885193824768
60 1.40233778953552
61 0.963800728321075
62 0.493080377578735
63 1.65378534793854
64 1.6774275302887
65 1.24718856811523
66 1.51236867904663
67 1.57029032707214
68 1.15611672401428
69 1.73740696907043
70 1.3687196969986
71 1.30913519859314
72 0.617701232433319
73 1.32711005210876
74 4.26269912719727
75 1.25169432163239
76 1.48426330089569
77 1.79199182987213
78 1.36223995685577
79 1.01605319976807
80 0.737365245819092
81 1.25444006919861
82 1.12206625938416
83 1.60676050186157
84 1.0251852273941
85 1.27342820167542
86 0.87004429101944
87 1.51110410690308
88 1.58871257305145
89 1.94062924385071
90 1.09692752361298
91 1.48531925678253
92 1.98443162441254
93 1.3261901140213
94 1.22760415077209
95 2.21852374076843
96 0.925139725208282
97 1.54944908618927
98 1.33462047576904
99 0.964331924915314
};
\addlegendentry{estimated}
\end{axis}

\end{tikzpicture}
        }
        \end{subfigure}
        \begin{subfigure}[b]{0.32\textwidth}
        \resizebox{\linewidth}{0.8\linewidth}{
\begin{tikzpicture}

\definecolor{darkgray176}{RGB}{176,176,176}
\definecolor{lightgray204}{RGB}{204,204,204}
\definecolor{steelblue31119180}{RGB}{31,119,180}

\begin{axis}[
legend cell align={left},
legend style={
  font=\tiny,
  fill opacity=0.8,
  draw opacity=1,
  text opacity=1,
  at={(0.03,0.97)},
  anchor=north west,
  draw=lightgray204
},
scale only axis,
width=4cm,
height=3cm,
tick align=outside,
tick pos=left,
title={\fontsize{6}{6}\selectfont $N=1000\ \text{(FN)}$},
tick label style={font=\tiny},
every tick/.style={
black,
semithick,
},
x label style={at={(axis description cs:0.5,-0.1)},anchor=north,font=\tiny},
y label style={at={(axis description cs:-0.06,.5)},rotate=90,anchor=south,font=\tiny},
x grid style={darkgray176},
xlabel={index},
xmin=-4.95, xmax=103.95,
xtick style={color=black},
y grid style={darkgray176},
ylabel={$\nu(t, Z)$},
ymin=0, ymax=5,
ytick style={color=black}
]
\addplot [width=1pt, steelblue31119180, mark=*, mark size=0.5, mark options={solid}]
table {%
0 1.5803998708725
1 1.33550202846527
2 1.56831169128418
3 1.35660552978516
4 1.67359554767609
5 2.31658291816711
6 1.77639508247375
7 2.05915355682373
8 1.3549679517746
9 1.39965355396271
10 0.630053579807281
11 3.85166573524475
12 2.69694805145264
13 3.94658541679382
14 0.599072337150574
15 1.20898294448853
16 1.48052942752838
17 1.70718002319336
18 1.91135597229004
19 0.683144271373749
20 1.40443646907806
21 1.60692179203033
22 0.922011435031891
23 0.982581973075867
24 1.10471200942993
25 0.770239174365997
26 1.1936936378479
27 0.961292088031769
28 1.3913631439209
29 2.3709568977356
30 1.54486167430878
31 2.15229058265686
32 1.35555577278137
33 3.00449728965759
34 0.908643305301666
35 1.45915484428406
36 1.18871068954468
37 1.34153246879578
38 1.54231023788452
39 0.792442858219147
40 0.668309092521667
41 1.03730869293213
42 0.668369293212891
43 1.66741919517517
44 1.71565139293671
45 1.5300098657608
46 2.7632257938385
47 1.8025906085968
48 0.752088010311127
49 0.668880820274353
50 1.38437247276306
51 1.31735861301422
52 1.64243221282959
53 1.91407084465027
54 1.24209487438202
55 1.16776537895203
56 0.761925280094147
57 0.949998795986176
58 0.97191709280014
59 0.376778155565262
60 1.44549453258514
61 1.05457043647766
62 0.394052714109421
63 1.69112682342529
64 1.82295894622803
65 1.22094392776489
66 1.59055876731873
67 1.67814803123474
68 1.24474513530731
69 1.76935398578644
70 1.44282734394073
71 1.2813868522644
72 0.476826071739197
73 1.40479564666748
74 4.18656396865845
75 1.3455274105072
76 1.46364843845367
77 1.83773052692413
78 1.35060667991638
79 1.00683796405792
80 0.783329665660858
81 1.25577104091644
82 1.16983044147491
83 1.74855446815491
84 0.957872331142426
85 1.29508185386658
86 0.882132947444916
87 1.55416762828827
88 1.64668881893158
89 2.03826498985291
90 1.15617573261261
91 1.54844224452972
92 2.10270237922668
93 1.34900939464569
94 1.23071610927582
95 2.34977126121521
96 0.848920166492462
97 1.64094316959381
98 1.31309974193573
99 0.961935579776764
};
\addlegendentry{ground truth}
\addplot [width=1pt, red, opacity=0.7]
table {%
0 1.57982742786407
1 1.30859661102295
2 1.53048992156982
3 1.29412043094635
4 1.61064636707306
5 2.16443228721619
6 1.62435603141785
7 1.82374107837677
8 1.30134546756744
9 1.29593288898468
10 0.745382845401764
11 3.54573583602905
12 2.43141961097717
13 3.47103357315063
14 0.625240385532379
15 1.14267575740814
16 1.34909403324127
17 1.5699315071106
18 1.82498240470886
19 0.715458810329437
20 1.3080552816391
21 1.43951952457428
22 0.79462331533432
23 0.890769183635712
24 1.16482794284821
25 0.764600336551666
26 1.14263164997101
27 0.975181877613068
28 1.21395099163055
29 2.09771132469177
30 1.4419709444046
31 1.95716786384583
32 1.32032930850983
33 2.74131655693054
34 0.985166013240814
35 1.42472016811371
36 1.10726773738861
37 1.3594583272934
38 1.49681234359741
39 0.88981682062149
40 0.684394180774689
41 1.08833372592926
42 0.658737003803253
43 1.60764002799988
44 1.59029018878937
45 1.51202619075775
46 2.49796986579895
47 1.713094830513
48 0.727181792259216
49 0.646669149398804
50 1.37640297412872
51 1.27958738803864
52 1.48917543888092
53 1.7244234085083
54 1.23255217075348
55 1.2249094247818
56 0.77898758649826
57 0.960130095481873
58 0.982890069484711
59 0.404571026563644
60 1.41737449169159
61 1.05930805206299
62 0.61262708902359
63 1.57019889354706
64 1.76102721691132
65 1.10762917995453
66 1.53148448467255
67 1.62229585647583
68 1.2098308801651
69 1.60655915737152
70 1.37344086170197
71 1.2446802854538
72 0.565797746181488
73 1.37388598918915
74 3.70939302444458
75 1.2756404876709
76 1.40222501754761
77 1.87154448032379
78 1.27328062057495
79 0.945150971412659
80 0.885943174362183
81 1.19008195400238
82 1.26469016075134
83 1.78513264656067
84 0.984847724437714
85 1.3425635099411
86 0.950227022171021
87 1.44857180118561
88 1.55309212207794
89 1.86455678939819
90 1.22537040710449
91 1.49142944812775
92 1.95692610740662
93 1.36952233314514
94 1.18069732189178
95 2.15171909332275
96 0.756003797054291
97 1.5149085521698
98 1.2169269323349
99 0.965559542179108
};
\addlegendentry{estimated}
\end{axis}

\end{tikzpicture}
        }
        \end{subfigure}
        \begin{subfigure}[b]{0.32\textwidth}
        \resizebox{\linewidth}{0.8\linewidth}{
\begin{tikzpicture}

\definecolor{darkgray176}{RGB}{176,176,176}
\definecolor{lightgray204}{RGB}{204,204,204}
\definecolor{steelblue31119180}{RGB}{31,119,180}

\begin{axis}[
legend cell align={left},
legend style={
  font=\tiny,
  fill opacity=0.8,
  draw opacity=1,
  text opacity=1,
  at={(0.03,0.97)},
  anchor=north west,
  draw=lightgray204
},
scale only axis,
width=4cm,
height=3cm,
tick align=outside,
tick pos=left,
title={\fontsize{6}{6}\selectfont $N=5000\ \text{(FN)}$},
tick label style={font=\tiny},
every tick/.style={
black,
semithick,
},
x label style={at={(axis description cs:0.5,-0.1)},anchor=north,font=\tiny},
y label style={at={(axis description cs:-0.06,.5)},rotate=90,anchor=south,font=\tiny},
x grid style={darkgray176},
xlabel={index},
xmin=-4.95, xmax=103.95,
xtick style={color=black},
y grid style={darkgray176},
ylabel={$\nu(t, Z)$},
ymin=0, ymax=5,
ytick style={color=black}
]
\addplot [width=1pt, steelblue31119180, mark=*, mark size=0.5, mark options={solid}]
table {%
0 1.5803998708725
1 1.33550202846527
2 1.56831169128418
3 1.35660552978516
4 1.67359554767609
5 2.31658291816711
6 1.77639508247375
7 2.05915355682373
8 1.3549679517746
9 1.39965355396271
10 0.630053579807281
11 3.85166573524475
12 2.69694805145264
13 3.94658541679382
14 0.599072337150574
15 1.20898294448853
16 1.48052942752838
17 1.70718002319336
18 1.91135597229004
19 0.683144271373749
20 1.40443646907806
21 1.60692179203033
22 0.922011435031891
23 0.982581973075867
24 1.10471200942993
25 0.770239174365997
26 1.1936936378479
27 0.961292088031769
28 1.3913631439209
29 2.3709568977356
30 1.54486167430878
31 2.15229058265686
32 1.35555577278137
33 3.00449728965759
34 0.908643305301666
35 1.45915484428406
36 1.18871068954468
37 1.34153246879578
38 1.54231023788452
39 0.792442858219147
40 0.668309092521667
41 1.03730869293213
42 0.668369293212891
43 1.66741919517517
44 1.71565139293671
45 1.5300098657608
46 2.7632257938385
47 1.8025906085968
48 0.752088010311127
49 0.668880820274353
50 1.38437247276306
51 1.31735861301422
52 1.64243221282959
53 1.91407084465027
54 1.24209487438202
55 1.16776537895203
56 0.761925280094147
57 0.949998795986176
58 0.97191709280014
59 0.376778155565262
60 1.44549453258514
61 1.05457043647766
62 0.394052714109421
63 1.69112682342529
64 1.82295894622803
65 1.22094392776489
66 1.59055876731873
67 1.67814803123474
68 1.24474513530731
69 1.76935398578644
70 1.44282734394073
71 1.2813868522644
72 0.476826071739197
73 1.40479564666748
74 4.18656396865845
75 1.3455274105072
76 1.46364843845367
77 1.83773052692413
78 1.35060667991638
79 1.00683796405792
80 0.783329665660858
81 1.25577104091644
82 1.16983044147491
83 1.74855446815491
84 0.957872331142426
85 1.29508185386658
86 0.882132947444916
87 1.55416762828827
88 1.64668881893158
89 2.03826498985291
90 1.15617573261261
91 1.54844224452972
92 2.10270237922668
93 1.34900939464569
94 1.23071610927582
95 2.34977126121521
96 0.848920166492462
97 1.64094316959381
98 1.31309974193573
99 0.961935579776764
};
\addlegendentry{ground truth}
\addplot [width=1pt, red, opacity=0.7]
table {%
0 1.57982742786407
1 1.30859661102295
2 1.53048992156982
3 1.29412043094635
4 1.61064636707306
5 2.16443228721619
6 1.62435603141785
7 1.82374107837677
8 1.30134546756744
9 1.29593288898468
10 0.745382845401764
11 3.54573583602905
12 2.43141961097717
13 3.47103357315063
14 0.625240385532379
15 1.14267575740814
16 1.34909403324127
17 1.5699315071106
18 1.82498240470886
19 0.715458810329437
20 1.3080552816391
21 1.43951952457428
22 0.79462331533432
23 0.890769183635712
24 1.16482794284821
25 0.764600336551666
26 1.14263164997101
27 0.975181877613068
28 1.21395099163055
29 2.09771132469177
30 1.4419709444046
31 1.95716786384583
32 1.32032930850983
33 2.74131655693054
34 0.985166013240814
35 1.42472016811371
36 1.10726773738861
37 1.3594583272934
38 1.49681234359741
39 0.88981682062149
40 0.684394180774689
41 1.08833372592926
42 0.658737003803253
43 1.60764002799988
44 1.59029018878937
45 1.51202619075775
46 2.49796986579895
47 1.713094830513
48 0.727181792259216
49 0.646669149398804
50 1.37640297412872
51 1.27958738803864
52 1.48917543888092
53 1.7244234085083
54 1.23255217075348
55 1.2249094247818
56 0.77898758649826
57 0.960130095481873
58 0.982890069484711
59 0.404571026563644
60 1.41737449169159
61 1.05930805206299
62 0.61262708902359
63 1.57019889354706
64 1.76102721691132
65 1.10762917995453
66 1.53148448467255
67 1.62229585647583
68 1.2098308801651
69 1.60655915737152
70 1.37344086170197
71 1.2446802854538
72 0.565797746181488
73 1.37388598918915
74 3.70939302444458
75 1.2756404876709
76 1.40222501754761
77 1.87154448032379
78 1.27328062057495
79 0.945150971412659
80 0.885943174362183
81 1.19008195400238
82 1.26469016075134
83 1.78513264656067
84 0.984847724437714
85 1.3425635099411
86 0.950227022171021
87 1.44857180118561
88 1.55309212207794
89 1.86455678939819
90 1.22537040710449
91 1.49142944812775
92 1.95692610740662
93 1.36952233314514
94 1.18069732189178
95 2.15171909332275
96 0.756003797054291
97 1.5149085521698
98 1.2169269323349
99 0.965559542179108
};
\addlegendentry{estimated}
\end{axis}

\end{tikzpicture}
        }
        \end{subfigure}
        \begin{subfigure}[b]{0.32\textwidth}
        \resizebox{\linewidth}{0.8\linewidth}{
\begin{tikzpicture}

\definecolor{darkgray176}{RGB}{176,176,176}
\definecolor{lightgray204}{RGB}{204,204,204}
\definecolor{steelblue31119180}{RGB}{31,119,180}

\begin{axis}[
legend cell align={left},
legend style={
  font=\tiny,
  fill opacity=0.8,
  draw opacity=1,
  text opacity=1,
  at={(0.03,0.97)},
  anchor=north west,
  draw=lightgray204
},
scale only axis,
width=4cm,
height=3cm,
tick align=outside,
tick pos=left,
title={\fontsize{6}{6}\selectfont $N=10000\ \text{(FN)}$},
tick label style={font=\tiny},
every tick/.style={
black,
semithick,
},
x label style={at={(axis description cs:0.5,-0.1)},anchor=north,font=\tiny},
y label style={at={(axis description cs:-0.06,.5)},rotate=90,anchor=south,font=\tiny},
x grid style={darkgray176},
xlabel={index},
xmin=-4.95, xmax=103.95,
xtick style={color=black},
y grid style={darkgray176},
ylabel={$\nu(t, Z)$},
ymin=0, ymax=5,
ytick style={color=black}
]
\addplot [width=1pt, steelblue31119180, mark=*, mark size=0.5, mark options={solid}]
table {%
0 1.5803998708725
1 1.33550202846527
2 1.56831169128418
3 1.35660552978516
4 1.67359554767609
5 2.31658291816711
6 1.77639508247375
7 2.05915355682373
8 1.3549679517746
9 1.39965355396271
10 0.630053579807281
11 3.85166573524475
12 2.69694805145264
13 3.94658541679382
14 0.599072337150574
15 1.20898294448853
16 1.48052942752838
17 1.70718002319336
18 1.91135597229004
19 0.683144271373749
20 1.40443646907806
21 1.60692179203033
22 0.922011435031891
23 0.982581973075867
24 1.10471200942993
25 0.770239174365997
26 1.1936936378479
27 0.961292088031769
28 1.3913631439209
29 2.3709568977356
30 1.54486167430878
31 2.15229058265686
32 1.35555577278137
33 3.00449728965759
34 0.908643305301666
35 1.45915484428406
36 1.18871068954468
37 1.34153246879578
38 1.54231023788452
39 0.792442858219147
40 0.668309092521667
41 1.03730869293213
42 0.668369293212891
43 1.66741919517517
44 1.71565139293671
45 1.5300098657608
46 2.7632257938385
47 1.8025906085968
48 0.752088010311127
49 0.668880820274353
50 1.38437247276306
51 1.31735861301422
52 1.64243221282959
53 1.91407084465027
54 1.24209487438202
55 1.16776537895203
56 0.761925280094147
57 0.949998795986176
58 0.97191709280014
59 0.376778155565262
60 1.44549453258514
61 1.05457043647766
62 0.394052714109421
63 1.69112682342529
64 1.82295894622803
65 1.22094392776489
66 1.59055876731873
67 1.67814803123474
68 1.24474513530731
69 1.76935398578644
70 1.44282734394073
71 1.2813868522644
72 0.476826071739197
73 1.40479564666748
74 4.18656396865845
75 1.3455274105072
76 1.46364843845367
77 1.83773052692413
78 1.35060667991638
79 1.00683796405792
80 0.783329665660858
81 1.25577104091644
82 1.16983044147491
83 1.74855446815491
84 0.957872331142426
85 1.29508185386658
86 0.882132947444916
87 1.55416762828827
88 1.64668881893158
89 2.03826498985291
90 1.15617573261261
91 1.54844224452972
92 2.10270237922668
93 1.34900939464569
94 1.23071610927582
95 2.34977126121521
96 0.848920166492462
97 1.64094316959381
98 1.31309974193573
99 0.961935579776764
};
\addlegendentry{ground truth}
\addplot [width=1pt, red, opacity=0.7]
table {%
0 1.63320815563202
1 1.26339435577393
2 1.49360692501068
3 1.37349796295166
4 1.5994473695755
5 2.19188642501831
6 1.77786326408386
7 2.03979229927063
8 1.26683855056763
9 1.44541573524475
10 0.618413627147675
11 4.16975736618042
12 2.58802628517151
13 4.06672239303589
14 0.616865992546082
15 1.19383478164673
16 1.41739630699158
17 1.67845940589905
18 1.93933284282684
19 0.748689949512482
20 1.41527378559113
21 1.61980819702148
22 0.957438468933105
23 1.01047730445862
24 0.990949928760529
25 0.825788140296936
26 1.13232183456421
27 0.856281042098999
28 1.43209671974182
29 2.36079573631287
30 1.57284951210022
31 2.08605074882507
32 1.32223844528198
33 2.92045259475708
34 0.844692766666412
35 1.47533512115479
36 1.26069629192352
37 1.30722749233246
38 1.44950520992279
39 0.795714974403381
40 0.713843584060669
41 1.03243792057037
42 0.671241581439972
43 1.60867714881897
44 1.66476702690125
45 1.46602940559387
46 2.70830225944519
47 1.69029664993286
48 0.802866458892822
49 0.74275940656662
50 1.28350901603699
51 1.37070202827454
52 1.64175605773926
53 1.87594199180603
54 1.17252945899963
55 1.08627676963806
56 0.791209399700165
57 0.863278567790985
58 0.91286563873291
59 0.532885193824768
60 1.40233778953552
61 0.963800728321075
62 0.493080377578735
63 1.65378534793854
64 1.6774275302887
65 1.24718856811523
66 1.51236867904663
67 1.57029032707214
68 1.15611672401428
69 1.73740696907043
70 1.3687196969986
71 1.30913519859314
72 0.617701232433319
73 1.32711005210876
74 4.26269912719727
75 1.25169432163239
76 1.48426330089569
77 1.79199182987213
78 1.36223995685577
79 1.01605319976807
80 0.737365245819092
81 1.25444006919861
82 1.12206625938416
83 1.60676050186157
84 1.0251852273941
85 1.27342820167542
86 0.87004429101944
87 1.51110410690308
88 1.58871257305145
89 1.94062924385071
90 1.09692752361298
91 1.48531925678253
92 1.98443162441254
93 1.3261901140213
94 1.22760415077209
95 2.21852374076843
96 0.925139725208282
97 1.54944908618927
98 1.33462047576904
99 0.964331924915314
};
\addlegendentry{estimated}
\end{axis}

\end{tikzpicture}
        }
        \end{subfigure}
        \caption{Visualizations of synthetic data results under the NFM framework. The plots in the first row compare the empirical estimates of the nonparametric component $\nu(t, Z)$ against its true value evaluated on $100$ hold-out points, under the PF scheme. The plots in the second row are obtained using the FN scheme, with analogous semantics to the first row.}
        \label{fig: synthesis_log_hazard}
    \end{figure}
    
    \subsection{Real-world data experiments}\label{sec: benchmark_data}
    To answer RQ$2$, we conduct extensive empirical assessments over $6$ benchmark datasets, comprising five survival datasets and one non-survival dataset. The survival datasets include the Molecular Taxonomy of Breast Cancer International Consortium (METABRIC) \cite{curtis2012metabric}, the Rotterdam tumor bank and German Breast Cancer Study Group (RotGBSG)\cite{knaus1995support}, the Assay Of Serum Free Light Chain (FLCHAIN) \cite{dispenzieri2012flchain}, the Study to Understand Prognoses Preferences Outcomes and Risks of Treatment (SUPPORT) \cite{knaus1995support}, and the Medical Information Mart for Intensive Care (MIMIC-III) \cite{johnson2016mimic}. For all the survival datasets, the event of interest is defined as the mortality after admission. In our experiments, we view METABRIC, RotGBSG, FLCHAIN, and SUPPORT as small-scale datasets and MIMIC-III as a moderate-scale dataset. We additionally use the KKBOX dataset \cite{kvamme2019time} as a large-scale evaluation. In this dataset, an event time is observed if a customer churns from the KKBOX platform. We summarize the basic statistics of all the datasets in table \ref{tab: datasets}.\par
    \textbf{Baselines} We compare NFM with $12$ baselines. The first one is linear \ph model \cite{cox1972regression}. Gradient Boosting Machine (GBM) \cite{friedman2001greedy, chen2016xgboost} and Random Survival Forests (RSF) \cite{ishwaran2008random} are two tree-based nonparametric survival regression methods. DeepSurv \cite{katzman2018deepsurv} and CoxTime \cite{kvamme2019time} are two models that adopt neural variants of partial likelihood as objectives. SuMo-net \cite{rindt2022a} is a neural variant of NHR. We additionally chose six latest state-of-the-art neural survival models: DeepHit \cite{lee2018deephit}, SurvNode \cite{groha2020general}, DeepEH \cite{zhong2021deep}, DCM \cite{nagpal2021deep}, DeSurv \cite{danks2022derivative} and SODEN \cite{tang2022soden}. Among the chosen baselines, DeepSurv and SuMo-net are viewed as implementations of neural \ph and neural NHR and are therefore of particular interest for the empirical verification of the efficacy of frailty. \par
    \begin{table}[]
        \centering
        \caption{Survival prediction results measured in IBS and INBLL metric (\%) on four small-scale survival datasets. In each column, the \textbf{boldfaced} score denotes the best result and the \underline{underlined} score represents the second-best result (both in mean).}
        \resizebox{\textwidth}{!}{%
            \begin{tabular}{l c c c c c c c c}
     \toprule
     Model & \multicolumn{2}{c}{METABRIC} & \multicolumn{2}{c}{RotGBSG}  &
     \multicolumn{2}{c}{FLCHAIN} & \multicolumn{2}{c}{SUPPORT} \\
     \cmidrule{2-9}
           & IBS & INBLL & IBS & INBLL & IBS & INBLL & IBS & INBLL \\
     \midrule
     \ph   & \result{16.46}{0.90} & \result{49.57}{2.66} & \result{18.25}{0.44} & \result{53.76}{1.11} & \result{10.05}{0.38} & \result{33.18}{1.16} & \result{20.54}{0.38} & \result{59.58}{0.86} \\
     GBM & \result{16.61}{0.82} & \result{49.87}{2.44} & \result{17.83}{0.44} & \result{52.78}{1.11} & \results{9.98}{0.37} & \results{32.88}{1.05} & \result{19.18}{0.39} & \result{56.46}{0.10} \\
     RSF & \result{16.62}{0.64} & \result{49.61}{1.54} & \result{17.89}{0.42} & \result{52.77}{1.01} & \resultf{9.96}{0.37} & \result{32.92}{1.05} & \results{19.11}{0.40} & \results{56.28}{1.00} \\
     DeepSurv & \result{16.55}{0.93} & \result{49.85}{3.02} & \result{17.80}{0.49} & \result{52.62}{1.25} & \result{10.09}{0.38} & \result{33.28}{1.15} & \result{19.20}{0.41} & \result{56.48}{1.08} \\
     CoxTime  & \result{16.54}{0.83} & \result{49.67}{2.67} & \result{17.80}{0.58} & \result{52.56}{1.47} & \result{10.28}{0.45} & \result{34.18}{1.53} & \result{19.17}{0.40} & \result{56.45}{1.10} \\
     DeepHit  & \result{17.50}{0.83} & \result{52.10}{2.16} & \result{19.61}{0.38} & \result{56.67}{1.10} & \result{11.83}{0.39} & \result{37.72}{1.02} & \result{20.66}{0.32} & \result{60.06}{0.72} \\
     DeepEH & \result{16.56}{0.65} & \result{49.42}{1.53} & \result{17.62}{0.52} & \results{52.08}{1.27} & \result{10.11}{0.37} & \result{33.30}{1.10} & \result{19.30}{0.39} & \result{56.67}{0.94} \\
     SuMo-net & \result{16.49}{0.83} & \result{49.74}{2.21} & \result{17.77}{0.47} & \result{52.62}{1.11} & \result{10.07}{0.40} & \result{33.20}{1.10} & \result{19.40}{0.38} & \result{56.87}{0.96} \\
     SODEN & \result{16.52}{0.63} & \result{49.39}{1.97} & \resultf{17.05}{0.63} & \resultf{50.45}{1.97} & \result{10.13}{0.24} & \result{33.37}{0.57} & \result{19.07}{0.50} & \result{56.15}{1.35} \\
     SurvNode & \result{16.67}{1.32} & \result{49.73}{3.89} & \result{17.42}{0.53} & \result{51.70}{1.16} & \result{10.40}{0.29} & \result{34.37}{1.03} & \result{19.58}{0.34} & \result{57.49}{0.84} \\
     DCM & \result{16.58}{0.87} & \result{49.48}{2.23} & \result{17.66}{0.54} & \result{52.26}{1.23} & \result{10.13}{0.50} & \result{33.40}{1.38} & \result{19.29}{0.42} & \result{56.68}{1.09} \\
     DeSurv & \result{16.71}{0.75} & \result{49.61}{2.15} & \result{17.98}{0.46} & \result{53.23}{1.15} & \result{10.06}{0.62} & \result{33.18}{1.93} & \result{19.50}{0.40} & \result{57.28}{0.89} \\
     \midrule
     \textbf{NFM-PF} & \results{16.33}{0.75} & \results{49.07}{1.96} & \results{17.60}{0.55} & \result{52.12}{1.34} & \resultf{9.96}{0.39} & \resultf{32.84}{1.15} &  \result{19.14}{0.39} & \result{56.35}{1.00} \\
     \textbf{NFM-FN} & \resultf{16.11}{0.81} & \resultf{48.21}{2.04} & \result{17.66}{0.52} & \result{52.41}{1.22} & \result{10.05}{0.39} & \result{33.11}{1.10} & \resultf{18.97}{0.60} & \resultf{55.87}{1.50} \\
     \bottomrule
\end{tabular}
        }
        \label{tab: survival_results}
    \end{table}
    \begin{table}[]
        \centering
        \small
        \caption{Survival prediction results measured in IBS and INBLL metric (\%) on two larger datasets. In each column, the \textbf{boldfaced} score denotes the best result and the \underline{underlined} score represents the second-best result (both in mean). Two models are not reported, namely SODEN and DeepEH, as we found empirically that their computational/memory cost is significantly worse than the rest, and we fail to obtain reasonable performances over the two datasets for these two models.}
            \begin{tabular}{l c c c c}
    \toprule
    Model           & \multicolumn{2}{c}{MIMIC-III} & \multicolumn{2}{c}{KKBOX}\\
    \cmidrule{2-5}
                    & IBS               & INBLL & IBS               & INBLL \\
    \midrule
    \ph             & \result{20.40}{0.00} & \result{60.02}{0.00} & \result{12.60}{0.00}    & \result{39.40}{0.00} \\
    GBM             & \result{17.70}{0.00} & \result{52.30}{0.00} & \result{11.81}{0.00}    & \result{38.15}{0.00} \\
    RSF             & \result{17.79}{0.19} & \result{53.34}{0.41} & \result{14.46}{0.00} & \result{44.39}{0.00} \\
    DeepSurv        & \result{18.58}{0.92} & \result{55.98}{2.43} & \result{11.31}{0.05} & \result{35.28}{0.15}\\
    CoxTime         & \result{17.68}{1.36} & \result{52.08}{3.06} & \results{10.70}{0.06} & \results{33.10}{0.21}\\
    DeepHit         & \result{19.80}{1.31} & \result{59.03}{4.20} & \result{16.00}{0.34} & \result{48.64}{1.04}\\
    SuMo-net        & \result{18.62}{1.23} & \result{54.51}{2.97} & \result{11.58}{0.11} & \result{36.61}{0.28} \\
    DCM             & \result{18.02}{0.49} & \result{52.83}{0.94} & \result{10.71}{0.11} & \result{33.24}{0.06} \\
    DeSurv          & \result{18.19}{0.65} & \result{54.69}{2.83} & \result{10.77}{0.21} & \result{33.22}{0.10} \\
    \midrule
    \textbf{NFM-PF} & \resultf{16.28}{0.36} & \resultf{49.18}{0.92} & \result{11.02}{0.11}    & \result{35.10}{0.22} \\
    \textbf{NFM-FN} & \results{17.47}{0.45} & \results{51.48}{1.23} & \resultf{10.63}{0.08}    & \resultf{32.81}{0.14} \\
    \bottomrule
\end{tabular}
        \label{tab: kkbox}
    \end{table}
    \textbf{Evaluation strategy} We use two standard metrics in survival predictions for evaluating model performance: integrated Brier score (IBS) and integrated negative binomial log-likelihood (INBLL). Both metrics are derived from the following:
    \begin{align}
        \begin{aligned}
            \mathcal{S}(\ell, t_1, t_2) =\int_{t_2}^{t_1}\dfrac{1}{n}\sum_{i=1}^n \left[\dfrac{\ell(0, \widehat{S}(t|Z_i)) I(T_i \le t, \delta_i = 1)}{\widehat{S}_C(T_i)} + \dfrac{\ell(1, \widehat{S}(t|Z_i))I(T_i > t)}{\widehat{S}_C(t)}\right] dt.
        \end{aligned}
    \end{align}
    Where $\widehat{S}_C(t)$ is an estimate of the survival function $S_C(t)$ of the censoring variable, obtained by the Kaplan-Meier estimate \cite{kaplan1958nonparametric} of the censored observations on the test data. $\ell: \{0, 1\} \times [0, 1] \mapsto \mathbb{R}^+$ is some proper loss function for binary classification \cite{gneiting2007strictly}. The IBS metric corresponds to $\ell$ being the square loss, and the INBLL metric corresponds to $\ell$ being the negative binomial (Bernoulli) log-likelihood \cite{graf1999assessment}. Both IBS and INBLL are proper scoring rules if the censoring times and survival times are independent. 
    \footnote{Otherwise, one may pose a covariate-dependent model on the censoring time and use $\widehat{S}_C(t|Z)$ instead of $\widehat{S}_C(t)$. We adopt the Kaplan-Meier approach since it's still the prevailing practice in evaluations of survival predictions.} We additionally report the result of another widely used metric, the concordance index (C-index), in appendix \ref{sec: additional_experiments}.
    Since all the survival datasets do not have standard train/test splits, we follow previous practice \cite{zhong2021deep} that uses $5$-fold cross-validation (CV): $1$ fold is for testing, and $20\%$ of the rest is held out for validation. In our experiments, we observed that a single random split into $5$ folds does not produce stable results for most survival datasets. Therefore we perform $10$ different CV runs for each survival dataset and report average metrics as well as their standard deviations. For the KKBOX dataset, we use the standard train/valid/test splits that are available via the \texttt{pycox} package \cite{kvamme2019time} and report results based on $10$ trial runs. \par
    \textbf{Experimental setup} We follow standard preprocessing strategies \cite{katzman2018deepsurv, kvamme2019time, zhong2021deep} that standardize continuous features into zero mean and unit variance, and do one-hot encodings for all categorical features. 
    We adopt MLP with ReLU activation for all function approximators, including $\widehat{h}$, $\widehat{m}$ in PF scheme, and $\widehat{\nu}$ in FN scheme, across all datasets, with the number of layers (depth) and the number of hidden units (width) within each layer being tunable. We tune the frailty transform over several standard choices: gamma frailty, Box-Cox transformation frailty and $\text{IGG}(\alpha)$ frailty, with their precise forms detailed in appendix \ref{sec: public_data_details}. 
    A more detailed description of the tuning procedure, as well as training configurations for baseline models, are reported in appendix \ref{sec: public_data_details}.\par
    \textbf{Results} We report experimental results of small-scale datasets in table \ref{tab: survival_results}, and results of two larger datasets in table \ref{tab: kkbox}. The proposed NFM framework achieves competitive performance which is comparable to the other state-of-the-art models. In particular, NFM attains best performance in mean on $5$ of the $6$ datasets, and is statistically significantly better over all the baselines at $0.05$ empirical level on the MIMIC-III dataset. \par
    \textbf{Ablation on the benefits of frailty} to better understand the additional benefits of introducing the frailty formulation, we compute the (relative) performance gain of NFM-PF and NFM-FN, against their non-frailty counterparts, namely DeepSurv \cite{katzman2018deepsurv} and SuMo-net \cite{rindt2022a}. The evaluation is conducted for all three metrics mentioned in this paper. The results are shown in table \ref{tab: frailty_benefits}. The results suggest a solid improvement in incorporating frailty, as the relative increase in performance could be over $10\%$ for both NFM models. A more detailed discussion is presented in section \ref{sec: benefits_frailty}.
    \section{Discussion and conclusion}
    We have introduced NFM as a flexible and powerful neural modeling framework for survival analysis, which is shown to be both statistically correct in theory, and empirically effective in predictive tasks. While our proposed framework provides a theoretically-principled tool of neural survival modeling, a few limitations and challenges need to be addressed in future works including predictive guarantees and better evaluation protocols, which we elaborate in appendix \ref{sec: limitations}. 
    \section{Acknowledgements}
    We would like to thank professor Zhiliang Ying and professor Guanhua Fang for helpful discussions. Wen Yu's research is supported by the National Natural Science Foundation of China Grants ($12071088$). Ming Zheng's research is supported by the National Natural Science Foundation of China Grants ($12271106$).
    \bibliographystyle{abbrv}
    \bibliography{transformation}

\newpage
\appendix
\onecolumn

    \section{Examples of frailty specifications}\label{sec: frailty_spec}
    We list several commonly used frailty models, and specify their corresponding characteristics via their frailty transform $G_\theta$:
    \begin{description}
        \item[Gamma frailty:] Arguably the gamma frailty is the most widely used frailty model \cite{murphy1994consistency, murphy1995asymptotic, parner1998asymptotic, wienke2010frailty, duchateau2007frailty}, with 
        \begin{align}
            G_\theta(x) = \frac{1}{\theta} \log (1 + \theta x), \theta \ge 0.
        \end{align}
        When $\theta = 0$, $G_0(x) = \lim_{\theta \rightarrow 0} G_\theta (x)$ is defined as the (pointwise) limit. A notable fact of the gamma frailty specification is that when the proportional frailty (PF) assumption \eqref{eqn: proportional_frailty} is met, if $\theta = 0$, the model degenerates to \ph. Otherwise if $\theta = 1$, the model corresponds to the proportional odds (PO) model \cite{bennett1983analysis}.
        \item[Box-Cox transformation frailty:] Under this specification, we have
        \begin{align}
            G_\theta(x) = \dfrac{(1+x)^\theta-1}{\theta}, \theta \ge 0.
        \end{align}
        The case of $\theta = 0$ is defined analogously to that of gamma frailty, which corresponds to the PO model under the PF assumption. When $\theta = 1$, the model reduces to \ph under the PF assumption. 
        \item[$\text{IGG}(\alpha)$ frailty:] This is an extension of gamma frailty \cite{kosorok2004robust} and includes other types of frailty specifications like the inverse gaussian frailty \cite{hougaard1984life}, with
        \begin{align}
            G_\theta(x) = \frac{1 - \alpha}{\alpha \theta} \left[ \left(1 + \frac{\theta x}{1 - \alpha}\right)^\alpha - 1\right], \theta \ge 0, \alpha \in [0, 1).
        \end{align}
        In the one-dimensional parameter paradigm, the parameter $\alpha$ is assumed known instead of being learnable. When $\alpha = 1/2$, we obtain the gamma frailty model. When $\alpha \rightarrow 0$, the limit corresponds to the inverse Gaussian frailty. 
    \end{description}
    \textbf{Satistiability of regularity condition \ref{cond: G}} In \cite[Proposition 1]{kosorok2004robust}, the authors verified the regularity condition of gamma and $\text{IGG}(\alpha)$ frailties. Using a similar argument, it is straightforward to verify the regularity of Box-Cox transformation frailty. 
    \section{Proofs of theorems}
\subsection{Preliminary}
\paragraph{Additional definitions} The theory of empirical processes \cite{van1996weak} will be involved heavily in the proof. Therefore we briefly introduce some common notations: For a function class $\mathcal{F}$, define $\cn{\epsilon}{\mathcal{F}}{\|\cdot\|}$ to be the covering number of $\mathcal{F}$ with respect to norm $\|\cdot\|$ under radius $\epsilon$, and define $\bn{\epsilon}{\mathcal{F}}{\|\cdot\|}$ to be the bracketing number of $\mathcal{F}$ with respect to norm $\|\cdot\|$ under radius $\epsilon$. We use $\vcdim{\mathcal{F}}$ to denote the VC-dimension of $\mathcal{F}$. Moreover, we use the notation $a \lesssim b$ to denote $a \le C b$ for some positive constant $C$.\par 
Before proving theorem \ref{thm: rate_pf} and \ref{thm: rate_fn},
we introduce some additional notations that will be useful throughout the proof process. \par

In the PF scheme, define
\begin{align*}
    l(T,\delta,Z; h, m,\theta) =& \delta \log g_{\theta}\left(e^{m(Z)}\int_{0}^{T}e^{h(s)}ds\right)+\delta h(T)+\delta m(Z)\\
    &-G_{\theta}\left(e^{m(Z)}\int_{0}^{T}e^{h(s)}ds\right),
\end{align*}
where we denote $g_\theta = G^\prime(\theta)$. Under the definition of the sieve space stated in condition \ref{cond: sieve_PF}, we restate the parameter estimates as
\begin{align*}
    \left(\widehat{h}_{n},\widehat{m}_{n},\widehat{\theta}_{n}\right) = \mathop{\mathrm{argmax}}\limits_{
    \widehat{h} \in \mathcal{H}_n,\widehat{m} \in \mathcal{M}_n,\theta \in \Theta
    }\frac{1}{n}\sum_{i\in [n]}l(T_{i},\delta_{i},Z_{i};\widehat{h},\widehat{m},\theta).
\end{align*}

Similarly, in the FN scheme, we define 
\begin{align*}
    l(T,\delta,Z; \nu,\theta) = \delta \log g_{\theta}\left(\int_{0}^{T} e^{\nu(s,Z)}ds\right)+\delta \nu(T,Z) - G_{\theta}\left(\int_{0}^{T} e^{\nu(s,Z)}ds\right)
\end{align*}
Under the definition of the sieve space stated in condition \ref{cond: sieve_FN}, we restate the parameter estimates as
\begin{align*}
    \left(\widehat{\nu}_{n}(t,z), \widehat{\theta}_{n}\right) = \mathop{\mathrm{argmax}}\limits_{
    \widehat{\nu} \in \mathcal{V}_n,\theta \in \Theta
    }\frac{1}{n}\sum_{i \in [n]}l(T_{i},\delta_{i},Z_{i};\widehat{\nu},\theta).
\end{align*}

We denote the conditional density function and survival function of the event time $\tilde{T}$ given $Z$ by $f_{\tilde{T}\mid Z}(t)$ and $S_{\tilde{T}\mid Z}(t)$, respectively. Similarly, we denote the conditional density function and survival function of the censoring time $C$ given $Z$ by $f_{C\mid Z}(t)$ and $S_{C\mid Z}(t)$. Under the assumption that $\tilde{T} \ind C \mid Z$, the joint conditional density of the observed time $T$ and the censoring indicator $\delta$ given $Z$ can be expressed as the following:
\begin{eqnarray*}
    p(T,\delta \mid Z) &=& f_{\tilde{T}\mid Z}(T)^{\delta}S_{\tilde{T}\mid Z}(T)^{1-\delta}f_{C\mid Z}(T)^{1-\delta}S_{C\mid Z}(T)^{\delta}\\
    &=& \lambda_{\tilde{T}\mid Z}(T)^{\delta}S_{\tilde{T}\mid Z}(T)f_{C\mid Z}(T)^{1-\delta}S_{C\mid Z}(T)^{\delta},
\end{eqnarray*}
where $\lambda_{\tilde{T}\mid Z}(T)$ is the conditional hazard function of the survival time $\tilde{T}$ given $Z$.

Under the model assumption of PF scheme, $p(T,\delta \mid Z)$ can be expressed by
\begin{eqnarray*}
p(T,\delta \mid Z;h,m,\theta) = \exp\left(l(T,\delta,Z;h,m,\theta)\right)f_{C\mid Z}(T)^{1-\delta}S_{C\mid Z}(T)^{\delta}.
\end{eqnarray*}

For $\phi_{0}=(h_{0},m_{0},\theta_{0})$ and an estimator  $\widehat{\phi}=(\widehat{h},\widehat{m},\widehat{\theta})$, the defined distance $d_{\textsf{PF}}\left(\widehat{\phi},\phi_{0}\right)$ can be explicitly expresses by
\begin{eqnarray*}
    d_{\textsf{FN}}\left(\widehat{\psi},\psi_{0}\right) = \sqrt{\mathbb{E}_Z\left[\int\left|\sqrt{p(T,\delta\mid Z;\widehat{h},\widehat{m},\widehat{\theta})}-\sqrt{p(T,\delta\mid Z;h_{0},m_{0},\theta_{0})}\right|^{2}\mu(dT\times d\delta)\right]}.
\end{eqnarray*}
Here the dominating measure $\mu$ is defined such that for any (measurable) function $r(T,\delta)$
\begin{align*}
    \int r(T,\delta)\mu(dT\times d\delta)=\int_{0}^{\tau}r(T,\delta = 1)dT+\int_{0}^{\tau}r(T,\delta=0)dT
\end{align*}

Under the model assumption of FN scheme, $p(T,\delta \mid Z)$ can be expressed by
\begin{eqnarray*}
p(T,\delta \mid Z;\nu,\theta) = \exp\left(l(T,\delta,Z;\nu, \theta)\right)f_{C\mid Z}(T)^{1-\delta}S_{C\mid Z}(T)^{\delta}.
\end{eqnarray*}

For $\psi_{0}=(\nu_{0},\theta_{0})$ and an estimator $\widehat{\psi}=(\widehat{\nu},\widehat{\theta})$, the defined distance $d_{\textsf{FN}}\left(\widehat{\psi},\psi_{0}\right)$ can be explicitly expresses by
\begin{eqnarray*}
    d_{\textsf{FN}}\left(\widehat{\psi},\psi_{0}\right) = \sqrt{\mathbb{E}_Z\left[\int\left|\sqrt{p(T,\delta\mid Z;\widehat{\nu},\widehat{\theta})}-\sqrt{p(T,\delta\mid Z;\nu_{0},\theta_{0})}\right|^{2}\mu(dT\times d\delta)\right]}.
\end{eqnarray*}

\subsection{Technical lemmas}
The following lemmas are needed for the proof of Theorem \ref{thm: rate_pf} and \ref{thm: rate_fn}. Hereafter for notational convenience, we will use $\widehat{h}, \widehat{m}$ for arbitrary elements in the corresponding sieve space listed in condition \ref{cond: sieve_PF}, $\widehat{\nu}$ for an arbitrary element in the sieve space listed in condition \ref{cond: sieve_FN}, and $\widehat{\theta}$ for an arbitrary element in $\Theta$.

\begin{lemma}\label{lem: pf_l_bound}
Under condition \ref{cond: param_PF}, \ref{cond: sieve_PF}, \ref{cond: G}, for $(T,\delta,Z)\in [0,\tau]\times\{0,1\}\times[-1,1]^{d}$, the following terms are bounded:
\begin{enumerate}[leftmargin=*]
    \item $l(T,\delta,Z;h_{0},m_{0},\theta_{0})$ with true parameter $(h_{0},m_{0},\theta_{0})$
    \item $l(T,\delta,Z;\widehat{h},\widehat{m},\widehat{\theta})$ with parameter estimates $(\widehat{h},\widehat{m},\widehat{\theta})$ in any sieve space listed in condition \ref{cond: sieve_PF}.
\end{enumerate}
\end{lemma}

\begin{lemma}\label{lem: fn_l_bound}
Under condition \ref{cond: param_FN}, \ref{cond: sieve_FN}, \ref{cond: G}, for $(T,\delta,Z)\in [0,\tau]\times\{0,1\}\times[-1,1]^{d}$, the following terms are bounded:
\begin{enumerate}[leftmargin=*]
    \item $l(T,\delta,Z;\nu_{0},\theta_{0})$ with true parameter $(\nu_{0},\theta_{0})$
    \item $l(T,\delta,Z;\widehat{\nu},\widehat{\theta})$ with parameter estimates $(\widehat{\nu},\widehat{\theta})$ in any sieve space listed in condition \ref{cond: sieve_FN}.
\end{enumerate}
\end{lemma}

\begin{lemma}\label{lem: pf_norm_decomp}
Under condition \ref{cond: param_PF}, \ref{cond: sieve_PF}, \ref{cond: G}, let
$(\widehat{h},\widehat{m},\widehat{\theta})$, $(\widehat{h}_{1},\widehat{m}_{1},\widehat{\theta}_{1})$, and $(\widehat{h}_{2},\widehat{m}_{2},\widehat{\theta}_{2})$ be arbitrary three parameter triples inside the sieve space defined in condition \ref{cond: sieve_PF}, 
the following two inequalities hold.
\begin{align*}
    &\|l(T,\delta,Z;h_{0},m_{0},\theta_{0})-l(T,\delta,Z;\widehat{h},\widehat{m},\widehat{\theta})\|_{\infty}\lesssim |\theta_{0}-\widehat{\theta}|+\|h_{0}-\widehat{h}\|_{\infty}+\|m_{0}-\widehat{m}\|_{\infty} \\
    &\|l(T,\delta,Z;\widehat{h}_{1},\widehat{m}_{1},\widehat{\theta}_{1})-l(T,\delta,Z;\widehat{h}_{2},\widehat{m}_{2},\widehat{\theta}_{2})\|_{\infty}\lesssim
    |\widehat{\theta}_{1}-\widehat{\theta}_{2}|+\|\widehat{h}_{1}-\widehat{h}_{2}\|_{\infty}+\|\widehat{m}_{1}-\widehat{m}_{2}\|_{\infty}.
\end{align*}
\end{lemma}

\begin{lemma}\label{lem: fn_norm_decomp}
    Under condition \ref{cond: param_FN}, \ref{cond: sieve_FN}, \ref{cond: G}, let $(\widehat{\nu},\widehat{\theta})$, $(\widehat{\nu}_{1},\widehat{\theta}_{1})$, and $(\widehat{\nu}_{2},\widehat{\theta}_{2})$ be arbitrary three parameter tuples inside the sieve space defined in condition \ref{cond: sieve_FN},, the following inequalities hold.
    \begin{align*}
    &\|l(T,\delta,Z;\nu_{0},\theta_{0})-l(T,\delta,Z;\widehat{\nu},\widehat{\theta})\|_{\infty}\lesssim |\theta_{0}-\widehat{\theta}|+\|\nu_{0}-\widehat{\nu}\|_{\infty}\\
    &\|l(T,\delta,Z;\widehat{\nu}_{1},\widehat{\theta}_{1})-l(T,\delta,Z;\widehat{\nu}_{2},\widehat{\theta}_{2})\|_{\infty}\lesssim
    |\widehat{\theta}_{1}-\widehat{\theta}_{2}|+\|\widehat{\nu}_{1}-\widehat{\nu}_{2}\|_{\infty}.
    \end{align*}
\end{lemma}

\begin{lemma}[Approximating error of PF scheme]\label{lem: pf_approx}
    In the PF scheme, for any $n$, there exists an element in the corresponding sieve space $\pi_{n}\phi_{0}=(\pi_{n}h_{0},\pi_{n}m_{0},\pi_{n}\theta_{0})$, satisfying $d_{\textsf{PF}}\left(\pi_{n}\phi_{0},\phi_{0}\right) = O\left(n^{-\frac{\beta}{\beta+d}}\right)$.
\end{lemma}

\begin{lemma}[Approximating error of FN scheme]\label{lem: fn_approx}
    In the FN scheme, for any $n$, there exists an element in the corresponding sieve space  $\pi_{n}\psi = (\pi_{n}\nu_{0},\pi_{n}\theta_{0})$ satisfying 
    $d_{\textsf{FN}}\left(\pi_{n}\psi_{0},\psi_{0}\right)= O\left(n^{-\frac{\beta}{\beta+d+1}}\right)$.
\end{lemma}

\begin{lemma}\label{lem: covering_number}
    Suppose $\mathcal{F}$ is a class of functions satisfying that $N(\varepsilon,\mathcal{F},\|\cdot\|)<\infty$ for $\forall \varepsilon>0$. We define $\widetilde{N}(\varepsilon,\mathcal{F},\|\cdot\|)$ to be the minimal number of $\varepsilon$-balls $B(f,\varepsilon)=\{g: \|g-f\|<\varepsilon\}$ needed to cover $\mathcal{F}$ with radius $\varepsilon$ and further constrain that $f\in \mathcal{F}$. Then we have 
    \begin{align*}
        N(\varepsilon,\mathcal{F},\|\cdot\|)\leq\widetilde{N}(\varepsilon,\mathcal{F},\|\cdot\|)\leq N(\frac{\varepsilon}{2},\mathcal{F},\|\cdot\|).
    \end{align*}
\end{lemma}

\begin{lemma}\label{lem: bracketing_number}
    Suppose $\mathcal{F}$ is a class of functions satisfying that $N_{[]}(\varepsilon,\mathcal{F},\|\cdot\|_{\infty})<\infty$ for $\forall \varepsilon>0$. We define $\widetilde{N}_{[]}(\varepsilon,\mathcal{F},\|\cdot\|_{\infty})$ to be the minimal number of brackets $[l,u]$ needed to cover $\mathcal{F}$ with $\|l-u\|_{\infty}< \varepsilon$ and further constrain that $f\in \mathcal{F}$, $l=f-\frac{\varepsilon}{2}$ and $u=f+\frac{\varepsilon}{2}$. Then we have
    \begin{align*}
        N_{[]}(\varepsilon,\mathcal{F},\|\cdot\|_{\infty})\leq\widetilde{N}_{[]}(\varepsilon,\mathcal{F},\|\cdot\|_{\infty})\leq N_{[]}(\frac{\varepsilon}{2},\mathcal{F},\|\cdot\|_{\infty})
    \end{align*}
    Furthermore, we have $\widetilde{N}_{[]}(\varepsilon,\mathcal{F},\|\cdot\|_{\infty})=\widetilde{N}(\frac{\varepsilon}{2},\mathcal{F},\|\cdot\|_{\infty})$.
\end{lemma}

\begin{lemma}[Model capacity of PF scheme]\label{lem: pf_capacity}
    Let $\mathcal{F}_{n}=\{l(T,\delta,Z;\widehat{h},\widehat{m},\widehat{\theta}): \widehat{h}\in \mathcal{H}_{n},\widehat{m} \in \mathcal{M}_{n},\widehat{\theta} \in \Theta\}$. Under condition \ref{cond: G}, with $s_{h}=\frac{2\beta}{2\beta+1}$ and $s_{m}=\frac{2\beta}{2\beta+d}$, there exist constants $c_{h}$ and $c_{m}>0$ such that
    \begin{eqnarray*}
    N_{[]}(\varepsilon,\mathcal{F}_{n},\|\cdot\|_{\infty})\lesssim \frac{1}{\varepsilon} N(c_{h}\varepsilon^{1/s_{h}},\mathcal{H}_{n},\|\cdot\|_{2})\times N(c_{m}\varepsilon^{1/s_{m}},\mathcal{M}_{n},\|\cdot\|_{2}).
    \end{eqnarray*}
\end{lemma}

\begin{lemma}[Model capacity of FN scheme]\label{lem: fn_capacity}
    Let $\mathcal{G}_{n}=\{l(T,\delta,Z;\widehat{\nu},\widehat{\theta}):\widehat{\nu}\in\mathcal{V}_n,\widehat{\theta}\in\Theta\}$. Under condition \ref{cond: G}, with $s_{\nu}=\frac{2\beta}{2\beta+d+1}$, there exists a constant $c_{\nu}>0$ such that
    \begin{eqnarray*}
    N_{[]}(\varepsilon,\mathcal{G}_{n},\|\cdot\|_{\infty})\lesssim \frac{1}{\varepsilon} N(c_{\nu}\varepsilon^{1/s_{\nu}},\mathcal{V}_n,\|\cdot\|_{2}).
    \end{eqnarray*}
\end{lemma}

\subsection{Proofs of theorem \ref{thm: rate_pf} and \ref{thm: rate_fn}}
\begin{proof}[Proof of theorem \ref{thm: rate_pf}]
    The proof is divided into four steps.
    
    \paragraph{Step $1$} We denote $\phi_{0}=(h_{0},m_{0},\theta_{0})$ and $\widehat{\phi}=(\widehat{h},\widehat{m},\widehat{\theta})$, where $\widehat{h}\in\mathcal{H}_{n}$,$\widehat{m}\in\mathcal{M}_{n}$ and $\widehat{\theta}\in\Theta$. For arbitrary small $\varepsilon>0$, we have that
    \begin{eqnarray*}
    &&\inf_{d_{\textsf{PF}}\left(\widehat{\phi},\phi_{0}\right)\geq \varepsilon} \mathbb{E}\left[ l(T,\delta,Z;h_{0},m_{0},\theta_{0})-l(T,\delta,Z;\widehat{h},\widehat{m},\widehat{\theta})\right]\\
    &&=\inf_{d_{\textsf{PF}}\left(\widehat{\phi},\phi_{0}\right)\geq \varepsilon} \mathbb{E}_Z\left[\mathbb{E}_{T,\delta\mid Z}\left[\log p(T,\delta \mid Z;h_{0},m_{0},\theta_{0})-\log p(T,\delta \mid Z;\widehat{h},\widehat{m},\widehat{\theta})\right]\right]\\
    &&=\inf_{d_{\textsf{PF}}\left(\widehat{\phi},\phi_{0}\right)\geq \varepsilon}\mathbb{E}_Z\left[ 
    \fdivergence{KL}{\mathbb{P}_{\widehat{\phi},Z}}{\mathbb{P}_{\phi_{0},Z}}
    \right]
    \end{eqnarray*}
    Using the fact that $\fdivergence{KL}{\mathbb{P}_{\widehat{\phi},Z}}{\mathbb{P}_{\phi_{0},Z}}\geq 2 H^{2}(\mathbb{P}_{\widehat{\phi},Z}\parallel\mathbb{P}_{\phi_{0},Z})$. Thus, we further obtain that
    \begin{eqnarray*}
    &&\inf_{d_{\textsf{PF}}\left(\widehat{\phi},\phi_{0}\right)\geq \varepsilon} \mathbb{E}\left[ l(T,\delta,Z;h_{0},m_{0},\theta_{0})-l(T,\delta,Z;\widehat{h},\widehat{m},\widehat{\theta})\right]\\
    &&\geq \inf_{d_{\textsf{PF}}\left(\widehat{\phi},\phi_{0}\right)\geq \varepsilon}2\mathbb{E}_Z \left[H^{2}(\mathbb{P}_{\widehat{\phi},Z}\parallel\mathbb{P}_{\phi_{0},Z})\right]\\
    &&=2\inf_{d_{\textsf{PF}}\left(\widehat{\phi},\phi_{0}\right)\geq \varepsilon}d^{2}_{\textsf{PF}}\left(\widehat{\phi},\phi_{0}\right)\\
    &&\geq 2\varepsilon^{2}.
    \end{eqnarray*}
    
    \paragraph{Step $2$} Consider the following derivation.
    \begin{eqnarray*}
    &&\sup_{d_{\textsf{PF}}\left(\widehat{\phi},\phi_{0}\right)\leq \varepsilon}\var\left[l(T,\delta,Z;h_{0},m_{0},\theta_{0})-l(T,\delta,Z;\widehat{h},\widehat{m},\widehat{\theta})\right]\\
    &&\leq \sup_{d_{\textsf{PF}}\left(\widehat{\phi},\phi_{0}\right)\leq \varepsilon} \mathbb{E}\left[\left(l(T,\delta,Z;h_{0},m_{0},\theta_{0})-l(T,\delta,Z;\widehat{h},\widehat{m},\widehat{\theta})\right)^{2}\right]\\
    &&=\sup_{d_{\textsf{PF}}\left(\widehat{\phi},\phi_{0}\right)\leq \varepsilon}\mathbb{E}_Z\mathbb{E}_{T,\delta\mid Z}\left[\log p(T,\delta,Z;h_{0},m_{0},\theta_{0})-\log p(T,\delta,Z;\widehat{h},\widehat{m},\widehat{\theta})\right]^{2}\\
    &&=4\sup_{d_{\textsf{PF}}\left(\widehat{\phi},\phi_{0}\right)\leq \varepsilon}\mathbb{E}_Z\left[\int \left(p(T,\delta,Z;h_{0},m_{0},\theta_{0})\left(\log \sqrt{\dfrac{p(T,\delta,Z;h_{0},m_{0},\theta_{0})}{p(T,\delta,Z;\widehat{h},\widehat{m},\widehat{\theta})}}\right)^{2}\right)\mu(dT\times d\delta)\right]
    \end{eqnarray*}
    By Taylor's expansion on $\log x$, there exists $\xi(T,\delta,Z)$ between $p^{\frac{1}{2}}(T,\delta,Z;h_{0},m_{0},\theta_{0})$ and $p^{\frac{1}{2}}(T,\delta,Z;\widehat{h},\widehat{m},\widehat{\theta})$ pointwisely such that
    \begin{eqnarray*}
    &&p(T,\delta,Z;h_{0},m_{0},\theta_{0})\left(\log \sqrt{\dfrac{p(T,\delta,Z;h_{0},m_{0},\theta_{0})}{p(T,\delta,Z;\widehat{h},\widehat{m},\widehat{\theta})}}\right)^{2}\\
    &&=p(T,\delta,Z;h_{0},m_{0},\theta_{0})\left(\log \sqrt{p(T,\delta,Z;h_{0},m_{0},\theta_{0})}-\log \sqrt{p(T,\delta,Z;\widehat{h},\widehat{m},\widehat{\theta})}\right)^{2}\\
    &&= \frac{p(T,\delta,Z;h_{0},m_{0},\theta_{0})}{\xi(T,\delta,Z)^{2}}\left(\sqrt{p(T,\delta,Z;h_{0},m_{0},\theta_{0})}-\sqrt{p(T,\delta,Z;\widehat{h},\widehat{m},\widehat{\theta})}\right)^{2}
    \end{eqnarray*}
    Since 
    \begin{align*}
        \dfrac{p(T,\delta,Z;h_{0},m_{0},\theta_{0})}{p(T,\delta,Z;\widehat{h},\widehat{m},\widehat{\theta})}=e^{l(T,\delta,Z;h_{0},m_{0},\theta_{0})-l(T,\delta,Z;\widehat{h},\widehat{m},\widehat{\theta})}
    \end{align*}
    by lemma \ref{lem: pf_l_bound}, $l(T,\delta,Z;h_{0},m_{0},\theta_{0})$ and $l(T,\delta,Z;\widehat{h},\widehat{m},\widehat{\theta})$ are bounded among$[0,\tau]\times\{0,1\}\times[-1,1]^{d}$ uniformly on all $\widehat{\phi}=(\widehat{h},\widehat{m},\widehat{\theta})$. Thus, there exist constants $C_{1}$ and $C_{2}$ such that $0<C_{1}\leq p(T,\delta,Z;h_{0},m_{0},\theta_{0})/p(T,\delta,Z;\widehat{h},\widehat{m},\widehat{\theta})\leq C_{2}$. This leads to the fact that $p(T,\delta,Z;h_{0},m_{0},\theta_{0})\frac{1}{\xi(T,\delta,Z)^{2}}$ is bounded. We further obtained that
    \begin{eqnarray*}
    &&p(T,\delta,Z;h_{0},m_{0},\theta_{0})\left(\log \sqrt{p(T,\delta,Z;h_{0},m_{0},\theta_{0})}-\log\sqrt{p(T,\delta,Z;\widehat{h},\widehat{m},\widehat{\theta})}\right)^{2}\\
    &&\lesssim \left|\sqrt{p(T,\delta,Z;h_{0},m_{0},\theta_{0})}-\sqrt{p(T,\delta,Z;\widehat{h},\widehat{m},\widehat{\theta})}\right|^{2}.
    \end{eqnarray*}
    Thus, we have that
    \begin{eqnarray*}
    &&\sup_{d_{\textsf{PF}}\left[\widehat{\phi},\phi_{0}\right]\leq \varepsilon}\var(l(T,\delta,Z;h_{0},m_{0},\theta_{0})-l(T,\delta,Z;\widehat{h},\widehat{m},\widehat{\theta}))\\
    &&\lesssim\sup_{d_{\textsf{PF}}\left(\widehat{\phi},\phi_{0}\right)\leq \varepsilon}\mathbb{E}_Z\left[\int\left|\sqrt{p(T,\delta,Z;h_{0},m_{0},\theta_{0})}-\sqrt{p(T,\delta,Z;\widehat{h},\widehat{m},\widehat{\theta})}\right|^{2}\mu(dT\times d\delta)\right]\\
    &&=\sup_{d_{\textsf{PF}}\left(\widehat{\phi},\phi_{0}\right)\leq \varepsilon}d^{2}_{\textsf{PF}}\left(\widehat{\phi},\phi_{0}\right)\\
    &&\leq \varepsilon^{2}.
    \end{eqnarray*}
    
    \paragraph{Step $3$} We define that $\widetilde{\mathcal{F}}_{n}=\{l(T,\delta,Z;\widehat{h},\widehat{m},\widehat{\theta})-l(T,\delta,Z;\pi_{n}h_{0},\pi_{n}m_{0},\pi_{n}\theta_{0}):\widehat{h}\in\mathcal{H}_{n},\widehat{m}\in\mathcal{M}_{n},\widehat{\theta}\in\Theta\}$.
    Here $(\pi_{n}h_{0},\pi_{n}m_{0},\pi_{n}\theta_{0})$ have been defined in lemma \ref{lem: pf_approx}. Obviously, we have that $\log N_{[]}(\varepsilon,\widetilde{\mathcal{F}}_{n},\|\cdot\|_{\infty}) = \log N_{[]}(\varepsilon,\mathcal{F}_{n},\|\cdot\|_{\infty})$, where $\mathcal{F}$ is defined in lemma \ref{lem: pf_capacity}. By lemma \ref{lem: pf_capacity}, we further have that
    \begin{eqnarray*}
    \log N_{[]}(\varepsilon,\mathcal{F}_{n},\|\cdot\|_{\infty})\lesssim \log\frac{1}{\varepsilon}+ \log N(c_{h}\varepsilon^{1/s_{h}},\mathcal{H}_{n},\|\cdot\|_{2})+ \log N(c_{m}\varepsilon^{1/s_{m}},\mathcal{M}_{n},\|\cdot\|_{2}).
    \end{eqnarray*}
    According to \cite[Theorem 7]{bartlett2019nearly}, under condition \ref{cond: sieve_PF}, we have that the VC-dimension of $\mathcal{H}_{n}$ and $\mathcal{M}_{n}$ satisfy that $\vcdim{\mathcal{H}_{n}} \lesssim n^{\frac{1}{\beta+d}}\log^{3}n$ and $\vcdim{\mathcal{M}_{n}}\lesssim n^{\frac{d}{\beta+d}}\log^{3}n$. Thus, we obtain that
    \begin{eqnarray*}
    \log N(c_{h}\varepsilon^{1/s_{h}},\mathcal{H}_{n},\|\cdot\|_{2})\lesssim \frac{\vcdim{\mathcal{H}_{n}}}{s_{h}}\log\frac{1}{\varepsilon}\lesssim n^{\frac{1}{\beta+d}}\log^{3}n\log\frac{1}{\varepsilon},
    \end{eqnarray*}
    and
    \begin{eqnarray*}
    \log N(c_{m}\varepsilon^{1/s_{m}},\mathcal{M}_{n},\|\cdot\|_{2})\lesssim \frac{\vcdim{\mathcal{M}_{n}}}{s_{\nu}}\log\frac{1}{\varepsilon}\lesssim n^{\frac{d}{\beta+d}}\log^{3}n\log\frac{1}{\varepsilon}.
    \end{eqnarray*}
    Thus, we obtain that $\log N_{[]}(\varepsilon,\widetilde{\mathcal{F}}_{n},\|\cdot\|_{\infty})\lesssim n^{\frac{d}{\beta+d}}\log^{3}n\log\frac{1}{\varepsilon}$.
    
    \paragraph{Step $4$} By the Cauchy-Schwartz inequality, we have that
    \begin{eqnarray*}
    \sqrt{\mathbb{E}\left[l(T,\delta,Z;\pi_{n}h_{0},\pi_{n}m_{0},\pi_{n}\theta_{0})-l(T,\delta,Z;h_{0},m_{0},\theta_{0})\right]}\\
    \leq \left[\mathbb{E}(l(T,\delta,Z;\pi_{n}h_{0},\pi_{n}m_{0},\pi_{n}\theta_{0})-l(T,\delta,Z;h_{0},m_{0},\theta_{0}))^{2}\right]^{\frac{1}{4}}.
    \end{eqnarray*}
    Similar to the second part and by lemma \ref{lem: pf_approx}, we further have that
    \begin{eqnarray*}
    \sqrt{\mathbb{E}\left[l(T,\delta,Z;\pi_{n}h_{0},\pi_{n}m_{0},\pi_{n}\theta_{0})-l(T,\delta,Z;h_{0},m_{0},\theta_{0})\right]}\lesssim \sqrt{d_{\textsf{PF}}\left(\pi_{n}\phi_{0},\phi_{0}\right)}\lesssim n^{-\frac{\beta}{2\beta+2d}}.
    \end{eqnarray*}
    Now let 
    \begin{align*}
        \tau = \frac{\beta}{2\beta+2d}-2\frac{\log\log n}{\log n}
    \end{align*}
    By Step 1,2,3 and \cite[Theorem 1]{shen1994convergence}, we have
    \begin{align*}
        d_{\textsf{PF}}\left(\widehat{\phi}_{n},\phi_{0}\right)&= \max\left(n^{-\tau},d_{\textsf{PF}}\left(\pi_{n}\phi_{0},\phi_{0}\right), \right. \\
        &\left. \sqrt{\mathbb{E}\left[l(T,\delta,Z;\pi_{n}h_{0},\pi_{n}m_{0},\pi_{n}\theta_{0})-l(T,\delta,Z;h_{0},m_{0},\theta_{0})\right]}\right)
    \end{align*}
    By lemma \ref{lem: pf_approx}, $d_{\textsf{PF}}\left(\pi_{n}\phi_{0},\phi_{0}\right)=O(n^{-\frac{\beta}{\beta+d}})$.\par 
    By Step 4, $\sqrt{\mathbb{E}\left[l(T,\delta,Z;\pi_{n}h_{0},\pi_{n}m_{0},\pi_{n}\theta_{0})-l(T,\delta,Z;h_{0},m_{0},\theta_{0})\right]}=O\left(n^{-\frac{\beta}{2\beta+2d}}\right)$.
    Thus, we have $d_{\textsf{PF}}\left(\widehat{\phi}_{n},\phi_{0}\right) = O(n^{-\frac{\beta}{2\beta+2d}}\log^{2}n) = \widetilde{O}(n^{-\frac{\beta}{2\beta+2d}})$.
\end{proof}

\begin{proof}[Proof of theorem \ref{thm: rate_fn}]
    The proof is divided into four steps.
    
    \paragraph{Step $1$} We denote $\psi_{0}=(\nu_{0},\theta_{0})$ and $\widehat{\psi}=(\widehat{\nu},\widehat{\theta})$, where $\widehat{\nu}\in\mathcal{V}_{n}$ and $\widehat{\theta}\in\Theta$. For arbitrary $0<\varepsilon\leq 1$, we have that
    \begin{eqnarray*}
    &&\inf_{d_{\textsf{FN}}\left(\widehat{\psi},\psi_{0}\right)\geq \varepsilon} \mathbb{E}\left[ l(T,\delta,Z;\nu_{0},\theta_{0})-l(T,\delta,Z;\widehat{\nu},\widehat{\theta})\right]\\
    &&=\inf_{d_{\textsf{FN}}\left(\widehat{\psi},\psi_{0}\right)\geq \varepsilon} \mathbb{E}_Z\left[\mathbb{E}_{T,\delta\mid Z}\left[\log p(T,\delta \mid Z;\nu_{0},\theta_{0})-\log p(T,\delta \mid Z;\widehat{\nu},\widehat{\theta})\right]\right]\\
    &&=\inf_{d_{\textsf{FN}}\left(\widehat{\psi},\psi_{0}\right)\geq \varepsilon}\mathbb{E}_Z \left[\fdivergence{KL}{\mathbb{P}_{\widehat{\psi},Z}}{\parallel\mathbb{P}_{\psi_{0},Z}}\right]
    \end{eqnarray*}
    Using the fact that $KL(\mathbb{P}_{\widehat{\psi},Z}\parallel\mathbb{P}_{\psi_{0},Z})\geq 2 H^{2}(\mathbb{P}_{\widehat{\psi},Z}\parallel\mathbb{P}_{\psi_{0},Z})$. Thus, we further obtain that
    \begin{eqnarray*}
    &&\inf_{d_{\textsf{FN}}\left(\widehat{\psi},\psi_{0}\right)\geq \varepsilon} \mathbb{E}\left[ l(T,\delta,Z;\nu_{0},\theta_{0})-l(T,\delta,Z;\widehat{\nu},\widehat{\theta})\right]\\
    &&\geq \inf_{d_{\textsf{FN}}\left(\widehat{\psi},\psi_{0}\right)\geq \varepsilon}2\mathbb{E}_Z\left[H^{2}(\mathbb{P}_{\widehat{\psi},Z}\parallel\mathbb{P}_{\psi_{0},Z})\right]\\
    &&=2\inf_{d_{\textsf{FN}}\left(\widehat{\psi},\psi_{0}\right)\geq \varepsilon}d^{2}_{\textsf{FN}}\left(\widehat{\psi},\psi_{0}\right)\\
    &&\geq 2\varepsilon^{2}.
    \end{eqnarray*}
    
    \paragraph{Step $2$} We consider the following derivation.
    \begin{eqnarray*}
    &&\sup_{d_{\textsf{FN}}\left(\widehat{\psi},\psi_{0}\right)\leq \varepsilon}\var\left[l(T,\delta,Z;\nu_{0},\theta_{0})-l(T,\delta,Z;\widehat{\nu},\widehat{\theta})\right]\\
    &&\leq \sup_{d_{\textsf{FN}}\left(\widehat{\psi},\psi_{0}\right)\leq \varepsilon} \mathbb{E}\left[\left(l(T,\delta,Z;\nu_{0},\theta_{0})-l(T,\delta,Z;\widehat{\nu},\widehat{\theta})\right)^{2}\right]\\
    &&=\sup_{d_{\textsf{FN}}\left(\widehat{\psi},\psi_{0}\right)\leq \varepsilon}\mathbb{E}_Z\left[\mathbb{E}_{T,\delta\mid Z}\left[\left(\log p(T,\delta,Z;\nu_{0},\theta_{0})-\log p(T,\delta,Z;\widehat{\nu},\widehat{\theta})\right)^{2}\right]\right]\\
    &&=4\sup_{d_{\textsf{FN}}\left(\widehat{\psi},\psi_{0}\right)\leq \varepsilon}\mathbb{E}_Z\left[\int \left(p(T,\delta,Z;\nu_{0},\theta_{0})(\log \sqrt{\dfrac{p(T,\delta,Z;\nu_{0},\theta_{0})}{p(T,\delta,Z;\widehat{\nu},\widehat{\theta})}})^{2}\right)\mu(dT\times d\delta)\right]\\
    \end{eqnarray*}
    By Taylor's expansion on $\log x$, there exists $\eta(T,\delta,Z)$ between $\sqrt{p(T,\delta,Z;\nu_{0},\theta_{0})}$ and $\sqrt{p(T,\delta,Z;\widehat{\nu},\widehat{\theta})}$ pointwisely such that
    \begin{eqnarray*}
    &&p(T,\delta,Z;\nu_{0},\theta_{0})(\log \sqrt{\dfrac{p(T,\delta,Z;\nu_{0},\theta_{0})}{p(T,\delta,Z;\widehat{\nu},\widehat{\theta})}})^{2}\\
    &&=p(T,\delta,Z;\nu_{0},\theta_{0})\left(\log \sqrt{p(T,\delta,Z;\nu_{0},\theta_{0})}-\log \sqrt{p(T,\delta,Z;\widehat{\nu},\widehat{\theta})}\right)^{2}\\
    &&= \dfrac{p(T,\delta,Z;\nu_{0},\theta_{0})}{\eta(T,\delta,Z)^{2}}\left(\sqrt{p(T,\delta,Z;\nu_{0},\theta_{0}})-\sqrt{p(T,\delta,Z;\widehat{\nu},\widehat{\theta})}\right)^{2}
    \end{eqnarray*}
    Since $p(T,\delta,Z;\nu_{0},\theta_{0})/p(T,\delta,Z;\widehat{\nu},\widehat{\theta})=e^{l(T,\delta,Z;\nu_{0},\theta_{0})-l(T,\delta,Z;\widehat{\nu},\widehat{\theta})}$, by lemma \ref{lem: fn_l_bound}, $l(T,\delta,Z;\nu_{0},\theta_{0})$ and $l(T,\delta,Z;\widehat{\nu},\widehat{\theta})$ are bounded on $[0,\tau]\times\{0,1\}\times[-1,1]^{d}$ uniformly for all $\widehat{\psi}=(\widehat{\nu},\widehat{\theta})$. Thus there exist constants $C_{3}$ and $C_{4}$ such that $0<C_{3}\leq p(T,\delta,Z;\nu_{0},\theta_{0})/p(T,\delta,Z;\widehat{\nu},\widehat{\theta})\leq C_{4}$. This leads to the fact that $p(T,\delta,Z;\nu_{0},\theta_{0})\frac{1}{\eta(T,\delta,Z)^{2}}$ is bounded. We further have that
    \begin{eqnarray*}
    &&p(T,\delta,Z;\nu_{0},\theta_{0})\left(\log \sqrt{p(T,\delta,Z;\nu_{0},\theta_{0})}-\log \sqrt{p(T,\delta,Z;\widehat{\nu},\widehat{\theta})}\right)^{2}\\
    &&\lesssim \left|\sqrt{p(T,\delta,Z;\nu_{0},\theta_{0})}-\sqrt{p(T,\delta,Z;\widehat{\nu},\widehat{\theta})}\right|^{2}.
    \end{eqnarray*}
    Thus, we have that
    \begin{eqnarray*}
    &&\sup_{d_{\textsf{FN}}\left(\widehat{\psi},\psi_{0}\right)\leq \varepsilon}\var\left[l(T,\delta,Z;\nu_{0},\theta_{0})-l(T,\delta,Z;\widehat{\nu},\widehat{\theta})\right]\\
    &&\lesssim\sup_{d_{\textsf{FN}}\left(\widehat{\psi},\psi_{0}\right)\leq \varepsilon}\mathbb{E}_Z\left[\int\left|\sqrt{p(T,\delta,Z;\nu_{0},\theta_{0})}-\sqrt{p(T,\delta,Z;\widehat{\nu},\widehat{\theta})}\right|^{2}\mu(dT\times d\delta)\right]\\
    &&=\sup_{d_{\textsf{FN}}\left(\widehat{\psi},\psi_{0}\right)\leq \varepsilon}d^{2}_{\textsf{FN}}\left(\widehat{\psi},\psi_{0}\right)\\
    &&\leq \varepsilon^{2}.
    \end{eqnarray*}
    
    \paragraph{Step $3$} We define that $\widetilde{\mathcal{G}}_{n}=\{l(T,\delta,Z;\widehat{\nu},\widehat{\theta})-l(T,\delta,Z;\pi_{n}\nu_{0},\pi_{n}\theta_{0}):\widehat{\nu}\in\mathcal{V}_{n},\theta\in\Theta\}$.
    Here $(\pi_{n}\nu_{0},\pi_{n}\theta_{0})$ have been defined in lemma \ref{lem: fn_approx}. Obviously, we have that $\log N_{[]}(\varepsilon,\widetilde{\mathcal{G}}_{n},\|\cdot\|_{\infty}) = \log N_{[]}(\varepsilon,\mathcal{G}_{n},\|\cdot\|_{\infty})$, where $\mathcal{G}$ is defined in lemma \ref{lem: fn_capacity}. By lemma \ref{lem: fn_capacity}, we further obtain that
    \begin{eqnarray*}
    \log N_{[]}(\varepsilon,\mathcal{G}_{n},\|\cdot\|_{\infty})\lesssim \log\frac{1}{\varepsilon}+ \log N(c_{\nu}\varepsilon^{1/s_{\nu}},\mathcal{V}_{n},\|\cdot\|_{2}).
    \end{eqnarray*}
    According to \cite[Theorem 7]{bartlett2019nearly}, under condition \ref{cond: sieve_FN}, we have that the VC-dimension of $\mathcal{V}_{n}$ satisfies that $\vcdim{\mathcal{V}_{n}}\lesssim n^{\frac{d+1}{\beta+d+1}}\log^{3}n$. Thus, we obtain that
    \begin{eqnarray*}
    \log N(c_{h}\varepsilon^{1/s_{\nu}},\mathcal{V}_{n},\|\cdot\|_{2})\lesssim \frac{\vcdim{\mathcal{V}_{n}}}{s_{\nu}}\log\frac{1}{\varepsilon}\lesssim n^{\frac{d+1}{\beta+d+1}}\log^{3}n\log\frac{1}{\varepsilon}.
    \end{eqnarray*}
    Furthermore, we get that $\log N_{[]}(\varepsilon,\widetilde{\mathcal{G}}_{n},\|\cdot\|_{\infty})\lesssim n^{\frac{d+1}{\beta+d+1}}\log^{3}n\log\frac{1}{\varepsilon}$.
    
    \paragraph{Step $4$} By the Cauchy-Schwartz inequality, we have that
    \begin{eqnarray*}
    \sqrt{\mathbb{E}[l(T,\delta,Z;\pi_{n}\nu_{0},\pi_{n}\theta_{0})-l(T,\delta,Z;\nu_{0},\theta_{0})]}\leq \left[\mathbb{E}\left(l(T,\delta,Z;\pi_{n}\nu_{0},\pi_{n}\theta_{0})-l(T,\delta,Z;\nu_{0},\theta_{0})\right)^{2}\right]^{\frac{1}{4}}.
    \end{eqnarray*}
    Similar to the second part and by lemma \ref{lem: fn_approx}, we further obtain that
    \begin{eqnarray*}
    &&\sqrt{\mathbb{E}[l(T,\delta,Z;\pi_{n}\nu_{0},\pi_{n}\theta_{0})-l(T,\delta,Z;\nu_{0},\theta_{0})]}\lesssim \sqrt{d_{\textsf{FN}}\left(\pi_{n}\psi_{0},\psi_{0}\right)}\lesssim n^{-\frac{\beta}{2\beta+2d+2}}
    \end{eqnarray*}
    Now let
    \begin{align*}
        \tau = \frac{\beta}{2\beta+2d+2}-2\frac{\log\log n}{\log n}.
    \end{align*}
    By step 1,2,3 and Step 1,2,3 and \cite[Theorem 1]{shen1994convergence},
    \begin{align*}
        d_{\textsf{FN}}\left(\widehat{\psi}_{n},\psi_{0}\right) &= \max\left(n^{-\tau},d_{\textsf{FN}}\left(\pi_{n}\psi_{0},\psi_{0}\right),\right.\\
        &\left.\sqrt{\mathbb{E}[l(T,\delta,Z;\pi_{n}\nu_{0},\pi_{n}\theta_{0})-l(T,\delta,Z;\nu_{0},\theta_{0})]}\right)
    \end{align*}
    By lemma\ref{lem: fn_approx}, $d_{\textsf{FN}}\left(\pi_{n}\psi_{0},\psi_{0}\right)=O(n^{-\frac{\beta}{\beta+d+1}})$ \par
    By Step 4, $\sqrt{\mathbb{E}[l(T,\delta,Z;\pi_{n}\nu_{0},\pi_{n}\theta_{0})-l(T,\delta,Z;\nu_{0},\theta_{0})]} = O(n^{-\frac{\beta}{2\beta+2d+2}})$.
    Thus, we have  $d_{\textsf{FN}}\left(\widehat{\psi}_{n},\psi_{0}\right) = O(n^{-\frac{\beta}{2\beta+2d+2}}\log^{2}n) = \widetilde{O}(n^{-\frac{\beta}{2\beta+2d+2}})$.
\end{proof}

\subsection{Proofs of technical lemmas}

\begin{proof}[Proof of lemma \ref{lem: pf_l_bound}]
    Since $h_{0}(T)\in \holderspace{\beta}{M}{[0, \tau]}$ and $m_{0}(Z)\in \holderspace{\beta}{M}{[-1, 1]^d}$, we have that $h_{0}(T)\leq M$, $m_{0}(Z)\leq M$ and $e^{m_{0}(Z)}\int_{0}^{T}h_{0}(s)ds\leq \tau e^{2M}$. Let $\mathcal{B}=[0,\tau e^{2M}]$, we have that
    \begin{eqnarray*}
    &&|l(T,\delta,Z;h_{0},m_{0},\theta_{0})|\\
    &&\leq \left|\log g_{\theta_{0}}\left(e^{m_{0}(Z)}\int_{0}^{T}e^{h_{0}(s)}ds\right)\right|+|h_{0}(T)|+ |m_{0}(Z)|+\left|G_{\theta_{0}}\left(e^{m_{0}(Z)}\int_{0}^{T}e^{h_{0}(s)}ds\right)\right|\\
    &&\leq 2M+\sup_{x\in \mathcal{B}}\left|\log g_{\theta_{0}}(x)\right|+\sup_{x\in \mathcal{B}}\left|G_{\theta_{0}}(x)\right|
    \end{eqnarray*}
    By condition 5, we have that $l(T,\delta,Z;h_{0},m_{0},\theta_{0})$ is bounded for $(T,\delta,Z)\in [0,\tau]\times\{0,1\}\times[-1,1]^{d}$. The proof of the boundness of $l(T,\delta,Z;\widehat{h},\widehat{m},\widehat{\theta})$ is similar.
\end{proof}

\begin{proof}[Proof of lemma \ref{lem: fn_l_bound}]
    Since $\nu_{0}(T,Z)\in \holderspace{\beta}{M}{[0, \tau] \times [-1, 1]^d}$, we have $\nu_{0}(T,Z)\leq M$ and $\int_{0}^{T} e^{\nu(s,Z)}ds\leq \tau e^{M}$. Let $\mathcal{B}=[0,\tau e^{M}]$, we have that
    \begin{eqnarray*}
    &&|l(T,\delta,Z;\nu_{0},\theta_{0})|\\
    &&\leq \left|\log G^{\prime}_{\theta_{0}}\left(\int_{0}^{T} e^{\nu_{0}(s,Z)}ds\right)\right|+ |\nu_{0}(T,Z)|+\left|G_{\theta_{0}}\left(\int_{0}^{T} e^{\nu_{0}(s,Z)}ds\right)\right|\\
    &&\leq M+\sup_{x\in \mathcal{B}}\left|\log G^{\prime}_{\theta_{0}}(x)\right|+\sup_{x\in \mathcal{B}}\left|G_{\theta_{0}}(x)\right|.
    \end{eqnarray*}
    By condition \ref{cond: G}, we have that $l(T,\delta,Z;\nu_{0},\theta_{0})$ is bounded among $(T,\delta,Z)\in [0,\tau]\times\{0,1\}\times[-1,1]^{d}$ The proof of the boundness of $l(T,\delta,Z;\widehat{\nu},\widehat{\theta})$ is similar.
\end{proof}

\begin{proof}[Proof of lemma \ref{lem: pf_norm_decomp}]
    By definition, we have that
    \begin{eqnarray*}
    &&|l(T,\delta,Z;h_{0},m_{0},\theta_{0})-l(T,\delta,Z;\widehat{h},\widehat{m},\widehat{\theta})|\\
    &&\leq
    \left|\log g_{\theta_{0}}\left(e^{m_{0}(Z)}\int_{0}^{T}e^{h_{0}(s)}ds\right)-\log g_{\widehat{\theta}}\left(e^{\widehat{m}(Z)}\int_{0}^{T}e^{\widehat{h}(s)}ds\right)\right|+\left|h_{0}(T)-\widehat{h}(T)\right|\\
    &&+|m_{0}(Z)-\widehat{m}(Z)|+\left|G_{\theta_{0}}\left(e^{m_{0}(Z)}\int_{0}^{T}e^{h_{0}(s)}ds\right)-G_{\widehat{\theta}}\left(e^{\widehat{m}(Z)}\int_{0}^{T}e^{\widehat{h}(s)}ds\right)\right|.
    \end{eqnarray*}
    Let $\mathcal{B}=[0,\tau\max(e^{2M},e^{M_{h}+M_{m}})]$. By Taylor's expansion, we can further show that
    \begin{eqnarray*}
    &&|l(T,\delta,Z;h_{0},m_{0},\theta_{0})-l(T,\delta,Z;\widehat{h},\widehat{m},\widehat{\theta})|\\
    &&\leq \sup_{\tilde{\theta}\in\Theta,\tilde{x}\in \mathcal{B}}\left|\frac{\partial \log g_{\tilde{\theta}}(\tilde{x})}{\partial_{\tilde{\theta}}}\right|\cdot\left|\theta_{0}-\widehat{\theta}\right|\\
    &&+\sup_{\tilde{\theta}\in\Theta,\tilde{x}\in \mathcal{B}}\left|\frac{\partial \log g_{\tilde{\theta}}(\tilde{x})}{\partial_{\tilde{x}}}\right|\cdot\left|e^{m_{0}(Z)}\int_{0}^{T}e^{h_{0}(s)}ds-e^{\widehat{m}(Z)}\int_{0}^{T}e^{\widehat{h}(s)}ds\right|\\
    &&+|h_{0}(T)-\widehat{h}(T)|+|m_{0}(Z)-\widehat{m}(Z)|+\sup_{\tilde{\theta}\in\Theta,\tilde{x}\in \mathcal{B}}\left|\frac{\partial G_{\tilde{\theta}}(\tilde{x})}{\partial_{\tilde{\theta}}}\right|\cdot\left|\theta_{0}-\widehat{\theta}\right|\\
    &&+\sup_{\tilde{\theta}\in\Theta,\tilde{x}\in \mathcal{B}}\left|\frac{\partial G_{\tilde{\theta}}(\tilde{x})}{\partial_{\tilde{x}}}\right|\cdot\left|e^{m_{0}(Z)}\int_{0}^{T}e^{h_{0}(s)}ds-e^{\widehat{m}(Z)}\int_{0}^{T}e^{\widehat{h}(s)}ds\right|.
    \end{eqnarray*}
    Again, by Taylor's expansion, we have that
    \begin{eqnarray*}
    &&\left|e^{m_{0}(Z)}\int_{0}^{T}e^{h_{0}(s)}ds-e^{\widehat{m}(Z)}\int_{0}^{T}e^{\widehat{h}(s)}ds\right|\\
    &&\leq \left|e^{m_{0}(Z)}\int_{0}^{T}(e^{h_{0}(s)}-e^{\widehat{h}(s)})ds\right|+\left|(e^{m_{0}(Z)}-e^{\widehat{m}(Z)})\int_{0}^{T}e^{\widehat{h}(s)}ds\right|\\
    &&\leq e^{M}\cdot\tau e^{\max(M,M_{h})}\left\|h_{0}-\widehat{h}\right\|_{\infty}+\tau e^{M_{h}}\cdot e^{\max(M,M_{m})}\|m_{0}-\widehat{m}\|_{\infty}.
    \end{eqnarray*}
    Finally, we obtain that
    \begin{eqnarray*}
    &&|l(T,\delta,Z;h_{0},m_{0},\theta_{0})-l(T,\delta,Z;\widehat{h},\widehat{m},\widehat{\theta})|\\
    &&\leq \sup_{\tilde{\theta}\in\Theta,\tilde{x}\in \mathcal{B}}\left|\frac{\partial \log g_{\tilde{\theta}}(\tilde{x})}{\partial_{\tilde{x}}}\right|\cdot \left[ e^{M}\cdot\tau e^{\max(M,M_{h})}\|h_{0}-\widehat{h}\|_{\infty}+\tau e^{M_{h}}\cdot e^{\max(M,M_{m})}\|m_{0}-\widehat{m}\|_{\infty}\right]\\
    &&+\sup_{\tilde{\theta}\in\Theta,\tilde{x}\in \mathcal{B}}\left|\frac{\partial \log g_{\tilde{\theta}}(\tilde{x})}{\partial_{\tilde{\theta}}}\right|\cdot\left|\theta_{0}-\widehat{\theta}\right|+\left|h_{0}(T)-\widehat{h}(T)\right|+|m_{0}(Z)-\widehat{m}(Z)|\\
    &&+\sup_{\tilde{\theta}\in\Theta,\tilde{x}\in \mathcal{B}}\left|\frac{\partial G_{\tilde{\theta}}(\tilde{x})}{\partial_{\tilde{\theta}}}\right|\cdot\left|\theta_{0}-\widehat{\theta}\right|\\
    &&+\sup_{\tilde{\theta}\in\Theta,\tilde{x}\in \mathcal{B}}\left|\frac{\partial G_{\tilde{\theta}}(\tilde{x})}{\partial_{\tilde{x}}}\right|\cdot \left[e^{M}\cdot\tau e^{\max(M,M_{h})}\|h_{0}-\widehat{h}\|_{\infty}+\tau e^{M_{h}}\cdot e^{\max(M,M_{m})}\|m_{0}-\widehat{m}\|_{\infty}\right].
    \end{eqnarray*}
    Taking supremum on both sides, we conclude that
    \begin{equation*}
    \|l(T,\delta,Z;h_{0},m_{0},\theta_{0})-l(T,\delta,Z;\widehat{h},\widehat{m},\widehat{\theta})\|_{\infty}\lesssim |\theta_{0}-\widehat{\theta}|+\|h_{0}-\widehat{h}\|_{\infty}+\|m_{0}-\widehat{m}\|_{\infty}.
    \end{equation*}
    The proof of the second inequality is similar.
\end{proof}

\begin{proof}[Proof of lemma \ref{lem: fn_norm_decomp}]
    By definition, we have that
    \begin{eqnarray*}
    &&|l(T,\delta,Z;\nu_{0},\theta_{0})-l(T,\delta,Z;\widehat{\nu},\widehat{\theta})|\\
    &&\leq
    \left|\log g_{\theta_{0}}\left(\int_{0}^{T}e^{\nu_{0}(s,Z)}ds\right)-\log g_{\widehat{\theta}}\left(\int_{0}^{T}e^{\widehat{\nu}(s,Z)}ds\right)\right|+\left|\nu_{0}(T,Z)-\widehat{\nu}(T,Z)\right|\\
    &&+\left|G_{\theta_{0}}\left(\int_{0}^{T}e^{\nu_{0}(s,Z)}ds\right)-G_{\widehat{\theta}}\left(\int_{0}^{T}e^{\widehat{\nu}(s,Z)}ds\right)\right|.
    \end{eqnarray*}
    Let $\mathcal{B}=[0,\tau\max(e^{M},e^{M_{\nu}})]$. By Taylor's expansion, we can further show that
    \begin{eqnarray*}
    &&|l(T,\delta,Z;\nu_{0},\theta_{0})-l(T,\delta,Z;\widehat{\nu},\widehat{\theta})|\\
    &&\leq \sup_{\tilde{\theta}\in\Theta,\tilde{x}\in \mathcal{B}}\left|\frac{\partial \log g_{\tilde{\theta}}(\tilde{x})}{\partial_{\tilde{\theta}}}\right|\cdot\left|\theta_{0}-\widehat{\theta}\right|+\sup_{\tilde{\theta}\in\Theta,\tilde{x}\in \mathcal{B}}\left|\frac{\partial \log g_{\tilde{\theta}}(\tilde{x})}{\partial_{\tilde{x}}}\right|\cdot\left|\int_{0}^{T}e^{\nu_{0}(s,Z)}ds-\int_{0}^{T}e^{\widehat{\nu}(s,Z)}ds\right|\\
    &&+|\nu_{0}(T,Z)-\widehat{\nu}(T,Z)|+\sup_{\tilde{\theta}\in\Theta,\tilde{x}\in \mathcal{B}}\left|\frac{\partial G_{\tilde{\theta}}(\tilde{x})}{\partial_{\tilde{\theta}}}\right|\cdot|\theta_{0}-\widehat{\theta}|\\
    &&+\sup_{\tilde{\theta}\in\Theta,\tilde{x}\in \mathcal{B}}\left|\frac{\partial G_{\tilde{\theta}}(\tilde{x})}{\partial_{\tilde{x}}}\right|\cdot\left|\int_{0}^{T}e^{\nu_{0}(s,Z)}ds-\int_{0}^{T}e^{\widehat{\nu}(s,Z)}ds\right|.
    \end{eqnarray*}
    Again, by Taylor's expansion, 
    \begin{eqnarray*}
    \left|\int_{0}^{T}e^{\nu_{0}(s,Z)}ds-\int_{0}^{T}e^{\widehat{\nu}(s,Z)}ds\right|\leq\tau e^{\max(M,M_{\nu})}\|\nu_{0}-\widehat{\nu}\|_{\infty},
    \end{eqnarray*}
    Finally, we obtain that
    \begin{eqnarray*}
    &&\left|l(T,\delta,Z;\nu_{0},\theta_{0})-l(T,\delta,Z;\widehat{\nu},\widehat{\theta})\right|\\
    &&\leq \sup_{\tilde{\theta}\in\Theta,\tilde{x}\in \mathcal{B}}\left|\frac{\partial \log g_{\tilde{\theta}}(\tilde{x})}{\partial_{\tilde{\theta}}}\right|\cdot\left|\theta_{0}-\widehat{\theta}\right| + \sup_{\tilde{\theta}\in\Theta,\tilde{x}\in \mathcal{B}}\left|\frac{\partial \log g_{\tilde{\theta}}(\tilde{x})}{\partial_{\tilde{x}}}\right|\cdot\tau e^{\max(M,M_{\nu})}\left\|\nu_{0}-\widehat{\nu}\right\|_{\infty}\\
    &&+|\nu_{0}(T,Z)-\widehat{\nu}(T,Z)|+\sup_{\tilde{\theta}\in\Theta,\tilde{x}\in \mathcal{B}}\left|\frac{\partial G_{\tilde{\theta}}(\tilde{x})}{\partial_{\tilde{\theta}}}\right|\cdot\left|\theta_{0}-\widehat{\theta}\right|\\
    &&+\sup_{\tilde{\theta}\in\Theta,\tilde{x}\in \mathcal{B}}\left|\frac{\partial G_{\tilde{\theta}}(\tilde{x})}{\partial_{\tilde{x}}}\right|\cdot\tau e^{\max(M,M_{\nu})}\left\|\nu_{0}-\widehat{\nu}\right\|_{\infty}.
    \end{eqnarray*}
    Taking supremum on both sides, we conclude that
    \begin{equation*}
    \|l(T,\delta,Z;\nu_{0},\theta_{0})-l(T,\delta,Z;\widehat{\nu},\widehat{\theta})\|_{\infty}\lesssim |\theta_{0}-\widehat{\theta}|+\|\nu_{0}-\widehat{\nu}\|_{\infty},
    \end{equation*}
    The proof of the second inequality is similar.
\end{proof}

\begin{proof}[Proof of lemma \ref{lem: pf_approx}]
    According to \cite[Theorem 1]{yarotsky2017error}, there exist approximating functions $\widehat{h}^{*}$ and $\widehat{m}^{*}$ such that $\|\widehat{h}^{*}-h_{0}\|_{\infty} = O\left(n^{-\frac{\beta}{\beta+d}}\right)$ and $\|\widehat{m}^{*}-m_{0}\|_{\infty} = O\left(n^{-\frac{\beta}{\beta+d}}\right)$. Let $\pi_{n}h_{0}=\widehat{h}^{*}$, $\pi_{n}m_{0}=\widehat{m}^{*}$, and $\pi_{n}\theta=\theta_{0}$. We have that
    \begin{eqnarray*}
    &&d_{\textsf{PF}}\left(\pi_{n}\phi_{0},\phi_{0}\right)\\
    &&=\sqrt{\mathbb{E}_Z\left[\int |\sqrt{p(T,\delta \mid Z;\pi_{n}h_{0},\pi_{n}m_{0},\pi_{n}\theta_{0})}-\sqrt{p(T,\delta \mid Z;h_{0},m_{0},\theta_{0})}|^{2}\mu(dT\times d\delta)\right]}\\
    &&=\sqrt{\mathbb{E}_Z\left[\int [e^{\frac{1}{2}l(T,\delta,Z;\pi_{n}h_{0},\pi_{n}m_{0},\pi_{n}\theta_{0})}-e^{\frac{1}{2}l(T,\delta,Z;h_{0},m_{0},\theta_{0})}]^{2}f_{C\mid Z}(T)^{1-\delta}S_{C\mid Z}(T)^{\delta}\mu(dT\times d\delta)\right]}\\
    &&\leq \left\|e^{\frac{1}{2}l(T,\delta,Z;\pi_{n}h_{0},\pi_{n}m_{0},\pi_{n}\theta_{0})}-e^{\frac{1}{2}l(T,\delta,Z;h_{0},m_{0},\theta_{0})}\right\|_{\infty} \\
    &&\qquad \qquad \times \sqrt{\mathbb{E}_Z\left[\int f_{C\mid Z}(T)^{1-\delta}S_{C\mid Z}(T)^{\delta}\mu(dT\times d\delta)\right]}.
    \end{eqnarray*}
    By lemma \ref{lem: pf_l_bound} and \ref{lem: pf_norm_decomp}, we have that
    \begin{eqnarray*}
    &&\|e^{\frac{1}{2}l(T,\delta,Z;\pi_{n}h_{0},\pi_{n}m_{0},\pi_{n}\theta_{0})}-e^{\frac{1}{2}l(T,\delta,Z;h_{0},m_{0},\theta_{0})}\|_{\infty}\\
    &&\lesssim \|\pi_{n}\theta_{0}-\theta_{0}\|+\|\pi_{n}h_{0}-h_{0}\|_{\infty}+\|\pi_{n}m_{0}-m_{0}\|_{\infty}\\
    &&=O\left(n^{-\frac{\beta}{\beta+d}}\right).
    \end{eqnarray*}
    Since $f_{C\mid Z}(T)^{1-\delta}\leq f_{C\mid Z}(T)+1$ and $S_{C\mid Z}(T)^{\delta}\leq 1$, we also have that
    \begin{eqnarray*}
    \sqrt{\mathbb{E}_Z\left[\int f_{C\mid Z}(T)^{1-\delta}S_{C\mid Z}(T)^{\delta}\mu(dT\times d\delta)\right]} &\leq&  \sqrt{\mathbb{E}_Z\left[\int (1+f_{C\mid Z}(T))\mu(dT\times d\delta)\right]}\\
    &\leq& \sqrt{2+2\tau}.
    \end{eqnarray*}
    Thus, we obtain that $d_{\textsf{PF}}\left(\pi_{n}\phi_{0},\phi_{0}\right) = O\left(n^{-\frac{\beta}{\beta+d}}\right)$.
\end{proof}

\begin{proof}[Proof of lemma \ref{lem: fn_approx}]
    According to \cite[Theorem 1]{yarotsky2017error}, there exists an approximating function $\widehat{\nu}^{*}$ such that $\|\widehat{\nu}^{*}-\nu_{0}\|_{\infty}=O\left(n^{-\frac{\beta}{\beta+d+1}}\right)$. Let $\pi_{n}\nu_{0} = \widehat{\nu}^{*}$ and $\pi_{n}\theta_{0}=\theta_{0}$. We have that
    \begin{eqnarray*}
    &&d_{\textsf{FN}}\left(\pi_{n}\psi_{0},\psi_{0}\right)\\
    &&=\sqrt{\mathbb{E}_Z\left[\int \left|\sqrt{p(T,\delta \mid Z;\pi_{n}\nu_{0},\pi_{n}\theta_{0})}-\sqrt{p(T,\delta \mid Z;\nu_{0},\theta_{0})}\right|^{2}\mu(dT\times d\delta)\right]}\\
    &&=\sqrt{\mathbb{E}_Z\left[\int \left[e^{\frac{1}{2}l(T,\delta,Z;\pi_{n}\nu_{0},\pi_{n}\theta_{0})}-e^{\frac{1}{2}l(T,\delta,Z;\nu_{0},\theta_{0})}\right]^{2}f_{C\mid Z}(T)^{1-\delta}S_{C\mid Z}(T)^{\delta}\mu(dT\times d\delta)\right]}\\
    &&\leq \left\|\frac{1}{2}e^{l(T,\delta,Z;\pi_{n}\nu_{0},\pi_{n}\theta_{0})}-\frac{1}{2}e^{l(T,\delta,Z;\nu_{0},\theta_{0})}\right\|_{\infty}\sqrt{\mathbb{E}_Z\left[\int f_{C\mid Z}(T)^{1-\delta}S_{C\mid Z}(T)^{\delta}\mu(dT\times d\delta)\right]}.
    \end{eqnarray*}
    By lemma \ref{lem: fn_l_bound} and \ref{lem: fn_norm_decomp}, we have that
    \begin{align*}
        \left\|e^{\frac{1}{2}l(T,\delta,Z;\pi_{n}\nu_{0},\pi_{n}\theta_{0})}-e^{\frac{1}{2}l(T,\delta,Z;\nu_{0},\theta_{0})}\right\|_{\infty} &\lesssim \|\pi_{n}\theta_{0}-\theta_{0}\|+\|\pi_{n}\nu_{0}-\nu_{0}\|_{\infty}\\
        &=O\left(n^{-\frac{\beta}{\beta+d+1}}\right).
    \end{align*}
    Since $f_{C\mid Z}(T)^{1-\delta}\leq f_{C\mid Z}(T)+1$ and $S_{C\mid Z}(T)^{\delta}\leq 1$, we also have that
    \begin{eqnarray*}
    \sqrt{\mathbb{E}_Z\left[\int f_{C\mid Z}(T)^{1-\delta}S_{C\mid Z}(T)^{\delta}\mu(dT\times d\delta)\right]} &\leq&  \sqrt{\mathbb{E}_Z\left[\int (1+f_{C\mid Z}(T))\mu(dT\times d\delta)\right]}\\
    &\leq& \sqrt{2+2\tau}.
    \end{eqnarray*}
    Thus, we obtain that $d_{\textsf{FN}}\left(\pi_{n}\psi_{0},\psi_{0}\right)= O\left(n^{-\frac{\beta}{\beta+d+1}}\right)$.
\end{proof}

\begin{proof}[Proof of lemma \ref{lem: covering_number}]
    The left inequality is trivial according to the definition of covering number.  We need to show that the correctness of the right inequality.

    Suppose that we have $\{B(g_{i},\frac{\varepsilon}{2})\},i=1\ldots,N$, where $N=N(\frac{\varepsilon}{2},\mathcal{F},\|\cdot\|)$, are the minimal number of $\frac{\varepsilon}{2}$-ball that covers $\mathcal{F}$. Then there exists at least one $f_{i}\in\mathcal{F}$ such that $f_{i}\in B(g_{i},\varepsilon)$. Consider the following $\varepsilon-balls$ $\{B(f_{i},\varepsilon)\},i=1\ldots,N$. For arbitrary $f\in \mathcal{F}\cap B(g_{i},\frac{\varepsilon}{2}),$ we have that $\|f-f_{i}\|\leq\|f-g_{i}\|+\|f_{i}-g_{i}\|\leq \varepsilon$. Thus $\{B(f_{i},\varepsilon)\},i=1\ldots,N$ forms a $\varepsilon$-covering of $\mathcal{F}$. By definition, we have that $\widetilde{N}(\varepsilon,\mathcal{F},\|\cdot\|)\leq N(\frac{\varepsilon}{2},\mathcal{F},\|\cdot\|)$.
\end{proof}

\begin{proof}[Proof of lemma \ref{lem: bracketing_number}]
    The proof of the first two inequalities follows exactly the same steps of lemma \ref{lem: covering_number}. Here we just need to mention the rest of the statement that $\widetilde{N}_{[]}(\varepsilon,\mathcal{F},\|\cdot\|_{\infty})=\widetilde{N}(\frac{\varepsilon}{2},\mathcal{F},\|\cdot\|_{\infty})$.
    We first choose a set of $\frac{\varepsilon}{2}$-covering balls $\{B(f_{i},\frac{\varepsilon}{2})\},i=1,\ldots,N_{1}$, where $N_{1}=\widetilde{N}(\frac{\varepsilon}{2},\mathcal{F},\|\cdot\|_{\infty})$. Now we construct a set of brackets $\{[l_{i},u_{i}]\},i=1\ldots,N_{1}$, where $l_{i}=f_{i}-\frac{\varepsilon}{2}$ and $u_{i}=f_{i}+\frac{\varepsilon}{2}$. Noting that the bracket $\{[l_{i},u_{i}]\}$ is exactly the same as $B(f_{i},\frac{\varepsilon}{2})$, The set $\{[l_{i},u_{i}]\},i=1,\ldots,N_{1}$ covers $\mathcal{F}$, which leads to $\widetilde{N}_{[]}(\varepsilon,\mathcal{F},\|\cdot\|_{\infty})\leq\widetilde{N}(\frac{\varepsilon}{2},\mathcal{F},\|\cdot\|_{\infty})$. Likewise, we have that $\widetilde{N}_{[]}(\varepsilon,\mathcal{F},\|\cdot\|_{\infty})\geq\widetilde{N}(\frac{\varepsilon}{2},\mathcal{F},\|\cdot\|_{\infty})$. Consequently,  we have that
    $\widetilde{N}_{[]}(\varepsilon,\mathcal{F},\|\cdot\|_{\infty})=\widetilde{N}(\frac{\varepsilon}{2},\mathcal{F},\|\cdot\|_{\infty})$.
\end{proof}

\begin{proof}[Proof of lemma \ref{lem: pf_capacity}]
    By lemma \ref{lem: bracketing_number}, first we have that $N_{[]}(\varepsilon,\mathcal{F}_{n},\|\cdot\|_{\infty})\leq \widetilde{N}_{[]}(\varepsilon,\mathcal{F}_{n},\|\cdot\|_{\infty})$. By lemma \ref{lem: pf_norm_decomp}, there exists a constant $c_{1}>0$ such that for arbitrary $\widehat{h}_{1},\widehat{h}_{2}\in\mathcal{H}_{n}$,$\widehat{m}_{1},\widehat{m}_{2}\in\mathcal{M}_{n}$ and $\widehat{\theta}_{1},\widehat{\theta}_{2}\in\Theta$, we have that
    \begin{eqnarray*}
    \|l(T,\delta,Z;\widehat{h}_{1},\widehat{m}_{1},\theta_{1})-l(T,\delta,Z;\widehat{h}_{2},\widehat{m}_{2},\theta_{2})\|_{\infty}\leq c_{1}[
    |\widehat{\theta}_{1}-\widehat{\theta}_{2}|+\|\widehat{h}_{1}-\widehat{h}_{2}\|_{\infty}+\|\widehat{m}_{1}-\widehat{m}_{2}\|_{\infty}],
    \end{eqnarray*}
    which indicates that as long as $|\widehat{\theta}_{1}-\widehat{\theta}_{2}|\leq\frac{\varepsilon}{3c_{1}}$, $\|\widehat{h}_{1}-\widehat{h}_{2}\|_{\infty}\leq\frac{\varepsilon}{3c_{1}}$ and $\|\widehat{m}_{1}-\widehat{m}_{2}\|_{\infty}\leq\frac{\varepsilon}{3c_{1}}$, we have that $\|l(T,\delta,Z;\widehat{h}_{1},\widehat{m}_{1},\theta_{1})-l(T,\delta,Z;\widehat{h}_{2},\widehat{m}_{2},\theta_{2})\|_{\infty}\leq \varepsilon$. Consequently, we have that
    \begin{eqnarray*}
    \widetilde{N}_{[]}(\varepsilon,\mathcal{F}_{n},\|\cdot\|_{\infty})\leq \widetilde{N}_{[]}(\frac{\varepsilon}{3c_{1}},\Theta,\|\cdot\|_{\infty})\times
    \widetilde{N}_{[]}(\frac{\varepsilon}{3c_{1}},\mathcal{H}_{n},\|\cdot\|_{\infty})\times\widetilde{N}_{[]}(\frac{\varepsilon}{3c_{1}},\mathcal{M}_{n},\|\cdot\|_{\infty}).
    \end{eqnarray*}
    
    Since $\Theta$ is a compact set on $\mathbb{R}$, by lemma \ref{lem: bracketing_number} and traditional volume argument, we have that $\widetilde{N}_{[]}(\frac{\varepsilon}{3c_{1}},\Theta,\|\cdot\|_{\infty})\leq N_{[]}(\frac{\varepsilon}{6c_{1}},\Theta,\|\cdot\|_{\infty})\lesssim\frac{1}{\varepsilon}$.
    
    For $\widetilde{N}_{[]}(\frac{\varepsilon}{3c_{1}},\mathcal{H}_{n},\|\cdot\|_{\infty})$, by lemma \ref{lem: bracketing_number}, we have that $\widetilde{N}_{[]}(\frac{\varepsilon}{3c_{1}},\mathcal{H}_{n},\|\cdot\|_{\infty})=\widetilde{N}(\frac{\varepsilon}{3c_{1}},\mathcal{H}_{n},\|\cdot\|_{\infty})$. By \cite[Lemma 2]{chen1998sieve}, there exists a constant $c_{2}>0$ such that $\|\widehat{h}_{1}-\widehat{h}_{2}\|_{\infty}\leq c_{2}\|\widehat{h}_{1}-\widehat{h}_{2}\|_{2}^{s_{h}}$, which leads to $\widetilde{N}(\frac{\varepsilon}{3c_{1}},\mathcal{H}_{n},\|\cdot\|_{\infty})\leq \widetilde{N}(\frac{\varepsilon^{1/s_{h}}}{(3c_{1}c_{2})^{1/s_{h}}},\mathcal{H}_{n},\|\cdot\|_{2})$. By lemma \ref{lem: covering_number} we further have that $\widetilde{N}(\frac{\varepsilon^{1/s_{h}}}{(3c_{1}c_{2})^{1/s_{h}}},\mathcal{H}_{n},\|\cdot\|_{2}) \leq N(\frac{\varepsilon^{1/s_{h}}}{2(3c_{1}c_{2})^{1/s_{h}}},\mathcal{H}_{n},\|\cdot\|_{2})$. Let $c_{h}=\frac{1}{2(3c_{1}c_{2})^{1/s_{h}}}$. We have that $\widetilde{N}_{[]}(\frac{\varepsilon}{3c_{1}},\mathcal{H}_{n},\|\cdot\|_{\infty})\leq N(c_{h}\varepsilon^{1/s_{h}},\mathcal{H}_{n},\|\cdot\|_{2})$.
    
    Similarly, there exists a constant $c_{m}>0$ such that $\widetilde{N}_{[]}(\frac{\varepsilon}{3c_{1}},\mathcal{M}_{n},\|\cdot\|_{\infty})\leq N(c_{m}\varepsilon^{1/s_{m}},\mathcal{M}_{n},\|\cdot\|_{2})$.
    
    Thus, finally we can obtain that
    \begin{eqnarray*}
    N_{[]}(\varepsilon,\mathcal{F}_{n},\|\cdot\|_{\infty})\lesssim \frac{1}{\varepsilon} N(c_{h}\varepsilon^{1/s_{h}},\mathcal{H}_{n},\|\cdot\|_{2})\times N(c_{m}\varepsilon^{1/s_{m}},\mathcal{M}_{n},\|\cdot\|_{2}).
    \end{eqnarray*}
\end{proof}

\begin{proof}[Proof of lemma \ref{lem: fn_capacity}]
    By lemma \ref{lem: bracketing_number}, first we have $N_{[]}(\varepsilon,\mathcal{G}_{n},\|\cdot\|_{\infty})\leq \widetilde{N}_{[]}(\varepsilon,\mathcal{G}_{n},\|\cdot\|_{\infty})$. By lemma \ref{lem: fn_norm_decomp}, there exists a constant $c_{3}>0$ such that for arbitrary $\widehat{\nu}_{1},\widehat{\nu}_{2}\in\mathcal{V}_{n}$ and $\widehat{\theta}_{1},\widehat{\theta}_{2}\in\Theta$, we have that
    \begin{eqnarray*}
    \|l(T,\delta,Z;\widehat{\nu}_{1},\widehat{\theta}_{1})-l(T,\delta,Z;\widehat{\nu}_{2},\widehat{\theta}_{2})\|_{\infty}\leq c_{3}[
    |\widehat{\theta}_{1}-\widehat{\theta}_{2}|+\|\widehat{\nu}_{1}-\widehat{\nu}_{2}\|_{\infty}],
    \end{eqnarray*}
    which indicates that as long as $|\widehat{\theta}_{1}-\widehat{\theta}_{2}|\leq\frac{\varepsilon}{2c_{3}}$and $\|\widehat{\nu}_{1}-\widehat{\nu}_{2}\|_{\infty}\leq\frac{\varepsilon}{2c_{3}}$, we have that $\|l(T,\delta,Z;\widehat{\nu}_{1},\widehat{\theta}_{1})-l(T,\delta,Z;\widehat{\nu}_{2},\widehat{\theta}_{2})\|_{\infty}\leq \varepsilon$. Thus, we have:
    \begin{eqnarray*}
    \widetilde{N}_{[]}(\varepsilon,\mathcal{G}_{n},\|\cdot\|_{\infty})\leq \widetilde{N}_{[]}(\frac{\varepsilon}{2c_{3}},\Theta,\|\cdot\|_{\infty})\times
    \widetilde{N}_{[]}(\frac{\varepsilon}{2c_{3}},\mathcal{V}_{n},\|\cdot\|_{\infty}).
    \end{eqnarray*}
    
    Since $\Theta$ is a compact set on $\mathbb{R}$, by lemma \ref{lem: bracketing_number} and traditional volume argument, we have that $\widetilde{N}_{[]}(\frac{\varepsilon}{2c_{3}},\Theta,\|\cdot\|_{\infty})\leq N_{[]}(\frac{\varepsilon}{4c_{3}},\Theta,\|\cdot\|_{\infty})\lesssim\frac{1}{\varepsilon}$.
    
    For $\widetilde{N}_{[]}(\frac{\varepsilon}{2c_{3}},\mathcal{V}_{n},\|\cdot\|_{\infty})$, by lemma \ref{lem: bracketing_number}, we have that $\widetilde{N}_{[]}(\frac{\varepsilon}{2c_{3}},\mathcal{V}_{n},\|\cdot\|_{\infty})=\widetilde{N}(\frac{\varepsilon}{2c_{3}},\mathcal{V}_{n},\|\cdot\|_{\infty})$. By \cite[Lemma 2]{chen1998sieve}, there exists a constant $c_{4}>0$ such that $\|\widehat{\nu}_{1}-\widehat{\nu}_{2}\|_{\infty}\leq c_{4}\|\widehat{\nu}_{1}-\widehat{\nu}_{2}\|_{2}^{s_{h}}$, which leads to $\widetilde{N}(\frac{\varepsilon}{2c_{3}},\mathcal{V}_{n},\|\cdot\|_{\infty})\leq \widetilde{N}(\frac{\varepsilon^{1/s_{\nu}}}{(2c_{3}c_{4})^{1/s_{\nu}}},\mathcal{V}_{n},\|\cdot\|_{2})$. By lemma \ref{lem: covering_number} we further have $\widetilde{N}(\frac{\varepsilon^{1/s_{\nu}}}{(2c_{3}c_{4})^{1/s_{\nu}}},\mathcal{V}_{n},\|\cdot\|_{2}) \leq N(\frac{\varepsilon^{1/s_{\nu}}}{2(2c_{3}c_{4})^{1/s_{\nu}}},\mathcal{V}_{n},\|\cdot\|_{2})$. Let $c_{\nu}=\frac{1}{2(2c_{3}c_{4})^{1/s_{\nu}}}$, we have that $\widetilde{N}_{[]}(\frac{\varepsilon}{2c_{3}},\mathcal{V}_{n},\|\cdot\|_{\infty})\leq N(c_{\nu}\varepsilon^{1/s_{\nu}},\mathcal{V}_{n},\|\cdot\|_{2})$.
    
    Thus, finally we can obtain that
    \begin{eqnarray*}
    N_{[]}(\varepsilon,\mathcal{G}_{n},\|\cdot\|_{\infty})\lesssim \frac{1}{\varepsilon} N(c_{\nu}\varepsilon^{1/s_{\nu}},\mathcal{V}_{n},\|\cdot\|_{2}).
    \end{eqnarray*}
\end{proof}


    \section{Experimental details}
    \subsection{Dataset summary}\label{sec: dataset_summary}
    We report summaries of descriptive statistics of the $6$ benchmark datasets used in section \ref{sec: benchmark_data} in table \ref{tab: datasets}.
    \begin{table}[]
        \centering
        \caption{Descriptive statistics of benchmark datasets}
        \begin{tabular}{l c c c c c c}
            \toprule
                              & METABRIC & RotGBSG    & FLCHAIN & SUPPORT & MIMIC-III & KKBOX     \\
            \midrule
            Size              & $1904$   & $2232$  & $6524$  & $8873$     & $35953$ & $2646746$ \\
            Censoring rate    & $0.423$  & $0.432$ & $0.699$ & $0.320$    & $0.901$ & $0.280$   \\  
            Features          & $9$      & $7$     & $8$     & $14$       & $26$    & $15$      \\
            \bottomrule
        \end{tabular}
        \label{tab: datasets}
    \end{table}
    \subsection{Details of synthetic experiments}\label{sec: synthetic_details}
    Since the true model is assumed to be of PF form, we generate event time according to the following transformed regression model \cite{dabrowska1988partial}:
    \begin{align}\label{eqn: transformation_mdoel}
        \log H(\tilde{T}) = -m(Z) + \epsilon,
    \end{align}
    where $H(t) = \int_0^t e^{h(s)} ds$ with $h$ defined in \eqref{eqn: proportional_frailty}. The error term $\epsilon$ is generated such that $e^{\epsilon}$ has cumulative hazard function $G_\theta$. The formulation \eqref{eqn: transformation_mdoel} is the equivalent to \eqref{eqn: proportional_frailty} \cite{dabrowska1988partial, cuzick1988rank, kosorok2004robust}. In our experiments, the covariates are of dimension $5$, sampled independently from the uniform distribution over $[0, 1]$. We set $h(t) = t$ and hence $H(t) = e^t$. The function form of $m(Z)$ is set to be $m(Z) = \sin(\langle Z, \beta \rangle) + \langle\sin(Z), \beta \rangle$, where $\beta = (0.1, 0.2, 0.3, 0.4, 0.5)$. Then censoring time $C$ is generated according to 
    \begin{align}
        \log H(C) = -m(Z) + \epsilon_C,
    \end{align}
    which reuses covariate $Z$, and draws independently a noise vector $\epsilon_C$ such that the censoring ratio is controlled at around $40\%$. We generate three datasets with $n \in \{1000, 5000, 10000\}$ respectively. \par
    \textbf{Hyperparameter configurations} We specify below the network architectures and optimization configurations used in all the tasks:
    \begin{description}
        \item[PF scheme: ] For both $\widehat{m}$ and $\widehat{h}$, we use $64$ hidden units for $n=1000$, $128$ hidden units for $n=5000$ and $256$ hidden units for $n=10000$. We train each model for $100$ epochs with batch size $128$, optimized using Adam with learning rate $0.0001$, and no weight decay.
        \item[FN scheme: ] For both $\widehat{\nu}$, we use $64$ hidden units for $n=1000$, $128$ hidden units for $n=5000$ and $256$ hidden units for $n=10000$. We train each model for $100$ epochs with batch size $128$, optimized using Adam with learning rate $0.0001$, and no weight decay.
    \end{description}
    \subsection{Details of public data experiments}\label{sec: public_data_details}
    \textbf{Dataset preprocessing} For METABRIC, RotGBSG, FLCHAIN, SUPPORT and KKBOX dataset, we take the version provided in the {\fontfamily{qcr}\selectfont pycox} package \cite{kvamme2019time}. We standardize continuous features into zero mean and unit variance and do one-hot encodings for all categorical features. For the MIMIC-III dataset, we follow the preprocessing routines in \cite{purushotham2018benchmarking} which extracts $26$ features. The event of interest is defined as the mortality after admission, and the censored time is defined as the last time of being discharged from the hospital. The definition is similar to that in \cite{tang2022soden}. But since the dataset is not open sourced, according to our implementation the resulting dataset exhibits a much higher censoring rate ($90.2\%$ as compared to $61.0\%$ as reported in the SODEN paper \cite{tang2022soden}). Since the major purpose of this paper is for the proposal of the NFM framework, We use our own version of the processed dataset to further verify the predictive performance of NFM.\par 
    \textbf{Hyperparameter configurations} We follow the general training template that uses MLP as all nonparametric function approximators (i.e., $\widehat{m}$ and $\widehat{h}$ in the PF scheme, and $\widehat{\nu}$ in the FN scheme), and train for $100$ epochs across all datasets using Adam as the optimizer. The tunable parameters and their respective tuning ranges are reported as follows:
    \begin{description}
        \item[Number of layers (network depth)] We tune the network depth $L \in \{2, 3, 4\}$. Typically, the performance of two-layer MLPs is sufficiently satisfactory.
        \item[Number of hidden units in each layer (network width)] We tune the network width $W \in \{2^k, 5 \le k \le 10 \}$.  
        \item[Optional dropout] We optionally apply dropout with probability $p \in \{0.1, 0.2, 0.3, 0.5, 0.7\}$. 
        \item[Batch size] We tune batch size within the range $\{128, 256, 512\}$, in the KKBOX dataset, we also tested with larger batch sizes $\{1024\}$.
        \item[Learning rate and weight decay] We tune both the learning rate and weight decay coefficient of Adam within range $\{0.01, 0.001, 0.0001\}$.
        \item[Frailty specification] We tested gamma frailty, Box-Cox transformation frailty, and $\text{IGG}(\alpha)$ frailty with $\alpha \in \{0, 0.25, 0.75\}$. Here note that $\text{IGG}(0.5)$ is equivalent to gamma frailty. We also empirically tried to set $\alpha$ to be a learnable parameter and found that this additional flexibility provides little performance improvement regarding the datasets used for evaluation.
    \end{description}
    
    \subsection{Implementations}
    We use \qcr{pytorch} to implement NFM. \textbf{The source code is provided in the supplementary material}. For the baseline models:
    \begin{itemize}[leftmargin=*]
        \item We use the implementations of \ph, GBM, and RSF from the \texttt{sksurv} package \cite{sksurv}, for the KKBOX dataset, we use the XGBoost library \cite{chen2016xgboost} to implement GBM and RSF, which might yield some performance degradation.
        \item We use the \qcr{pycox} package to implement DeepSurv, CoxTime, and DeepHit models.
        \item We use the official code provided in the SODEN paper \cite{tang2022soden} to implement SODEN.
        \item We obtain results of SuMo and DeepEH based on our re-implementations.
    \end{itemize}
     
    \section{Additional experiments}\label{sec: additional_experiments}
    \subsection{Recovery assessment of $m(Z)$ in PF scheme}\label{sec: m_z}
    We plot empirical recovery results targeting the $m$ function in \eqref{eqn: proportional_frailty} in figure \ref{fig: synthesis_pf_m}. The result demonstrates satisfactory recovery with a moderate amount of data, i.e., $n \ge 1000$.
    \begin{figure}
        \begin{subfigure}[b]{.33\linewidth}
        \resizebox{\linewidth}{0.8\linewidth}{
\begin{tikzpicture}

\definecolor{darkgray176}{RGB}{176,176,176}
\definecolor{lightgray204}{RGB}{204,204,204}
\definecolor{steelblue31119180}{RGB}{31,119,180}

\begin{axis}[
legend cell align={left},
legend style={
  font=\tiny,
  fill opacity=0.8,
  draw opacity=1,
  text opacity=1,
  at={(0.03,0.97)},
  anchor=north west,
  draw=lightgray204
},
scale only axis,
width=4cm,
height=3cm,
tick align=outside,
tick pos=left,
title={\fontsize{6}{6}\selectfont $N=1000$},
tick label style={font=\tiny},
every tick/.style={
black,
semithick,
},
x label style={at={(axis description cs:0.5,-0.1)},anchor=north,font=\tiny},
y label style={at={(axis description cs:-0.06,.5)},rotate=90,anchor=south,font=\tiny},
x grid style={darkgray176},
xlabel={index},
xmin=-4.95, xmax=103.95,
xtick style={color=black},
y grid style={darkgray176},
ylabel={$m(Z)$},
ymin=-0.75, ymax=2.5,
ytick style={color=black}
]
\addplot [width=1pt, steelblue31119180, mark=*, mark size=0.5, mark options={solid}]
table {%
0 1.57497000694275
1 1.15095257759094
2 1.24041545391083
3 1.3399760723114
4 1.28155565261841
5 0.879355072975159
6 0.719038009643555
7 1.02315044403076
8 0.771758854389191
9 1.35975790023804
10 0.400826245546341
11 0.611046493053436
12 1.01262211799622
13 1.00679564476013
14 0.508017241954803
15 0.83271312713623
16 1.10645389556885
17 0.976385712623596
18 0.501031994819641
19 0.602039337158203
20 1.07092881202698
21 1.28955626487732
22 0.877973675727844
23 0.728686451911926
24 0.900317311286926
25 0.748838305473328
26 0.897845506668091
27 0.782133460044861
28 0.879701375961304
29 0.946312308311462
30 1.44770526885986
31 1.15533781051636
32 0.787004947662354
33 1.12570643424988
34 0.875810384750366
35 1.34348261356354
36 1.16493248939514
37 1.2958949804306
38 0.902328908443451
39 0.78005850315094
40 0.502965211868286
41 0.895460367202759
42 0.568686366081238
43 1.49195671081543
44 0.808770895004272
45 1.38960492610931
46 1.40202236175537
47 1.22090935707092
48 0.618276715278625
49 0.541274189949036
50 0.57879626750946
51 1.24983549118042
52 1.17480874061584
53 1.09575653076172
54 0.916525602340698
55 1.15049576759338
56 0.703235447406769
57 0.648393869400024
58 0.828435242176056
59 0.30644565820694
60 0.564027011394501
61 0.975495159626007
62 0.347937047481537
63 0.907636761665344
64 1.01808476448059
65 0.8969566822052
66 0.817118585109711
67 1.30936086177826
68 1.01583576202393
69 1.00723612308502
70 1.23634219169617
71 1.18156313896179
72 0.455988466739655
73 0.742524027824402
74 1.25001001358032
75 1.11372780799866
76 1.37222063541412
77 0.427757084369659
78 1.20141196250916
79 0.922923922538757
80 0.612215280532837
81 1.25139260292053
82 0.437425285577774
83 0.570523858070374
84 0.950319588184357
85 0.582384705543518
86 0.560596466064453
87 1.1056444644928
88 0.855743169784546
89 1.25217723846436
90 0.801432192325592
91 1.00741791725159
92 0.87658429145813
93 0.5574049949646
94 0.7025306224823
95 0.990389108657837
96 0.755933701992035
97 0.963387846946716
98 1.30714702606201
99 0.795306086540222
};
\addlegendentry{ground truth}
\addplot [width=1pt, red, opacity=0.7]
table {%
0 1.46627020835876
1 0.897751450538635
2 1.07277381420135
3 1.16461038589478
4 1.09390449523926
5 0.753188788890839
6 0.591244339942932
7 1.08397006988525
8 0.615553379058838
9 1.31365919113159
10 0.338740646839142
11 0.50578236579895
12 0.773455142974854
13 0.830889105796814
14 0.366430848836899
15 0.36038863658905
16 1.04778110980988
17 0.793281197547913
18 0.37189507484436
19 0.452165931463242
20 0.691297888755798
21 1.2732857465744
22 0.872041702270508
23 0.562065482139587
24 0.782313823699951
25 0.616823613643646
26 0.728181123733521
27 0.598521888256073
28 0.790544748306274
29 0.845514953136444
30 1.18972682952881
31 1.17812585830688
32 0.558782398700714
33 1.34303951263428
34 0.775462746620178
35 1.02600860595703
36 0.89701521396637
37 1.43358910083771
38 0.605113983154297
39 0.483703672885895
40 0.443449139595032
41 1.14561867713928
42 0.310254216194153
43 1.7034512758255
44 0.66513454914093
45 1.35068941116333
46 1.27393507957458
47 1.3245313167572
48 0.231390058994293
49 0.167154967784882
50 0.560238242149353
51 1.02528369426727
52 0.923549592494965
53 1.2445056438446
54 1.01678991317749
55 1.03876209259033
56 0.42801433801651
57 0.382587969303131
58 0.74907124042511
59 0.106025271117687
60 0.495606064796448
61 0.761173009872437
62 0.365689039230347
63 0.600231647491455
64 0.692326307296753
65 0.377035200595856
66 0.580044746398926
67 1.44493722915649
68 0.963061809539795
69 0.93490594625473
70 1.22871708869934
71 0.890346467494965
72 0.554029226303101
73 0.393452376127243
74 1.35665416717529
75 1.10235333442688
76 1.06051921844482
77 0.347979128360748
78 0.79928183555603
79 0.746177554130554
80 0.547842383384705
81 1.25723242759705
82 0.30756676197052
83 0.356465876102448
84 0.571497797966003
85 0.57218211889267
86 0.636925339698792
87 0.644957780838013
88 0.564066052436829
89 0.947020053863525
90 0.492300033569336
91 0.538690328598022
92 0.70821475982666
93 0.47789254784584
94 0.440053045749664
95 0.799906253814697
96 0.520761728286743
97 0.834645509719849
98 1.22271060943604
99 0.531633377075195
};
\addlegendentry{estimated}
\end{axis}

\end{tikzpicture}
        }
        \end{subfigure}
        \begin{subfigure}[b]{.33\linewidth}
        \resizebox{\linewidth}{0.8\linewidth}{
\begin{tikzpicture}

\definecolor{darkgray176}{RGB}{176,176,176}
\definecolor{lightgray204}{RGB}{204,204,204}
\definecolor{steelblue31119180}{RGB}{31,119,180}

\begin{axis}[
legend cell align={left},
legend style={
  font=\tiny,
  fill opacity=0.8,
  draw opacity=1,
  text opacity=1,
  at={(0.03,0.97)},
  anchor=north west,
  draw=lightgray204
},
scale only axis,
width=4cm,
height=3cm,
tick align=outside,
tick pos=left,
title={\fontsize{6}{6}\selectfont $N=5000$},
tick label style={font=\tiny},
every tick/.style={
black,
semithick,
},
x label style={at={(axis description cs:0.5,-0.1)},anchor=north,font=\tiny},
y label style={at={(axis description cs:-0.06,.5)},rotate=90,anchor=south,font=\tiny},
x grid style={darkgray176},
xlabel={index},
xmin=-4.95, xmax=103.95,
xtick style={color=black},
y grid style={darkgray176},
ylabel={$m(Z)$},
ymin=-0.75, ymax=2.5,
ytick style={color=black}
]
\addplot [width=1pt, steelblue31119180, mark=*, mark size=0.5, mark options={solid}]
table {%
0 1.57497000694275
1 1.15095257759094
2 1.24041545391083
3 1.3399760723114
4 1.28155565261841
5 0.879355072975159
6 0.719038009643555
7 1.02315044403076
8 0.771758854389191
9 1.35975790023804
10 0.400826245546341
11 0.611046493053436
12 1.01262211799622
13 1.00679564476013
14 0.508017241954803
15 0.83271312713623
16 1.10645389556885
17 0.976385712623596
18 0.501031994819641
19 0.602039337158203
20 1.07092881202698
21 1.28955626487732
22 0.877973675727844
23 0.728686451911926
24 0.900317311286926
25 0.748838305473328
26 0.897845506668091
27 0.782133460044861
28 0.879701375961304
29 0.946312308311462
30 1.44770526885986
31 1.15533781051636
32 0.787004947662354
33 1.12570643424988
34 0.875810384750366
35 1.34348261356354
36 1.16493248939514
37 1.2958949804306
38 0.902328908443451
39 0.78005850315094
40 0.502965211868286
41 0.895460367202759
42 0.568686366081238
43 1.49195671081543
44 0.808770895004272
45 1.38960492610931
46 1.40202236175537
47 1.22090935707092
48 0.618276715278625
49 0.541274189949036
50 0.57879626750946
51 1.24983549118042
52 1.17480874061584
53 1.09575653076172
54 0.916525602340698
55 1.15049576759338
56 0.703235447406769
57 0.648393869400024
58 0.828435242176056
59 0.30644565820694
60 0.564027011394501
61 0.975495159626007
62 0.347937047481537
63 0.907636761665344
64 1.01808476448059
65 0.8969566822052
66 0.817118585109711
67 1.30936086177826
68 1.01583576202393
69 1.00723612308502
70 1.23634219169617
71 1.18156313896179
72 0.455988466739655
73 0.742524027824402
74 1.25001001358032
75 1.11372780799866
76 1.37222063541412
77 0.427757084369659
78 1.20141196250916
79 0.922923922538757
80 0.612215280532837
81 1.25139260292053
82 0.437425285577774
83 0.570523858070374
84 0.950319588184357
85 0.582384705543518
86 0.560596466064453
87 1.1056444644928
88 0.855743169784546
89 1.25217723846436
90 0.801432192325592
91 1.00741791725159
92 0.87658429145813
93 0.5574049949646
94 0.7025306224823
95 0.990389108657837
96 0.755933701992035
97 0.963387846946716
98 1.30714702606201
99 0.795306086540222
};
\addlegendentry{ground truth}
\addplot [width=1pt, red, opacity=0.7]
table {%
0 1.50803232192993
1 1.06128025054932
2 1.12493681907654
3 1.22440195083618
4 1.18201470375061
5 0.815121650695801
6 0.517638862133026
7 0.818450629711151
8 0.705982446670532
9 1.21139216423035
10 0.376715093851089
11 0.525013089179993
12 0.954158365726471
13 0.797038674354553
14 0.399192303419113
15 0.712458729743958
16 0.944109857082367
17 0.795746445655823
18 0.399291008710861
19 0.57230818271637
20 0.924117922782898
21 1.07906293869019
22 0.766451478004456
23 0.538043916225433
24 0.837923169136047
25 0.682390451431274
26 0.81484317779541
27 0.721214056015015
28 0.663907945156097
29 0.811123311519623
30 1.33946025371552
31 1.0397834777832
32 0.681319117546082
33 1.09027886390686
34 0.8216832280159
35 1.26218032836914
36 1.04172921180725
37 1.28559708595276
38 0.803898930549622
39 0.821199893951416
40 0.455263197422028
41 0.839714169502258
42 0.444897413253784
43 1.39755117893219
44 0.703104853630066
45 1.26050078868866
46 1.38027346134186
47 1.05815649032593
48 0.499504446983337
49 0.413176476955414
50 0.568924427032471
51 1.17266356945038
52 1.04647445678711
53 0.91742330789566
54 0.792809367179871
55 1.04525542259216
56 0.606478571891785
57 0.533961236476898
58 0.710926651954651
59 0.165618553757668
60 0.597756743431091
61 0.931312620639801
62 0.312218844890594
63 0.756780326366425
64 0.960880219936371
65 0.713736891746521
66 0.739117026329041
67 1.24435901641846
68 0.867084801197052
69 0.939327597618103
70 1.13351047039032
71 1.08258247375488
72 0.451942712068558
73 0.655619323253632
74 1.15557622909546
75 0.962614774703979
76 1.27903032302856
77 0.388431131839752
78 1.09044563770294
79 0.797232568264008
80 0.530967891216278
81 1.12726151943207
82 0.339643508195877
83 0.528816223144531
84 0.99182540178299
85 0.569911956787109
86 0.522472143173218
87 1.04966354370117
88 0.719271063804626
89 1.16427576541901
90 0.952123284339905
91 0.951805114746094
92 0.77667510509491
93 0.496179670095444
94 0.578585028648376
95 0.864468574523926
96 0.561168968677521
97 0.842448234558105
98 1.15654265880585
99 0.805738508701324
};
\addlegendentry{estimated}
\end{axis}

\end{tikzpicture}
        }
        \end{subfigure}
        \begin{subfigure}[b]{.33\linewidth}
        \resizebox{\linewidth}{0.8\linewidth}{
\begin{tikzpicture}

\definecolor{darkgray176}{RGB}{176,176,176}
\definecolor{lightgray204}{RGB}{204,204,204}
\definecolor{steelblue31119180}{RGB}{31,119,180}

\begin{axis}[
legend cell align={left},
legend style={
  font=\tiny,
  fill opacity=0.8,
  draw opacity=1,
  text opacity=1,
  at={(0.03,0.97)},
  anchor=north west,
  draw=lightgray204
},
scale only axis,
width=4cm,
height=3cm,
tick align=outside,
tick pos=left,
title={\fontsize{6}{6}\selectfont $N=10000$},
tick label style={font=\tiny},
every tick/.style={
black,
semithick,
},
x label style={at={(axis description cs:0.5,-0.1)},anchor=north,font=\tiny},
y label style={at={(axis description cs:-0.06,.5)},rotate=90,anchor=south,font=\tiny},
x grid style={darkgray176},
xlabel={index},
xmin=-4.95, xmax=103.95,
xtick style={color=black},
y grid style={darkgray176},
ylabel={$m(Z)$},
ymin=-0.75, ymax=2.5,
ytick style={color=black}
]
\addplot [width=1pt, steelblue31119180, mark=*, mark size=0.5, mark options={solid}]
table {%
0 1.57497000694275
1 1.15095257759094
2 1.24041545391083
3 1.3399760723114
4 1.28155565261841
5 0.879355072975159
6 0.719038009643555
7 1.02315044403076
8 0.771758854389191
9 1.35975790023804
10 0.400826245546341
11 0.611046493053436
12 1.01262211799622
13 1.00679564476013
14 0.508017241954803
15 0.83271312713623
16 1.10645389556885
17 0.976385712623596
18 0.501031994819641
19 0.602039337158203
20 1.07092881202698
21 1.28955626487732
22 0.877973675727844
23 0.728686451911926
24 0.900317311286926
25 0.748838305473328
26 0.897845506668091
27 0.782133460044861
28 0.879701375961304
29 0.946312308311462
30 1.44770526885986
31 1.15533781051636
32 0.787004947662354
33 1.12570643424988
34 0.875810384750366
35 1.34348261356354
36 1.16493248939514
37 1.2958949804306
38 0.902328908443451
39 0.78005850315094
40 0.502965211868286
41 0.895460367202759
42 0.568686366081238
43 1.49195671081543
44 0.808770895004272
45 1.38960492610931
46 1.40202236175537
47 1.22090935707092
48 0.618276715278625
49 0.541274189949036
50 0.57879626750946
51 1.24983549118042
52 1.17480874061584
53 1.09575653076172
54 0.916525602340698
55 1.15049576759338
56 0.703235447406769
57 0.648393869400024
58 0.828435242176056
59 0.30644565820694
60 0.564027011394501
61 0.975495159626007
62 0.347937047481537
63 0.907636761665344
64 1.01808476448059
65 0.8969566822052
66 0.817118585109711
67 1.30936086177826
68 1.01583576202393
69 1.00723612308502
70 1.23634219169617
71 1.18156313896179
72 0.455988466739655
73 0.742524027824402
74 1.25001001358032
75 1.11372780799866
76 1.37222063541412
77 0.427757084369659
78 1.20141196250916
79 0.922923922538757
80 0.612215280532837
81 1.25139260292053
82 0.437425285577774
83 0.570523858070374
84 0.950319588184357
85 0.582384705543518
86 0.560596466064453
87 1.1056444644928
88 0.855743169784546
89 1.25217723846436
90 0.801432192325592
91 1.00741791725159
92 0.87658429145813
93 0.5574049949646
94 0.7025306224823
95 0.990389108657837
96 0.755933701992035
97 0.963387846946716
98 1.30714702606201
99 0.795306086540222
};
\addlegendentry{ground truth}
\addplot [width=1pt, red, opacity=0.7]
table {%
0 1.506915807724
1 0.995303273200989
2 1.08683085441589
3 1.24711966514587
4 1.16423368453979
5 0.650518298149109
6 0.703270554542542
7 0.958547472953796
8 0.597302913665771
9 1.31974005699158
10 0.294819295406342
11 0.516876220703125
12 0.903245329856873
13 0.911450386047363
14 0.417524814605713
15 0.763852477073669
16 0.97513222694397
17 0.886940121650696
18 0.412055432796478
19 0.49459308385849
20 1.04031980037689
21 1.24880170822144
22 0.821953773498535
23 0.745874464511871
24 0.652772784233093
25 0.646919131278992
26 0.744763314723969
27 0.597540438175201
28 0.926343441009521
29 0.864323019981384
30 1.38812637329102
31 1.01368355751038
32 0.725333333015442
33 0.986115753650665
34 0.633220672607422
35 1.21857452392578
36 1.13362336158752
37 1.23694610595703
38 0.776928186416626
39 0.681355953216553
40 0.411112397909164
41 0.735036849975586
42 0.482979118824005
43 1.43013787269592
44 0.676398873329163
45 1.25818407535553
46 1.34350800514221
47 1.07970094680786
48 0.581497967243195
49 0.511860489845276
50 0.360148280858994
51 1.20105934143066
52 1.15568494796753
53 0.994195818901062
54 0.737276077270508
55 0.998487591743469
56 0.671376585960388
57 0.525413691997528
58 0.710247159004211
59 0.347412526607513
60 0.468660295009613
61 0.775604486465454
62 0.274635404348373
63 0.867542743682861
64 0.882698893547058
65 0.827828109264374
66 0.742825269699097
67 1.21224117279053
68 0.886755347251892
69 0.821677803993225
70 1.13378167152405
71 1.0764936208725
72 0.440938174724579
73 0.644577980041504
74 1.14731240272522
75 0.974859297275543
76 1.27189207077026
77 0.37605282664299
78 1.1191041469574
79 0.847446978092194
80 0.442495405673981
81 1.14640414714813
82 0.378545731306076
83 0.456591725349426
84 0.883922100067139
85 0.527442693710327
86 0.46040090918541
87 0.997472882270813
88 0.802040696144104
89 1.12571859359741
90 0.699444055557251
91 0.889258027076721
92 0.71195775270462
93 0.54178375005722
94 0.674866199493408
95 0.848102509975433
96 0.794758081436157
97 0.853863000869751
98 1.20589900016785
99 0.667768478393555
};
\addlegendentry{estimated}
\end{axis}

\end{tikzpicture}
        }
        \end{subfigure}
        \caption{Visualizations of synthetic data results under the PF scheme of NFM framework, regarding empirical recovery of the $m$ function in \eqref{eqn: proportional_frailty}}
        \label{fig: synthesis_pf_m}
    \end{figure}
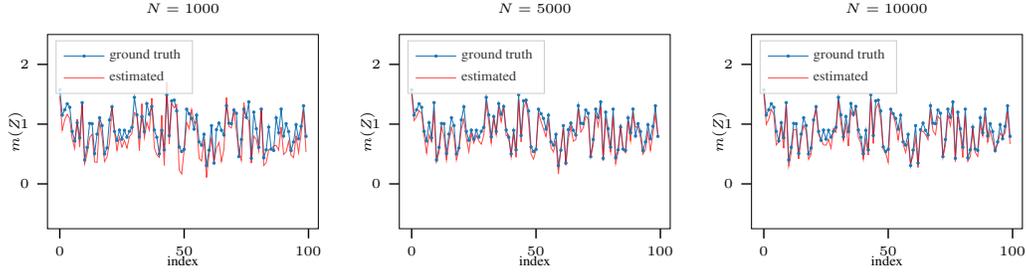

    \subsection{Recovery assessment of survival functions}\label{sec: surv}
    To assess the recovery performance of NFM with respect to survival functions, we consider the following setup: under the same data generation framework as in section \ref{sec: synthetic_details}, we compute the test feature $\bar{Z}$ as the sample mean of all the $100$ hold-out test points. And plot $\widehat{S}(t|\bar{Z})$ against the ground truth $S(t|\bar{Z})$ regarding both PF and FN schemes. The results are shown in figure \ref{fig: surv_recovery}. The results suggest that both scheme provides accurate estimation of survival functions when the sample size is sufficiently large. 
    \begin{figure}
        \centering
        \begin{subfigure}[b]{.32\linewidth}
        \resizebox{\linewidth}{0.8\linewidth}{
\begin{tikzpicture}

\definecolor{darkgray176}{RGB}{176,176,176}
\definecolor{lightgray204}{RGB}{204,204,204}
\definecolor{steelblue31119180}{RGB}{31,119,180}

\begin{axis}[
legend cell align={left},
legend style={
  font=\tiny,
  fill opacity=0.8,
  draw opacity=1,
  text opacity=1,
  at={(0.40,0.97)},
  anchor=north west,
  draw=lightgray204
},
scale only axis,
width=4cm,
height=3cm,
tick align=outside,
tick pos=left,
title={\fontsize{6}{6}\selectfont $N=1000$},
tick label style={font=\tiny},
every tick/.style={
black,
semithick,
},
x label style={at={(axis description cs:0.5,-0.1)},anchor=north,font=\tiny},
y label style={at={(axis description cs:-0.06,.5)},rotate=90,anchor=south,font=\tiny},
x grid style={darkgray176},
xlabel={$t$},
xmin=-0.157433523936197, xmax=3.40243121623062,
xtick style={color=black},
y grid style={darkgray176},
ylabel={$S(t|Z)$},
ymin=-0.00870856288820505, ymax=1.04395221937448,
ytick={0, 0.25, 0.75, 1},
ytick style={color=black}
]
\addplot [semithick, steelblue31119180, mark=*, mark size=0.5, mark options={solid}]
table {%
0.00437850970774889 0.995631098747253
0.00542990444228053 0.994584858417511
0.00595268839970231 0.99406498670578
0.00755271734669805 0.992475688457489
0.0123842898756266 0.987692058086395
0.0166294146329165 0.983508110046387
0.0172695983201265 0.982878684997559
0.0208375565707684 0.979378044605255
0.0214008055627346 0.978826522827148
0.0237781554460526 0.976502299308777
0.0328329727053642 0.967700242996216
0.0398957096040249 0.960889637470245
0.044037751853466 0.956917881965637
0.0456374324858189 0.955388247966766
0.0461156778037548 0.954931497573853
0.0586898550391197 0.942999243736267
0.0675230771303177 0.934706151485443
0.0703325197100639 0.93208384513855
0.0790752917528152 0.923970401287079
0.0811049863696098 0.922096908092499
0.0839139670133591 0.919510304927826
0.0910551473498344 0.912967383861542
0.0914278626441956 0.912627160549164
0.0929863452911377 0.911205887794495
0.0971563383936882 0.907414138317108
0.0996829569339752 0.905124306678772
0.0998237505555153 0.904996871948242
0.115672208368778 0.890767157077789
0.127606585621834 0.880199551582336
0.133811295032501 0.874755144119263
0.140404969453812 0.869006276130676
0.141848400235176 0.86775279045105
0.143481746315956 0.866336584091187
0.14919476211071 0.861401319503784
0.165343880653381 0.847602188587189
0.166629523038864 0.846513211727142
0.17111437022686 0.842725217342377
0.175462424755096 0.83906888961792
0.179158613085747 0.835973262786865
0.184549406170845 0.831478834152222
0.204394772648811 0.815140545368195
0.206485122442245 0.813438355922699
0.228909358382225 0.795400619506836
0.229227304458618 0.795147776603699
0.231799572706223 0.793105065822601
0.25389552116394 0.775772869586945
0.295848220586777 0.74390035867691
0.301604896783829 0.739630222320557
0.317365616559982 0.72806453704834
0.321536540985107 0.725034117698669
0.323987185955048 0.723259508609772
0.325569301843643 0.722116112709045
0.327896296977997 0.720437705516815
0.333507657051086 0.716406404972076
0.354743450880051 0.701353311538696
0.368787109851837 0.691572666168213
0.374075591564178 0.687924921512604
0.376269787549973 0.686417102813721
0.392039835453033 0.675677180290222
0.448523104190826 0.638570547103882
0.467623382806778 0.626489400863647
0.51166170835495 0.599498510360718
0.528185546398163 0.589673936367035
0.541024386882782 0.582151591777802
0.568550884723663 0.566345512866974
0.581681251525879 0.558957815170288
0.583209156990051 0.558104455471039
0.639981210231781 0.527302324771881
0.662271738052368 0.515678524971008
0.67755538225174 0.507856965065002
0.712697148323059 0.490319907665253
0.730794370174408 0.481526345014572
0.732405066490173 0.48075133562088
0.762117743492126 0.466677069664001
0.773440182209015 0.461422920227051
0.783490061759949 0.456808924674988
0.78608775138855 0.455623835325241
0.790945649147034 0.453415811061859
0.791604399681091 0.453117221593857
0.804874062538147 0.447144210338593
0.805576264858246 0.446830362081528
0.81831419467926 0.441174775362015
0.881467580795288 0.414174646139145
0.906880497932434 0.403781890869141
0.996952712535858 0.369002193212509
1.03600311279297 0.354870229959488
1.0573570728302 0.347372680902481
1.17803061008453 0.307884484529495
1.22611820697784 0.293429404497147
1.35938203334808 0.256819427013397
1.36120343208313 0.256352096796036
1.40997350215912 0.244149759411812
1.41032409667969 0.244064167141914
1.42464458942413 0.240593954920769
1.43722772598267 0.237585499882698
1.68432581424713 0.185569494962692
1.87879085540771 0.152774721384048
2.93655371665955 0.0530482344329357
2.93978977203369 0.0528768450021744
3.24061918258667 0.0391396544873714
};
\addlegendentry{ground truth}
\addplot [width=1pt, red, opacity=0.7]
table {%
0.00437850970774889 0.996104001998901
0.00542990444228053 0.995167970657349
0.00595268839970231 0.994702279567719
0.00755271734669805 0.993276000022888
0.0123842898756266 0.98895651102066
0.0166294146329165 0.985145926475525
0.0172695983201265 0.98456996679306
0.0208375565707684 0.981354475021362
0.0214008055627346 0.980845868587494
0.0237781554460526 0.978696823120117
0.0328329727053642 0.970471501350403
0.0398957096040249 0.964013040065765
0.044037751853466 0.960208296775818
0.0456374324858189 0.958735406398773
0.0461156778037548 0.958294808864594
0.0586898550391197 0.946653723716736
0.0675230771303177 0.938418328762054
0.0703325197100639 0.935790419578552
0.0790752917528152 0.927588045597076
0.0811049863696098 0.92567902803421
0.0839139670133591 0.923034310340881
0.0910551473498344 0.916297793388367
0.0914278626441956 0.915945708751678
0.0929863452911377 0.914472997188568
0.0971563383936882 0.910529017448425
0.0996829569339752 0.908136963844299
0.0998237505555153 0.908003568649292
0.115672208368778 0.892966985702515
0.127606585621834 0.881620049476624
0.133811295032501 0.875717163085938
0.140404969453812 0.869444191455841
0.141848400235176 0.868071138858795
0.143481746315956 0.866517543792725
0.14919476211071 0.861084818840027
0.165343880653381 0.845748543739319
0.166629523038864 0.844529449939728
0.17111437022686 0.840279340744019
0.175462424755096 0.836163282394409
0.179158613085747 0.832667946815491
0.184549406170845 0.827576875686646
0.204394772648811 0.808916389942169
0.206485122442245 0.806959450244904
0.228909358382225 0.786098003387451
0.229227304458618 0.785804510116577
0.231799572706223 0.783433198928833
0.25389552116394 0.763274550437927
0.295848220586777 0.726230084896088
0.301604896783829 0.721282303333282
0.317365616559982 0.707896709442139
0.321536540985107 0.704392969608307
0.323987185955048 0.702344000339508
0.325569301843643 0.701023697853088
0.327896296977997 0.699085593223572
0.333507657051086 0.694433569908142
0.354743450880051 0.677102744579315
0.368787109851837 0.66587108373642
0.374075591564178 0.661692917346954
0.376269787549973 0.659966349601746
0.392039835453033 0.647685825824738
0.448523104190826 0.60561192035675
0.467623382806778 0.592050492763519
0.51166170835495 0.56207001209259
0.528185546398163 0.551283955574036
0.541024386882782 0.543063461780548
0.568550884723663 0.52594667673111
0.581681251525879 0.518006205558777
0.583209156990051 0.517092287540436
0.639981210231781 0.484487324953079
0.662271738052368 0.472387492656708
0.67755538225174 0.464304715394974
0.712697148323059 0.446380734443665
0.730794370174408 0.437488585710526
0.732405066490173 0.436708718538284
0.762117743492126 0.422617107629776
0.773440182209015 0.417378902435303
0.783490061759949 0.412783354520798
0.78608775138855 0.411604881286621
0.790945649147034 0.409407436847687
0.791604399681091 0.409110575914383
0.804874062538147 0.403181284666061
0.805576264858246 0.402870416641235
0.81831419467926 0.397273540496826
0.881467580795288 0.370755285024643
0.906880497932434 0.360659211874008
0.996952712535858 0.327358603477478
1.03600311279297 0.314072608947754
1.0573570728302 0.307088434696198
1.17803061008453 0.271072804927826
1.22611820697784 0.258228242397308
1.35938203334808 0.226565092802048
1.36120343208313 0.2261683344841
1.40997350215912 0.21590293943882
1.41032409667969 0.215831249952316
1.42464458942413 0.212931528687477
1.43722772598267 0.210428223013878
1.68432581424713 0.168468460440636
1.87879085540771 0.143174260854721
2.93655371665955 0.0688162744045258
2.93978977203369 0.0686835870146751
3.24061918258667 0.0578004196286201
};
\addlegendentry{estimated}
\end{axis}

\end{tikzpicture}
        }
        \end{subfigure}
        \begin{subfigure}[b]{.32\linewidth}
        \resizebox{\linewidth}{0.8\linewidth}{
\begin{tikzpicture}

\definecolor{darkgray176}{RGB}{176,176,176}
\definecolor{lightgray204}{RGB}{204,204,204}
\definecolor{steelblue31119180}{RGB}{31,119,180}

\begin{axis}[
legend cell align={left},
legend style={
  font=\tiny,
  fill opacity=0.8,
  draw opacity=1,
  text opacity=1,
  at={(0.40,0.97)},
  anchor=north west,
  draw=lightgray204
},
scale only axis,
width=4cm,
height=3cm,
tick align=outside,
tick pos=left,
title={\fontsize{6}{6}\selectfont $N=5000$},
tick label style={font=\tiny},
every tick/.style={
black,
semithick,
},
x label style={at={(axis description cs:0.5,-0.1)},anchor=north,font=\tiny},
y label style={at={(axis description cs:-0.06,.5)},rotate=90,anchor=south,font=\tiny},
x grid style={darkgray176},
xlabel={$t$},
xmin=-0.157433523936197, xmax=3.40243121623062,
xtick style={color=black},
y grid style={darkgray176},
ylabel={$S(t|Z)$},
ymin=-0.00870856288820505, ymax=1.04395221937448,
ytick={0, 0.25, 0.75, 1},
ytick style={color=black}
]
\addplot [semithick, steelblue31119180, mark=*, mark size=0.5, mark options={solid}]
table {%
0.00437850970774889 0.995631098747253
0.00542990444228053 0.994584858417511
0.00595268839970231 0.99406498670578
0.00755271734669805 0.992475688457489
0.0123842898756266 0.987692058086395
0.0166294146329165 0.983508110046387
0.0172695983201265 0.982878684997559
0.0208375565707684 0.979378044605255
0.0214008055627346 0.978826522827148
0.0237781554460526 0.976502299308777
0.0328329727053642 0.967700242996216
0.0398957096040249 0.960889637470245
0.044037751853466 0.956917881965637
0.0456374324858189 0.955388247966766
0.0461156778037548 0.954931497573853
0.0586898550391197 0.942999243736267
0.0675230771303177 0.934706151485443
0.0703325197100639 0.93208384513855
0.0790752917528152 0.923970401287079
0.0811049863696098 0.922096908092499
0.0839139670133591 0.919510304927826
0.0910551473498344 0.912967383861542
0.0914278626441956 0.912627160549164
0.0929863452911377 0.911205887794495
0.0971563383936882 0.907414138317108
0.0996829569339752 0.905124306678772
0.0998237505555153 0.904996871948242
0.115672208368778 0.890767157077789
0.127606585621834 0.880199551582336
0.133811295032501 0.874755144119263
0.140404969453812 0.869006276130676
0.141848400235176 0.86775279045105
0.143481746315956 0.866336584091187
0.14919476211071 0.861401319503784
0.165343880653381 0.847602188587189
0.166629523038864 0.846513211727142
0.17111437022686 0.842725217342377
0.175462424755096 0.83906888961792
0.179158613085747 0.835973262786865
0.184549406170845 0.831478834152222
0.204394772648811 0.815140545368195
0.206485122442245 0.813438355922699
0.228909358382225 0.795400619506836
0.229227304458618 0.795147776603699
0.231799572706223 0.793105065822601
0.25389552116394 0.775772869586945
0.295848220586777 0.74390035867691
0.301604896783829 0.739630222320557
0.317365616559982 0.72806453704834
0.321536540985107 0.725034117698669
0.323987185955048 0.723259508609772
0.325569301843643 0.722116112709045
0.327896296977997 0.720437705516815
0.333507657051086 0.716406404972076
0.354743450880051 0.701353311538696
0.368787109851837 0.691572666168213
0.374075591564178 0.687924921512604
0.376269787549973 0.686417102813721
0.392039835453033 0.675677180290222
0.448523104190826 0.638570547103882
0.467623382806778 0.626489400863647
0.51166170835495 0.599498510360718
0.528185546398163 0.589673936367035
0.541024386882782 0.582151591777802
0.568550884723663 0.566345512866974
0.581681251525879 0.558957815170288
0.583209156990051 0.558104455471039
0.639981210231781 0.527302324771881
0.662271738052368 0.515678524971008
0.67755538225174 0.507856965065002
0.712697148323059 0.490319907665253
0.730794370174408 0.481526345014572
0.732405066490173 0.48075133562088
0.762117743492126 0.466677069664001
0.773440182209015 0.461422920227051
0.783490061759949 0.456808924674988
0.78608775138855 0.455623835325241
0.790945649147034 0.453415811061859
0.791604399681091 0.453117221593857
0.804874062538147 0.447144210338593
0.805576264858246 0.446830362081528
0.81831419467926 0.441174775362015
0.881467580795288 0.414174646139145
0.906880497932434 0.403781890869141
0.996952712535858 0.369002193212509
1.03600311279297 0.354870229959488
1.0573570728302 0.347372680902481
1.17803061008453 0.307884484529495
1.22611820697784 0.293429404497147
1.35938203334808 0.256819427013397
1.36120343208313 0.256352096796036
1.40997350215912 0.244149759411812
1.41032409667969 0.244064167141914
1.42464458942413 0.240593954920769
1.43722772598267 0.237585499882698
1.68432581424713 0.185569494962692
1.87879085540771 0.152774721384048
2.93655371665955 0.0530482344329357
2.93978977203369 0.0528768450021744
3.24061918258667 0.0391396544873714
};
\addlegendentry{ground truth}
\addplot [width=1pt, red, opacity=0.7]
table {%
0.00437850970774889 0.995184004306793
0.00542990444228053 0.994030892848969
0.00595268839970231 0.993458211421967
0.00755271734669805 0.991707026958466
0.0123842898756266 0.986437916755676
0.0166294146329165 0.981831610202789
0.0172695983201265 0.98113876581192
0.0208375565707684 0.977286696434021
0.0214008055627346 0.976680040359497
0.0237781554460526 0.974123477935791
0.0328329727053642 0.964446067810059
0.0398957096040249 0.956960678100586
0.044037751853466 0.952596426010132
0.0456374324858189 0.950916111469269
0.0461156778037548 0.95041424036026
0.0586898550391197 0.937311053276062
0.0675230771303177 0.928211450576782
0.0703325197100639 0.92533540725708
0.0790752917528152 0.916442036628723
0.0811049863696098 0.914389550685883
0.0839139670133591 0.911556720733643
0.0910551473498344 0.904395580291748
0.0914278626441956 0.904023468494415
0.0929863452911377 0.902469635009766
0.0971563383936882 0.898327350616455
0.0996829569339752 0.895828306674957
0.0998237505555153 0.895689249038696
0.115672208368778 0.880188286304474
0.127606585621834 0.868705809116364
0.133811295032501 0.862800061702728
0.140404969453812 0.856572151184082
0.141848400235176 0.855215668678284
0.143481746315956 0.853683412075043
0.14919476211071 0.848346829414368
0.165343880653381 0.833458304405212
0.166629523038864 0.832285463809967
0.17111437022686 0.828208386898041
0.175462424755096 0.82427567243576
0.179158613085747 0.820949792861938
0.184549406170845 0.816124796867371
0.204394772648811 0.798633456230164
0.206485122442245 0.796815633773804
0.228909358382225 0.777606546878815
0.229227304458618 0.777337968349457
0.231799572706223 0.77516907453537
0.25389552116394 0.756822764873505
0.295848220586777 0.723361253738403
0.301604896783829 0.718905925750732
0.317365616559982 0.706872582435608
0.321536540985107 0.703725934028625
0.323987185955048 0.701884269714355
0.325569301843643 0.700698375701904
0.327896296977997 0.698957920074463
0.333507657051086 0.694781899452209
0.354743450880051 0.679223656654358
0.368787109851837 0.669140338897705
0.374075591564178 0.665386557579041
0.376269787549973 0.663834810256958
0.392039835453033 0.652807533740997
0.448523104190826 0.614897191524506
0.467623382806778 0.602611601352692
0.51166170835495 0.575220108032227
0.528185546398163 0.565252661705017
0.541024386882782 0.557629406452179
0.568550884723663 0.541635692119598
0.581681251525879 0.534176647663116
0.583209156990051 0.533315598964691
0.639981210231781 0.502331972122192
0.662271738052368 0.490661382675171
0.67755538225174 0.482806146144867
0.712697148323059 0.46521458029747
0.730794370174408 0.456397891044617
0.732405066490173 0.455621659755707
0.762117743492126 0.441553443670273
0.773440182209015 0.436306029558182
0.783490061759949 0.43170490860939
0.78608775138855 0.430525034666061
0.790945649147034 0.428329110145569
0.791604399681091 0.428032010793686
0.804874062538147 0.42209193110466
0.805576264858246 0.42178013920784
0.81831419467926 0.416168540716171
0.881467580795288 0.389508664608002
0.906880497932434 0.379310846328735
0.996952712535858 0.345549434423447
1.03600311279297 0.332005351781845
1.0573570728302 0.324866235256195
1.17803061008453 0.287800818681717
1.22611820697784 0.274488806724548
1.35938203334808 0.241483390331268
1.36120343208313 0.241068601608276
1.40997350215912 0.230302304029465
1.41032409667969 0.230227261781693
1.42464458942413 0.227188259363174
1.43722772598267 0.224562123417854
1.68432581424713 0.18033342063427
1.87879085540771 0.153501272201538
2.93655371665955 0.0740536600351334
2.93978977203369 0.0739123001694679
3.24061918258667 0.0623059906065464
};
\addlegendentry{estimated}
\end{axis}

\end{tikzpicture}
        }
        \end{subfigure}
        \begin{subfigure}[b]{.32\linewidth}
        \resizebox{\linewidth}{0.8\linewidth}{
\begin{tikzpicture}

\definecolor{darkgray176}{RGB}{176,176,176}
\definecolor{lightgray204}{RGB}{204,204,204}
\definecolor{steelblue31119180}{RGB}{31,119,180}

\begin{axis}[
legend cell align={left},
legend style={
  font=\tiny,
  fill opacity=0.8,
  draw opacity=1,
  text opacity=1,
  at={(0.40,0.97)},
  anchor=north west,
  draw=lightgray204
},
scale only axis,
width=4cm,
height=3cm,
tick align=outside,
tick pos=left,
title={\fontsize{6}{6}\selectfont $N=10000$},
tick label style={font=\tiny},
every tick/.style={
black,
semithick,
},
x label style={at={(axis description cs:0.5,-0.1)},anchor=north,font=\tiny},
y label style={at={(axis description cs:-0.06,.5)},rotate=90,anchor=south,font=\tiny},
x grid style={darkgray176},
xlabel={$t$},
xmin=-0.157433523936197, xmax=3.40243121623062,
xtick style={color=black},
y grid style={darkgray176},
ylabel={$S(t|Z)$},
ymin=-0.00870856288820505, ymax=1.04395221937448,
ytick={0, 0.25, 0.75, 1},
ytick style={color=black}
]
\addplot [semithick, steelblue31119180, mark=*, mark size=0.5, mark options={solid}]
table {%
0.00437850970774889 0.995631098747253
0.00542990444228053 0.994584858417511
0.00595268839970231 0.99406498670578
0.00755271734669805 0.992475688457489
0.0123842898756266 0.987692058086395
0.0166294146329165 0.983508110046387
0.0172695983201265 0.982878684997559
0.0208375565707684 0.979378044605255
0.0214008055627346 0.978826522827148
0.0237781554460526 0.976502299308777
0.0328329727053642 0.967700242996216
0.0398957096040249 0.960889637470245
0.044037751853466 0.956917881965637
0.0456374324858189 0.955388247966766
0.0461156778037548 0.954931497573853
0.0586898550391197 0.942999243736267
0.0675230771303177 0.934706151485443
0.0703325197100639 0.93208384513855
0.0790752917528152 0.923970401287079
0.0811049863696098 0.922096908092499
0.0839139670133591 0.919510304927826
0.0910551473498344 0.912967383861542
0.0914278626441956 0.912627160549164
0.0929863452911377 0.911205887794495
0.0971563383936882 0.907414138317108
0.0996829569339752 0.905124306678772
0.0998237505555153 0.904996871948242
0.115672208368778 0.890767157077789
0.127606585621834 0.880199551582336
0.133811295032501 0.874755144119263
0.140404969453812 0.869006276130676
0.141848400235176 0.86775279045105
0.143481746315956 0.866336584091187
0.14919476211071 0.861401319503784
0.165343880653381 0.847602188587189
0.166629523038864 0.846513211727142
0.17111437022686 0.842725217342377
0.175462424755096 0.83906888961792
0.179158613085747 0.835973262786865
0.184549406170845 0.831478834152222
0.204394772648811 0.815140545368195
0.206485122442245 0.813438355922699
0.228909358382225 0.795400619506836
0.229227304458618 0.795147776603699
0.231799572706223 0.793105065822601
0.25389552116394 0.775772869586945
0.295848220586777 0.74390035867691
0.301604896783829 0.739630222320557
0.317365616559982 0.72806453704834
0.321536540985107 0.725034117698669
0.323987185955048 0.723259508609772
0.325569301843643 0.722116112709045
0.327896296977997 0.720437705516815
0.333507657051086 0.716406404972076
0.354743450880051 0.701353311538696
0.368787109851837 0.691572666168213
0.374075591564178 0.687924921512604
0.376269787549973 0.686417102813721
0.392039835453033 0.675677180290222
0.448523104190826 0.638570547103882
0.467623382806778 0.626489400863647
0.51166170835495 0.599498510360718
0.528185546398163 0.589673936367035
0.541024386882782 0.582151591777802
0.568550884723663 0.566345512866974
0.581681251525879 0.558957815170288
0.583209156990051 0.558104455471039
0.639981210231781 0.527302324771881
0.662271738052368 0.515678524971008
0.67755538225174 0.507856965065002
0.712697148323059 0.490319907665253
0.730794370174408 0.481526345014572
0.732405066490173 0.48075133562088
0.762117743492126 0.466677069664001
0.773440182209015 0.461422920227051
0.783490061759949 0.456808924674988
0.78608775138855 0.455623835325241
0.790945649147034 0.453415811061859
0.791604399681091 0.453117221593857
0.804874062538147 0.447144210338593
0.805576264858246 0.446830362081528
0.81831419467926 0.441174775362015
0.881467580795288 0.414174646139145
0.906880497932434 0.403781890869141
0.996952712535858 0.369002193212509
1.03600311279297 0.354870229959488
1.0573570728302 0.347372680902481
1.17803061008453 0.307884484529495
1.22611820697784 0.293429404497147
1.35938203334808 0.256819427013397
1.36120343208313 0.256352096796036
1.40997350215912 0.244149759411812
1.41032409667969 0.244064167141914
1.42464458942413 0.240593954920769
1.43722772598267 0.237585499882698
1.68432581424713 0.185569494962692
1.87879085540771 0.152774721384048
2.93655371665955 0.0530482344329357
2.93978977203369 0.0528768450021744
3.24061918258667 0.0391396544873714
};
\addlegendentry{ground truth}
\addplot [width=1pt, red, opacity=0.7]
table {%
0.00437850970774889 0.995076715946198
0.00542990444228053 0.993899464607239
0.00595268839970231 0.993314921855927
0.00755271734669805 0.991528630256653
0.0123842898756266 0.986161708831787
0.0166294146329165 0.981479465961456
0.0172695983201265 0.980776071548462
0.0208375565707684 0.976868510246277
0.0214008055627346 0.976253688335419
0.0237781554460526 0.973664224147797
0.0328329727053642 0.963889181613922
0.0398957096040249 0.956355929374695
0.044037751853466 0.951973617076874
0.0456374324858189 0.95028817653656
0.0461156778037548 0.949785053730011
0.0586898550391197 0.936680912971497
0.0675230771303177 0.927616834640503
0.0703325197100639 0.924758076667786
0.0790752917528152 0.915934383869171
0.0811049863696098 0.913901329040527
0.0839139670133591 0.91109710931778
0.0910551473498344 0.904015719890594
0.0914278626441956 0.903647899627686
0.0929863452911377 0.902111947536469
0.0971563383936882 0.898017346858978
0.0996829569339752 0.895547330379486
0.0998237505555153 0.895409941673279
0.115672208368778 0.880107223987579
0.127606585621834 0.868799209594727
0.133811295032501 0.862989902496338
0.140404969453812 0.856867074966431
0.141848400235176 0.85553377866745
0.143481746315956 0.854027986526489
0.14919476211071 0.848785996437073
0.165343880653381 0.834176659584045
0.166629523038864 0.833026707172394
0.17111437022686 0.829030215740204
0.175462424755096 0.825177788734436
0.179158613085747 0.821919977664948
0.184549406170845 0.817196130752563
0.204394772648811 0.800092160701752
0.206485122442245 0.798316538333893
0.228909358382225 0.779549121856689
0.229227304458618 0.779286205768585
0.231799572706223 0.77716326713562
0.25389552116394 0.7591592669487
0.295848220586777 0.726113617420197
0.301604896783829 0.721691846847534
0.317365616559982 0.709729254245758
0.321536540985107 0.706596672534943
0.323987185955048 0.70476359128952
0.325569301843643 0.70358270406723
0.327896296977997 0.701849460601807
0.333507657051086 0.69768899679184
0.354743450880051 0.682178676128387
0.368787109851837 0.672122418880463
0.374075591564178 0.6683748960495
0.376269787549973 0.666826367378235
0.392039835453033 0.655802071094513
0.448523104190826 0.61784827709198
0.467623382806778 0.605554044246674
0.51166170835495 0.578223645687103
0.528185546398163 0.568314611911774
0.541024386882782 0.560718894004822
0.568550884723663 0.544732689857483
0.581681251525879 0.537250280380249
0.583209156990051 0.536385715007782
0.639981210231781 0.505207121372223
0.662271738052368 0.49347248673439
0.67755538225174 0.485579609870911
0.712697148323059 0.467846721410751
0.730794370174408 0.458923935890198
0.732405066490173 0.458136320114136
0.762117743492126 0.443766385316849
0.773440182209015 0.438371956348419
0.783490061759949 0.433622390031815
0.78608775138855 0.432400226593018
0.790945649147034 0.430118709802628
0.791604399681091 0.429810047149658
0.804874062538147 0.423628896474838
0.805576264858246 0.423303872346878
0.81831419467926 0.417442858219147
0.881467580795288 0.389206260442734
0.906880497932434 0.378253966569901
0.996952712535858 0.341399341821671
1.03600311279297 0.326436340808868
1.0573570728302 0.318504184484482
1.17803061008453 0.277163565158844
1.22611820697784 0.262294918298721
1.35938203334808 0.225568816065788
1.36120343208313 0.225109964609146
1.40997350215912 0.213213086128235
1.41032409667969 0.21313039958477
1.42464458942413 0.209779977798462
1.43722772598267 0.206886023283005
1.68432581424713 0.158740550279617
1.87879085540771 0.130456298589706
2.93655371665955 0.0537940114736557
2.93978977203369 0.0536710657179356
3.24061918258667 0.0438040047883987
};
\addlegendentry{estimated}
\end{axis}

\end{tikzpicture}
        }
        \end{subfigure}
        \begin{subfigure}[b]{.32\linewidth}
        \resizebox{\linewidth}{0.8\linewidth}{
\begin{tikzpicture}

\definecolor{darkgray176}{RGB}{176,176,176}
\definecolor{lightgray204}{RGB}{204,204,204}
\definecolor{steelblue31119180}{RGB}{31,119,180}

\begin{axis}[
legend cell align={left},
legend style={
  font=\tiny,
  fill opacity=0.8,
  draw opacity=1,
  text opacity=1,
  at={(0.40,0.97)},
  anchor=north west,
  draw=lightgray204
},
scale only axis,
width=4cm,
height=3cm,
tick align=outside,
tick pos=left,
title={\fontsize{6}{6}\selectfont $N=1000$},
tick label style={font=\tiny},
every tick/.style={
black,
semithick,
},
x label style={at={(axis description cs:0.5,-0.1)},anchor=north,font=\tiny},
y label style={at={(axis description cs:-0.06,.5)},rotate=90,anchor=south,font=\tiny},
x grid style={darkgray176},
xlabel={$t$},
xmin=-0.157433523936197, xmax=3.40243121623062,
xtick style={color=black},
y grid style={darkgray176},
ylabel={$S(t|Z)$},
ymin=-0.00870856288820505, ymax=1.04395221937448,
ytick={0, 0.25, 0.75, 1},
ytick style={color=black}
]
\addplot [semithick, steelblue31119180, mark=*, mark size=0.5, mark options={solid}]
table {%
0.00437850970774889 0.995631098747253
0.00542990444228053 0.994584858417511
0.00595268839970231 0.99406498670578
0.00755271734669805 0.992475688457489
0.0123842898756266 0.987692058086395
0.0166294146329165 0.983508110046387
0.0172695983201265 0.982878684997559
0.0208375565707684 0.979378044605255
0.0214008055627346 0.978826522827148
0.0237781554460526 0.976502299308777
0.0328329727053642 0.967700242996216
0.0398957096040249 0.960889637470245
0.044037751853466 0.956917881965637
0.0456374324858189 0.955388247966766
0.0461156778037548 0.954931497573853
0.0586898550391197 0.942999243736267
0.0675230771303177 0.934706151485443
0.0703325197100639 0.93208384513855
0.0790752917528152 0.923970401287079
0.0811049863696098 0.922096908092499
0.0839139670133591 0.919510304927826
0.0910551473498344 0.912967383861542
0.0914278626441956 0.912627160549164
0.0929863452911377 0.911205887794495
0.0971563383936882 0.907414138317108
0.0996829569339752 0.905124306678772
0.0998237505555153 0.904996871948242
0.115672208368778 0.890767157077789
0.127606585621834 0.880199551582336
0.133811295032501 0.874755144119263
0.140404969453812 0.869006276130676
0.141848400235176 0.86775279045105
0.143481746315956 0.866336584091187
0.14919476211071 0.861401319503784
0.165343880653381 0.847602188587189
0.166629523038864 0.846513211727142
0.17111437022686 0.842725217342377
0.175462424755096 0.83906888961792
0.179158613085747 0.835973262786865
0.184549406170845 0.831478834152222
0.204394772648811 0.815140545368195
0.206485122442245 0.813438355922699
0.228909358382225 0.795400619506836
0.229227304458618 0.795147776603699
0.231799572706223 0.793105065822601
0.25389552116394 0.775772869586945
0.295848220586777 0.74390035867691
0.301604896783829 0.739630222320557
0.317365616559982 0.72806453704834
0.321536540985107 0.725034117698669
0.323987185955048 0.723259508609772
0.325569301843643 0.722116112709045
0.327896296977997 0.720437705516815
0.333507657051086 0.716406404972076
0.354743450880051 0.701353311538696
0.368787109851837 0.691572666168213
0.374075591564178 0.687924921512604
0.376269787549973 0.686417102813721
0.392039835453033 0.675677180290222
0.448523104190826 0.638570547103882
0.467623382806778 0.626489400863647
0.51166170835495 0.599498510360718
0.528185546398163 0.589673936367035
0.541024386882782 0.582151591777802
0.568550884723663 0.566345512866974
0.581681251525879 0.558957815170288
0.583209156990051 0.558104455471039
0.639981210231781 0.527302324771881
0.662271738052368 0.515678524971008
0.67755538225174 0.507856965065002
0.712697148323059 0.490319907665253
0.730794370174408 0.481526345014572
0.732405066490173 0.48075133562088
0.762117743492126 0.466677069664001
0.773440182209015 0.461422920227051
0.783490061759949 0.456808924674988
0.78608775138855 0.455623835325241
0.790945649147034 0.453415811061859
0.791604399681091 0.453117221593857
0.804874062538147 0.447144210338593
0.805576264858246 0.446830362081528
0.81831419467926 0.441174775362015
0.881467580795288 0.414174646139145
0.906880497932434 0.403781890869141
0.996952712535858 0.369002193212509
1.03600311279297 0.354870229959488
1.0573570728302 0.347372680902481
1.17803061008453 0.307884484529495
1.22611820697784 0.293429404497147
1.35938203334808 0.256819427013397
1.36120343208313 0.256352096796036
1.40997350215912 0.244149759411812
1.41032409667969 0.244064167141914
1.42464458942413 0.240593954920769
1.43722772598267 0.237585499882698
1.68432581424713 0.185569494962692
1.87879085540771 0.152774721384048
2.93655371665955 0.0530482344329357
2.93978977203369 0.0528768450021744
3.24061918258667 0.0391396544873714
};
\addlegendentry{ground truth}
\addplot [width=1pt, red, opacity=0.7]
table {%
0.00437850970774889 0.994512021541595
0.00542990444228053 0.99319976568222
0.00595268839970231 0.992547929286957
0.00755271734669805 0.990556478500366
0.0123842898756266 0.984572410583496
0.0166294146329165 0.979350864887238
0.0172695983201265 0.978566288948059
0.0208375565707684 0.974207580089569
0.0214008055627346 0.973521530628204
0.0237781554460526 0.970632493495941
0.0328329727053642 0.959721088409424
0.0398957096040249 0.951310694217682
0.044037751853466 0.946418762207031
0.0456374324858189 0.944537460803986
0.0461156778037548 0.943975925445557
0.0586898550391197 0.929350733757019
0.0675230771303177 0.919236123561859
0.0703325197100639 0.916046321392059
0.0790752917528152 0.906203091144562
0.0811049863696098 0.903935611248016
0.0839139670133591 0.90080863237381
0.0910551473498344 0.892915964126587
0.0914278626441956 0.892506301403046
0.0929863452911377 0.890795528888702
0.0971563383936882 0.886236965656281
0.0996829569339752 0.883488297462463
0.0998237505555153 0.88333535194397
0.115672208368778 0.866349101066589
0.127606585621834 0.853843510150909
0.133811295032501 0.847436904907227
0.140404969453812 0.840695917606354
0.141848400235176 0.839230597019196
0.143481746315956 0.837575376033783
0.14919476211071 0.831817865371704
0.165343880653381 0.815816700458527
0.166629523038864 0.814559042453766
0.17111437022686 0.810189843177795
0.175462424755096 0.805985510349274
0.179158613085747 0.802433431148529
0.184549406170845 0.797285676002502
0.204394772648811 0.778684139251709
0.206485122442245 0.776757836341858
0.228909358382225 0.756449639797211
0.229227304458618 0.756166398525238
0.231799572706223 0.753879010677338
0.25389552116394 0.7345831990242
0.295848220586777 0.699539005756378
0.301604896783829 0.694893658161163
0.317365616559982 0.68235981464386
0.321536540985107 0.679089546203613
0.323987185955048 0.677176654338837
0.325569301843643 0.675945103168488
0.327896296977997 0.674137353897095
0.333507657051086 0.669801414012909
0.354743450880051 0.653687238693237
0.368787109851837 0.643277764320374
0.374075591564178 0.63940954208374
0.376269787549973 0.637812316417694
0.392039835453033 0.626461505889893
0.448523104190826 0.587673664093018
0.467623382806778 0.575173020362854
0.51166170835495 0.54749196767807
0.528185546398163 0.537496089935303
0.541024386882782 0.529877007007599
0.568550884723663 0.51394134759903
0.581681251525879 0.506524205207825
0.583209156990051 0.505669176578522
0.639981210231781 0.474951058626175
0.662271738052368 0.463401824235916
0.67755538225174 0.455637097358704
0.712697148323059 0.438295841217041
0.730794370174408 0.429626822471619
0.732405066490173 0.42886421084404
0.762117743492126 0.415042519569397
0.773440182209015 0.40988427400589
0.783490061759949 0.405362904071808
0.78608775138855 0.404202669858932
0.790945649147034 0.402042031288147
0.791604399681091 0.401749968528748
0.804874062538147 0.395911157131195
0.805576264858246 0.39560455083847
0.81831419467926 0.390083491802216
0.881467580795288 0.363785445690155
0.906880497932434 0.353721708059311
0.996952712535858 0.32012528181076
1.03600311279297 0.306571573019028
1.0573570728302 0.299386084079742
1.17803061008453 0.261900842189789
1.22611820697784 0.248321354389191
1.35938203334808 0.214512974023819
1.36120343208313 0.214087769389153
1.40997350215912 0.203033953905106
1.41032409667969 0.202956765890121
1.42464458942413 0.199824526906013
1.43722772598267 0.19711010158062
1.68432581424713 0.151219561696053
1.87879085540771 0.123288325965405
2.93655371665955 0.0437710545957088
2.93978977203369 0.0436383895576
3.24061918258667 0.0329694077372551
};
\addlegendentry{estimated}
\end{axis}

\end{tikzpicture}
        }
        \end{subfigure}
        \begin{subfigure}[b]{.32\linewidth}
        \resizebox{\linewidth}{0.8\linewidth}{
\begin{tikzpicture}

\definecolor{darkgray176}{RGB}{176,176,176}
\definecolor{lightgray204}{RGB}{204,204,204}
\definecolor{steelblue31119180}{RGB}{31,119,180}

\begin{axis}[
legend cell align={left},
legend style={
  font=\tiny,
  fill opacity=0.8,
  draw opacity=1,
  text opacity=1,
  at={(0.40,0.97)},
  anchor=north west,
  draw=lightgray204
},
scale only axis,
width=4cm,
height=3cm,
tick align=outside,
tick pos=left,
title={\fontsize{6}{6}\selectfont $N=5000$},
tick label style={font=\tiny},
every tick/.style={
black,
semithick,
},
x label style={at={(axis description cs:0.5,-0.1)},anchor=north,font=\tiny},
y label style={at={(axis description cs:-0.06,.5)},rotate=90,anchor=south,font=\tiny},
x grid style={darkgray176},
xlabel={$t$},
xmin=-0.157433523936197, xmax=3.40243121623062,
xtick style={color=black},
y grid style={darkgray176},
ylabel={$S(t|Z)$},
ymin=-0.00870856288820505, ymax=1.04395221937448,
ytick={0, 0.25, 0.75, 1},
ytick style={color=black}
]
\addplot [semithick, steelblue31119180, mark=*, mark size=0.5, mark options={solid}]
table {%
0.00437850970774889 0.995631098747253
0.00542990444228053 0.994584858417511
0.00595268839970231 0.99406498670578
0.00755271734669805 0.992475688457489
0.0123842898756266 0.987692058086395
0.0166294146329165 0.983508110046387
0.0172695983201265 0.982878684997559
0.0208375565707684 0.979378044605255
0.0214008055627346 0.978826522827148
0.0237781554460526 0.976502299308777
0.0328329727053642 0.967700242996216
0.0398957096040249 0.960889637470245
0.044037751853466 0.956917881965637
0.0456374324858189 0.955388247966766
0.0461156778037548 0.954931497573853
0.0586898550391197 0.942999243736267
0.0675230771303177 0.934706151485443
0.0703325197100639 0.93208384513855
0.0790752917528152 0.923970401287079
0.0811049863696098 0.922096908092499
0.0839139670133591 0.919510304927826
0.0910551473498344 0.912967383861542
0.0914278626441956 0.912627160549164
0.0929863452911377 0.911205887794495
0.0971563383936882 0.907414138317108
0.0996829569339752 0.905124306678772
0.0998237505555153 0.904996871948242
0.115672208368778 0.890767157077789
0.127606585621834 0.880199551582336
0.133811295032501 0.874755144119263
0.140404969453812 0.869006276130676
0.141848400235176 0.86775279045105
0.143481746315956 0.866336584091187
0.14919476211071 0.861401319503784
0.165343880653381 0.847602188587189
0.166629523038864 0.846513211727142
0.17111437022686 0.842725217342377
0.175462424755096 0.83906888961792
0.179158613085747 0.835973262786865
0.184549406170845 0.831478834152222
0.204394772648811 0.815140545368195
0.206485122442245 0.813438355922699
0.228909358382225 0.795400619506836
0.229227304458618 0.795147776603699
0.231799572706223 0.793105065822601
0.25389552116394 0.775772869586945
0.295848220586777 0.74390035867691
0.301604896783829 0.739630222320557
0.317365616559982 0.72806453704834
0.321536540985107 0.725034117698669
0.323987185955048 0.723259508609772
0.325569301843643 0.722116112709045
0.327896296977997 0.720437705516815
0.333507657051086 0.716406404972076
0.354743450880051 0.701353311538696
0.368787109851837 0.691572666168213
0.374075591564178 0.687924921512604
0.376269787549973 0.686417102813721
0.392039835453033 0.675677180290222
0.448523104190826 0.638570547103882
0.467623382806778 0.626489400863647
0.51166170835495 0.599498510360718
0.528185546398163 0.589673936367035
0.541024386882782 0.582151591777802
0.568550884723663 0.566345512866974
0.581681251525879 0.558957815170288
0.583209156990051 0.558104455471039
0.639981210231781 0.527302324771881
0.662271738052368 0.515678524971008
0.67755538225174 0.507856965065002
0.712697148323059 0.490319907665253
0.730794370174408 0.481526345014572
0.732405066490173 0.48075133562088
0.762117743492126 0.466677069664001
0.773440182209015 0.461422920227051
0.783490061759949 0.456808924674988
0.78608775138855 0.455623835325241
0.790945649147034 0.453415811061859
0.791604399681091 0.453117221593857
0.804874062538147 0.447144210338593
0.805576264858246 0.446830362081528
0.81831419467926 0.441174775362015
0.881467580795288 0.414174646139145
0.906880497932434 0.403781890869141
0.996952712535858 0.369002193212509
1.03600311279297 0.354870229959488
1.0573570728302 0.347372680902481
1.17803061008453 0.307884484529495
1.22611820697784 0.293429404497147
1.35938203334808 0.256819427013397
1.36120343208313 0.256352096796036
1.40997350215912 0.244149759411812
1.41032409667969 0.244064167141914
1.42464458942413 0.240593954920769
1.43722772598267 0.237585499882698
1.68432581424713 0.185569494962692
1.87879085540771 0.152774721384048
2.93655371665955 0.0530482344329357
2.93978977203369 0.0528768450021744
3.24061918258667 0.0391396544873714
};
\addlegendentry{ground truth}
\addplot [width=1pt, red, opacity=0.7]
table {%
0.00437850970774889 0.994074881076813
0.00542990444228053 0.992660045623779
0.00595268839970231 0.991957664489746
0.00755271734669805 0.989812850952148
0.0123842898756266 0.983378827571869
0.0166294146329165 0.977778315544128
0.0172695983201265 0.97693794965744
0.0208375565707684 0.972274422645569
0.0214008055627346 0.971541345119476
0.0237781554460526 0.968456327915192
0.0328329727053642 0.956841707229614
0.0398957096040249 0.947928667068481
0.044037751853466 0.942760050296783
0.0456374324858189 0.940775275230408
0.0461156778037548 0.940183162689209
0.0586898550391197 0.924815237522125
0.0675230771303177 0.914245665073395
0.0703325197100639 0.910922110080719
0.0790752917528152 0.900691986083984
0.0811049863696098 0.898341000080109
0.0839139670133591 0.895102202892303
0.0910551473498344 0.886944890022278
0.0914278626441956 0.886522233486176
0.0929863452911377 0.884757816791534
0.0971563383936882 0.880061328411102
0.0996829569339752 0.877233564853668
0.0998237505555153 0.877076387405396
0.115672208368778 0.859642148017883
0.127606585621834 0.846842467784882
0.133811295032501 0.840297162532806
0.140404969453812 0.833420515060425
0.141848400235176 0.831925809383392
0.143481746315956 0.830239236354828
0.14919476211071 0.824378848075867
0.165343880653381 0.808133006095886
0.166629523038864 0.806859374046326
0.17111437022686 0.802439510822296
0.175462424755096 0.798188090324402
0.179158613085747 0.794599711894989
0.184549406170845 0.789408266544342
0.204394772648811 0.770711779594421
0.206485122442245 0.768779575824738
0.228909358382225 0.748477041721344
0.229227304458618 0.748194456100464
0.231799572706223 0.745913684368134
0.25389552116394 0.726711273193359
0.295848220586777 0.692038953304291
0.301604896783829 0.687451779842377
0.317365616559982 0.67510050535202
0.321536540985107 0.671882092952728
0.323987185955048 0.670000910758972
0.325569301843643 0.668790221214294
0.327896296977997 0.667014956474304
0.333507657051086 0.662756145000458
0.354743450880051 0.646953165531158
0.368787109851837 0.636779129505157
0.374075591564178 0.633001446723938
0.376269787549973 0.631440877914429
0.392039835453033 0.62036794424057
0.448523104190826 0.582635045051575
0.467623382806778 0.57052206993103
0.51166170835495 0.543747961521149
0.528185546398163 0.534088432788849
0.541024386882782 0.526728749275208
0.568550884723663 0.51135641336441
0.581681251525879 0.504201114177704
0.583209156990051 0.50337541103363
0.639981210231781 0.473784744739532
0.662271738052368 0.462662309408188
0.67755538225174 0.455203086137772
0.712697148323059 0.438543111085892
0.730794370174408 0.430226057767868
0.732405066490173 0.429493606090546
0.762117743492126 0.416213244199753
0.773440182209015 0.411265432834625
0.783490061759949 0.406927317380905
0.78608775138855 0.405812501907349
0.790945649147034 0.403735816478729
0.791604399681091 0.403454929590225
0.804874062538147 0.397845774888992
0.805576264858246 0.397551447153091
0.81831419467926 0.392252951860428
0.881467580795288 0.367063194513321
0.906880497932434 0.357408493757248
0.996952712535858 0.325240522623062
1.03600311279297 0.312225133180618
1.0573570728302 0.305340170860291
1.17803061008453 0.269067138433456
1.22611820697784 0.255856066942215
1.35938203334808 0.222493812441826
1.36120343208313 0.222069203853607
1.40997350215912 0.211005792021751
1.41032409667969 0.210928112268448
1.42464458942413 0.207780584692955
1.43722772598267 0.205056190490723
1.68432581424713 0.158287703990936
1.87879085540771 0.129083678126335
2.93655371665955 0.0417474396526814
2.93978977203369 0.0416013486683369
3.24061918258667 0.0299900081008673
};
\addlegendentry{estimated}
\end{axis}

\end{tikzpicture}
        }
        \end{subfigure}
        \begin{subfigure}[b]{.32\linewidth}
        \resizebox{\linewidth}{0.8\linewidth}{
\begin{tikzpicture}

\definecolor{darkgray176}{RGB}{176,176,176}
\definecolor{lightgray204}{RGB}{204,204,204}
\definecolor{steelblue31119180}{RGB}{31,119,180}

\begin{axis}[
legend cell align={left},
legend style={
  font=\tiny,
  fill opacity=0.8,
  draw opacity=1,
  text opacity=1,
  at={(0.40,0.97)},
  anchor=north west,
  draw=lightgray204
},
scale only axis,
width=4cm,
height=3cm,
tick align=outside,
tick pos=left,
title={\fontsize{6}{6}\selectfont $N=10000$},
tick label style={font=\tiny},
every tick/.style={
black,
semithick,
},
x label style={at={(axis description cs:0.5,-0.1)},anchor=north,font=\tiny},
y label style={at={(axis description cs:-0.06,.5)},rotate=90,anchor=south,font=\tiny},
x grid style={darkgray176},
xlabel={$t$},
xmin=-0.157433523936197, xmax=3.40243121623062,
xtick style={color=black},
y grid style={darkgray176},
ylabel={$S(t|Z)$},
ymin=-0.00870856288820505, ymax=1.04395221937448,
ytick={0, 0.25, 0.75, 1},
ytick style={color=black}
]
\addplot [semithick, steelblue31119180, mark=*, mark size=0.5, mark options={solid}]
table {%
0.00437850970774889 0.995631098747253
0.00542990444228053 0.994584858417511
0.00595268839970231 0.99406498670578
0.00755271734669805 0.992475688457489
0.0123842898756266 0.987692058086395
0.0166294146329165 0.983508110046387
0.0172695983201265 0.982878684997559
0.0208375565707684 0.979378044605255
0.0214008055627346 0.978826522827148
0.0237781554460526 0.976502299308777
0.0328329727053642 0.967700242996216
0.0398957096040249 0.960889637470245
0.044037751853466 0.956917881965637
0.0456374324858189 0.955388247966766
0.0461156778037548 0.954931497573853
0.0586898550391197 0.942999243736267
0.0675230771303177 0.934706151485443
0.0703325197100639 0.93208384513855
0.0790752917528152 0.923970401287079
0.0811049863696098 0.922096908092499
0.0839139670133591 0.919510304927826
0.0910551473498344 0.912967383861542
0.0914278626441956 0.912627160549164
0.0929863452911377 0.911205887794495
0.0971563383936882 0.907414138317108
0.0996829569339752 0.905124306678772
0.0998237505555153 0.904996871948242
0.115672208368778 0.890767157077789
0.127606585621834 0.880199551582336
0.133811295032501 0.874755144119263
0.140404969453812 0.869006276130676
0.141848400235176 0.86775279045105
0.143481746315956 0.866336584091187
0.14919476211071 0.861401319503784
0.165343880653381 0.847602188587189
0.166629523038864 0.846513211727142
0.17111437022686 0.842725217342377
0.175462424755096 0.83906888961792
0.179158613085747 0.835973262786865
0.184549406170845 0.831478834152222
0.204394772648811 0.815140545368195
0.206485122442245 0.813438355922699
0.228909358382225 0.795400619506836
0.229227304458618 0.795147776603699
0.231799572706223 0.793105065822601
0.25389552116394 0.775772869586945
0.295848220586777 0.74390035867691
0.301604896783829 0.739630222320557
0.317365616559982 0.72806453704834
0.321536540985107 0.725034117698669
0.323987185955048 0.723259508609772
0.325569301843643 0.722116112709045
0.327896296977997 0.720437705516815
0.333507657051086 0.716406404972076
0.354743450880051 0.701353311538696
0.368787109851837 0.691572666168213
0.374075591564178 0.687924921512604
0.376269787549973 0.686417102813721
0.392039835453033 0.675677180290222
0.448523104190826 0.638570547103882
0.467623382806778 0.626489400863647
0.51166170835495 0.599498510360718
0.528185546398163 0.589673936367035
0.541024386882782 0.582151591777802
0.568550884723663 0.566345512866974
0.581681251525879 0.558957815170288
0.583209156990051 0.558104455471039
0.639981210231781 0.527302324771881
0.662271738052368 0.515678524971008
0.67755538225174 0.507856965065002
0.712697148323059 0.490319907665253
0.730794370174408 0.481526345014572
0.732405066490173 0.48075133562088
0.762117743492126 0.466677069664001
0.773440182209015 0.461422920227051
0.783490061759949 0.456808924674988
0.78608775138855 0.455623835325241
0.790945649147034 0.453415811061859
0.791604399681091 0.453117221593857
0.804874062538147 0.447144210338593
0.805576264858246 0.446830362081528
0.81831419467926 0.441174775362015
0.881467580795288 0.414174646139145
0.906880497932434 0.403781890869141
0.996952712535858 0.369002193212509
1.03600311279297 0.354870229959488
1.0573570728302 0.347372680902481
1.17803061008453 0.307884484529495
1.22611820697784 0.293429404497147
1.35938203334808 0.256819427013397
1.36120343208313 0.256352096796036
1.40997350215912 0.244149759411812
1.41032409667969 0.244064167141914
1.42464458942413 0.240593954920769
1.43722772598267 0.237585499882698
1.68432581424713 0.185569494962692
1.87879085540771 0.152774721384048
2.93655371665955 0.0530482344329357
2.93978977203369 0.0528768450021744
3.24061918258667 0.0391396544873714
};
\addlegendentry{ground truth}
\addplot [width=1pt, red, opacity=0.7]
table {%
0.00437850970774889 0.994504988193512
0.00542990444228053 0.993192315101624
0.00595268839970231 0.99254047870636
0.00755271734669805 0.990549802780151
0.0123842898756266 0.984574735164642
0.0166294146329165 0.979369640350342
0.0172695983201265 0.978588283061981
0.0208375565707684 0.974250614643097
0.0214008055627346 0.973568499088287
0.0237781554460526 0.970697343349457
0.0328329727053642 0.959877550601959
0.0398957096040249 0.95156317949295
0.044037751853466 0.946737349033356
0.0456374324858189 0.944883406162262
0.0461156778037548 0.944330155849457
0.0586898550391197 0.929957151412964
0.0675230771303177 0.920054852962494
0.0703325197100639 0.916938424110413
0.0790752917528152 0.907340407371521
0.0811049863696098 0.905133545398712
0.0839139670133591 0.902092576026917
0.0910551473498344 0.894429683685303
0.0914278626441956 0.894032418727875
0.0929863452911377 0.892374038696289
0.0971563383936882 0.88795930147171
0.0996829569339752 0.885300099849701
0.0998237505555153 0.885152220726013
0.115672208368778 0.868742644786835
0.127606585621834 0.8566814661026
0.133811295032501 0.850506067276001
0.140404969453812 0.8440101146698
0.141848400235176 0.842597186565399
0.143481746315956 0.841001749038696
0.14919476211071 0.835455417633057
0.165343880653381 0.820041477680206
0.166629523038864 0.818830847740173
0.17111437022686 0.814627766609192
0.175462424755096 0.810581922531128
0.179158613085747 0.807161569595337
0.184549406170845 0.802208662033081
0.204394772648811 0.784327805042267
0.206485122442245 0.782475411891937
0.228909358382225 0.762971162796021
0.229227304458618 0.7626993060112
0.231799572706223 0.760504126548767
0.25389552116394 0.741998732089996
0.295848220586777 0.70846962928772
0.301604896783829 0.704026281833649
0.317365616559982 0.692033648490906
0.321536540985107 0.688905537128448
0.323987185955048 0.68707549571991
0.325569301843643 0.685896754264832
0.327896296977997 0.684167683124542
0.333507657051086 0.680022537708282
0.354743450880051 0.66463565826416
0.368787109851837 0.654708981513977
0.374075591564178 0.651019871234894
0.376269787549973 0.649496972560883
0.392039835453033 0.638674676418304
0.448523104190826 0.601759552955627
0.467623382806778 0.589893698692322
0.51166170835495 0.563591837882996
0.528185546398163 0.554091155529022
0.541024386882782 0.546838581562042
0.568550884723663 0.531643509864807
0.581681251525879 0.524574935436249
0.583209156990051 0.523758769035339
0.639981210231781 0.494429796934128
0.662271738052368 0.483410656452179
0.67755538225174 0.476006060838699
0.712697148323059 0.459415197372437
0.730794370174408 0.451112240552902
0.732405066490173 0.450381219387054
0.762117743492126 0.437132000923157
0.773440182209015 0.432177990674973
0.783490061759949 0.42783185839653
0.78608775138855 0.426716029644012
0.790945649147034 0.424639016389847
0.791604399681091 0.424358308315277
0.804874062538147 0.41874486207962
0.805576264858246 0.418449997901917
0.81831419467926 0.413141280412674
0.881467580795288 0.387777000665665
0.906880497932434 0.378011882305145
0.996952712535858 0.345370590686798
1.03600311279297 0.332095980644226
1.0573570728302 0.325040221214294
1.17803061008453 0.287877351045609
1.22611820697784 0.274237453937531
1.35938203334808 0.239694118499756
1.36120343208313 0.239253655076027
1.40997350215912 0.227749735116959
1.41032409667969 0.227668732404709
1.42464458942413 0.224388346076012
1.43722772598267 0.221543356776237
1.68432581424713 0.172323629260063
1.87879085540771 0.141303211450577
2.93655371665955 0.0472247749567032
2.93978977203369 0.0470651276409626
3.24061918258667 0.0343340300023556
};
\addlegendentry{estimated}
\end{axis}

\end{tikzpicture}
        }
        \end{subfigure}
        \caption{Visualizations of synthetic data results under the NFM framework. The plots in the first row compare the empirical estimates of the survival function $S(t|\bar{Z})$ against its true value with $\bar{Z}$ being the average of the features of the $100$ hold-out points, under the PF scheme. The plots in the second row are obtained using the FN scheme, with analogous semantics to the first row.}
        \label{fig: surv_recovery}
    \end{figure}
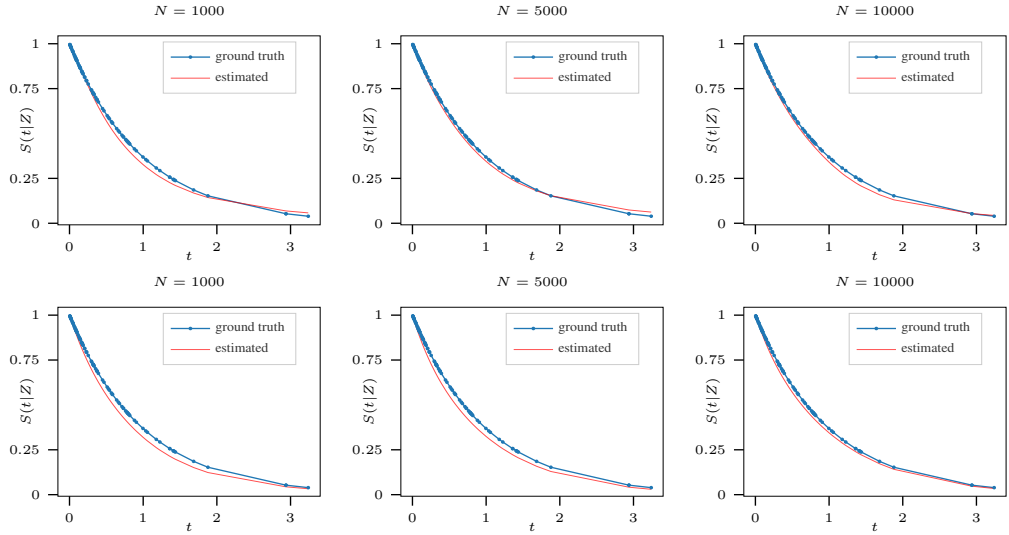

    \subsection{Numerical results of the synthetic experiments}\label{sec: synthetic_numericals}
    Following \cite{zhong2021deep}, we report the relative integrated mean squared error (RISE) of the estimated survival function against the ground truth and list the results in table . The reults suggest that the goodness of fit becomes better with a larger sample size. Moreover, since the true model in the simulation is generated as an PF model, we found PF to perform slightly better than FN, which is reasonable since the inductive bias of PF is more correct in this setup.
    \begin{table}[]
        \centering
        \caption{RISE of the estimated survival function in synthetic experiments}
        \begin{tabular}{l c c c}
     \toprule
     & $N=1000$ & $N=5000$ & $N=10000$ \\
     \midrule
     NFM-PF	& $0.0473$ & $0.0145$ & $0.0137$ \\
     NFM-FN	& $0.0430$ & $0.0184$ & $0.0165$ \\
     \bottomrule
\end{tabular}
        \label{tab: synthetic}
    \end{table}
    
    \subsection{Performance evaluations under the concordance index (C-index)}\label{sec: c_index}
    The concordance index (C-index) \cite{antolini2005time} is yet another evaluation metric that is commonly used in survival analysis. The C-index estimates the probability that, for a random pair of individuals, the predicted survival times of the two individuals have the same ordering as their true survival times. Formally, C-index is defined as
    \begin{align}\label{eqn: c_index}
        \mathcal{C} = \mathbb{P}\left[\widehat{S}(T_i \mid Z_i) < \widehat{S}(T_j \mid Z_j) \mid T_i < T_j, \delta_i = 1\right].
    \end{align}
    We report performance evaluations based on C-index over all the $6$ benchmark datasets in table \ref{tab: cindex}. 
    \begin{table}[]
        \centering
        \caption{Survival prediction results measured in C-index (\%) on all the $6$ benchmark datasets. In each column, the \textbf{boldfaced} score denotes the best result and the \underline{underlined} score represents the second-best result. The average rank of each model is reported in the rightmost column. We did not manage to obtain reasonable results for DeepEH and SODEN on two larger datasets MIMIC-III and KKBOX, and we set corresponding ranks to be the worst on those datasets.}
        \resizebox{\textwidth}{!}{%
            \begin{tabular}{l c c c c c c >{\columncolor[gray]{0.8}}c}
    \toprule
    Model & METABRIC & RotGBSG & FLCHAIN & SUPPORT & MIMIC-III & KKBOX & Ave. Rank \\
    \midrule
    \ph & \result{63.42}{1.81} & \result{66.14}{1.46} & \result{79.09}{1.11} & \result{56.89}{0.91} & \result{74.91}{0.00} & \result{83.01}{0.00} & $11.33$ \\
    GBM & \result{64.02}{1.79} & \result{67.35}{1.16} & \resultf{79.47}{1.08} & \result{61.46}{0.80} & \result{75.20}{0.00} & \result{85.84}{0.00} & $7.17$\\
    RSF & \result{64.47}{1.82} & \result{67.33}{1.34} & \result{78.75}{1.07} & \result{61.63}{0.84} & \result{75.47}{0.17} & \result{85.79}{0.00} & $8.00$\\
    DeepSurv & \result{63.95}{2.12} & \result{67.20}{1.22} & \result{79.04}{1.14} & \result{60.91}{0.85} & \result{80.08}{0.44} & \result{85.59}{0.08} & $8.50$\\
    CoxTime & \result{66.22}{1.69} & \result{67.41}{1.35} & \result{78.95}{1.01} & \result{61.54}{0.87} & \result{78.78}{0.62} & \resultf{87.31}{0.24} & $5.00$\\
    DeepHit & \result{66.33}{1.61} & \result{66.38}{1.07} & \result{78.48}{1.09} & \resultf{63.20}{0.85} & \result{79.16}{0.59} & \result{86.12}{0.26} & $6.50$\\
    DeepEH & \result{66.59}{2.00} & \resultf{67.93}{1.28} & \result{78.71}{1.44} & \result{61.51}{1.04} & $-$ & $-$ & $6.33$ \\
    SuMo-net & \result{64.82}{1.80} & \result{67.20}{1.31} & \result{79.28}{1.02} & \result{62.18}{0.78} & \result{76.23}{1.06} & \result{84.77}{0.02} & $7.00$\\
    SODEN & \result{64.82}{1.05} & \result{66.97}{0.75} & \result{79.00}{0.96} & \result{61.10}{0.59} & $-$ & $-$ &
    $10.17$\\
    SurvNode & \result{64.64}{4.91} & \result{67.30}{1.65} & \result{76.11}{0.98} & \result{55.37}{0.77} & $-$ & $-$ &
    $11.50$\\
    DCM & \result{65.76}{1.25} & \result{66.75}{1.35} & \result{78.61}{0.79} & \result{62.19}{0.95} & \result{76.45}{0.34} & \result{83.48}{0.07} &
    $8.33$\\
    DeSurv & \result{65.88}{2.02} & \result{67.30}{1.45} & \result{78.97}{1.64} & \result{61.47}{0.97} & \resultf{80.97}{0.30} & \result{86.11}{0.05} &$5.67$\\
    \midrule
    \textbf{NFM-PF} & \result{64.98}{1.87} & \results{67.77}{1.35} & \results{79.45}{1.03} & \result{61.33}{0.83} & \result{79.56}{0.15} & \result{86.23}{0.01} & $\underline{4.67}$\\
    \textbf{NFM-FN} & \resultf{66.63}{1.82} & \result{67.73}{1.29} & \result{79.29}{0.93} & \results{62.21}{0.41} & \results{80.18}{0.20} & \results{86.61}{0.01} & $\mathbf{2.16}$ \\
    \bottomrule
\end{tabular}
        }
        \label{tab: cindex}
    \end{table}
    From table \ref{tab: cindex}, it appears that there's no clear winner regarding the C-index metric across the $6$ selected datasets. We conjecture this phenomenon to be closely related to the loose correlation between the C-index and the likelihood-based learning objective, as was observed in \cite{rindt2022a}. Therefore we compute the average rank of each model as an overall assessment of performance, as illustrated in the last column in table \ref{tab: cindex}. The results suggest that the two NFM models perform better on average.
    \subsection{Benefits of frailty}\label{sec: benefits_frailty}
    We compute the (relative) performance gain of NFM-PF and NFM-FN, against their non-frailty counterparts, namely DeepSurv \cite{katzman2018deepsurv} and SuMo-net \cite{rindt2022a} based on results in table \ref{tab: survival_results}, table \ref{tab: kkbox} and table \ref{tab: cindex}. The results are shown in table \ref{tab: frailty_benefits}
    \begin{table}[]
        \centering
        \caption{Relative improvement of NFM models in comparison to their non-frailty counterparts, measured in IBS, INBLL, and C-index.}
        \begin{tabular}{l c c c c c c}
            \toprule
            Dataset           & \multicolumn{3}{c}{NFM-PF vs DeepSurv} & \multicolumn{3}{c}{NFM-FN vs SuMo-net}\\
                            & IBS   & INBLL  & C-index  & IBS  & INBLL  & C-index \\
            \midrule
            METABRIC            & $+1.33\%$ & $+1.56\%$ & $+1.61\%$ & $+2.30\%$ & $+3.08\%$  & $+2.79\%$ \\
            RotGBSG             & $+1.11\%$ & $+0.95\%$ & $+0.84\%$ & $+0.62\%$ & $+0.40\%$ & $+0.79\%$ \\
            FLCHAIN             & $+1.29\%$ & $+1.32\%$ & $+0.52\%$ & $+0.20\%$ & $+0.27\%$ & $+0.01\%$ \\
            SUPPORT             & $+0.31\%$ & $+0.23\%$ & $+0.69\%$ & $+2.22\%$ & $+1.76\%$ & $+0.05\%$ \\
            MIMIC-III           & $+12.38\%$ & $+12.15\%$ & \cellcolor{gray!25}$-0.64\%$ & $+6.18\%$ & $+5.56\%$ & $+5.18\%$ \\
            KKBOX               & $+2.56\%$ & $+0.51\%$ & $+0.75\%$ & $+8.20\%$ & $+10.38\%$ & $+2.17\%$ \\
            \bottomrule
        \end{tabular}
        \label{tab: frailty_benefits}
    \end{table}
    The results suggest solid improvement in incorporating frailty, especially for IBS and INBLL metrics, as the relative increase in performance could be over $10\%$ for both NFM models. For the IBS and INBLL metrics, the performance improvement is consistent across all datasets. The only performance degradation appears on the MIMIC-III dataset evaluated under C-index. This phenomenon is also understandable: Since the DeepSurv model utilized a variant of partial likelihood (PL) for model training, as previous works \cite{NIPS2007_33e8075e} pointed out that PL type objective is closely related to the ranking problem. As C-index could be considered a certain type of ranking measure, it is possible that DeepSurv obtains better ranking performance than NFM-type models which are trained using scale-sensitive likelihood objective. 
    \subsection{Limitations}\label{sec: limitations}
    In this section we discuss the limitations of this paper form both theoretical and empirical standpoints.
    \begin{description}
        \item[Theoretical limitations] As we have established formal statistical guarantees regarding the estimation properties of NFM, the guanrantees do not necessarily imply that NFM perform well on prediction tasks under metrics such as IBS and INBLL. Following the spirit of classical learning theory \cite{shalev2014understanding}, for prediction guarantees it is ideal to directly optimize the underlying metric or its surrogates, which is difficult in survival problems since the metrics involve both a model over the event time and a working model on the censoring time. So far as we have noticed, the only effort that aims to address this issue is the method of inversely-weighted survival games \cite{han2021inverse}. However, the authors in \cite{han2021inverse} did not provide rigorous learning-theoretic statements, which is a promising research direction for future works. 
        \item[Empirical limitations] While NFM is shown to be a competitive survival model for prognosis empirically, we have observed from the empirical results that we can hardly obtain \emph{statistically significant improvements} over the baselines, which is a common problem exhibited in previous works on neural survival regressions \cite{zhong2021deep, rindt2022a}. We conjecture that this phenomenon is primarily due to two facts: Firstly, there lacks open-to-public large-scale survival datasets that allows scalable evaluation of neural survival models. Secondly, for most of the current available datasets, there are no authoritative train-test splits. Consequently, most experiments are done using cross-validation on moderate scale datasets, causing the resulting variability of the modeling algorithms to be relatively large. Therefore we suggest the survival analysis community to release (in a privacy-preserving manner) more large-scale, sanitized datasets equipped with standard train-test splits, which we believe will greatly benefit the state-of-the-art for neural survival modeling.
    \end{description}

\end{document}